\documentclass{article}
\usepackage{algorithm}
\usepackage{algpseudocode}[1]
\usepackage{tikz}
\usepackage{shortcuts}
\usepackage{hyperref}
\usepackage{tocloft}   
\usepackage{appendix}  
\usepackage{minitoc}

\usepackage[authoryear, round]{natbib}
 \bibpunct[, ]{(}{)}{,}{a}{}{,}%

\usepackage[utf8]{inputenc} 
\usepackage[T1]{fontenc}    
\usepackage{hyperref}       
\usepackage{url}            
\usepackage{booktabs}       
\usepackage{amsfonts}       
\usepackage{nicefrac}       
\usepackage{microtype}      
\usepackage{xcolor}         
\usepackage{cancel}
\usepackage{microtype}
\usepackage{amsmath,amsthm,amssymb}
\newtheorem{lemma}{Lemma}

\newtheorem{theorem}{Theorem}
\newtheorem{assumption}{Assumption}
\newtheorem{definition}{Definition}
\newtheorem{remark}{Remark}

\usepackage{amsthm,amsmath,amssymb,booktabs,xcolor}
\usepackage{outlines}
\usepackage{shortcuts}
\usepackage[margin=1in]{geometry}
\usepackage[utf8]{inputenc}
\usepackage{mathtools, xcolor}
\usepackage{hyperref}
\usepackage{ulem}
\usepackage{amsthm}
\usepackage{multirow}

\usepackage{siunitx} 
\usepackage{enumitem}
\usepackage{authblk}

\begin{document}

%

%

\title{Exploration in the Limit}
\author[1]{Brian M. Cho}
\author[1,2]{Nathan Kallus}
\affil[1]{Cornell University}
\affil[2]{Netflix}
\date{}

\maketitle


\begin{abstract}
In fixed-confidence best arm identification (BAI), the objective is to quickly identify the optimal option while controlling the probability of error below a desired threshold. Despite the plethora of BAI algorithms, existing methods typically fall short in practical settings, as stringent exact error control requires using loose tail inequalities and/or parametric restrictions. To overcome these limitations, we introduce a relaxed formulation that requires valid error control asymptotically with respect to a minimum sample size. This aligns with many real-world settings that often involve weak signals, high desired significance, and post-experiment inference requirements, all of which necessitate long horizons. This allows us to achieve tighter optimality, while better handling flexible nonparametric outcome distributions and fully leveraging individual-level contexts.
We develop a novel asymptotic anytime-valid confidence sequences over arm indices, and we use it to design a new BAI algorithm for our asymptotic framework. Our method flexibly incorporates covariates for variance reduction and ensures approximate error control in fully nonparametric settings. Under mild convergence assumptions, we provide asymptotic bounds on the sample complexity and show the worst-case sample complexity of our approach matches the best-case sample complexity of Gaussian BAI under exact error guarantees and known variances. Experiments suggest our approach reduces average sample complexities while maintaining error control.  
\end{abstract}

\section{Introduction}\label{sec:Intro}

In modern experiments, researchers often test multiple treatment options/arms with the goal of finding the best-performing option. In applications such as drug testing in clinical trials~\citep{bai_healthcare}, channel allocation for cellular networks~\citep{bai}, and ad optimization on online platforms~\citep{bhattacharjee2023bestarmidentificationrare}, analysts test multiple options within an experiment, hoping to deduce the most promising option among those tested. For this goal, it is natural to ask: How should a researcher allocate measurement efforts across options? When should a researcher deem an option as the best-performing option and stop the trial, given that they want a certain level of confidence?

To address such questions, researchers use best arm identification (BAI) approaches from the multi-armed bandit literature. In BAI, the researcher sequentially chooses options to measure and observes independent, noisy signals regarding their quality. The goal is to allocate samples effectively such that the best option can be identified confidently in the smallest number of measurements possible. Despite aligning with the aforementioned goals, current approaches to best arm identification often fail to model real-world experimentation scenarios. Current methods that obtain optimal sample complexities require response distributions to follow restrictive parametric assumptions (e.g., exponential family with known variances, \citealp{garivier2016optimalbestarmidentification, jedra2020optimalbestarmidentificationlinear}) and can only incorporate contextual information, such as individual attributes, in limited settings \citep{kato2024rolecontextualinformationbest}.

However, in practice, context and outcome distributions are often complex, and making strong restrictions that surely cannot hold exactly is at odds with the stringent requirement of exact type-I error control.
At the same time, experimenters often collect a substantial amount of data before terminating the trial due to small signals, stringent error control requirements, and/or post-experiment inference considerations. This gives rise to the opportunity of using asymptotic approximations fo arbitrary nonparametric distributions of outcomes, including conditional distributions with respect to contexts. While standard BAI methods provide error control and instance-optimal sample complexity under simple parametric models, they fail to provide such guarantees under realistic data generation processes.

In this work, we develop a best-arm identification method tailored to (i) long horizon experiments, (ii) unknown outcome distributions, and (iii) potentially complex, nonlinear relationships between individual contexts and outcomes. Distinct from existing definitions for BAI in the literature, our approach relies on a relaxation of the error constraint that ensures error is \textit{approximately} controlled beyond a minimum number of samples, which arises naturally in settings with small signals and/or stringent error probability guarantees. Under this relaxed guarantee, we propose a BAI framework based on a novel asymptotic anytime-valid confidence sequence over arm indices that contains the best arm with high probability. By minimizing the expected sample complexity of our framework, we provide mild conditions under which our approach has worst-case sample complexity \textit{no larger} than the optimal sample complexity for Gaussian BAI with known variances. Beyond theoretical guarantees, we conduct synthetic experiments under a simple set-up matching that of existing work. Our results show average sampling complexity reductions up to 33\% relative to existing methods, while still satisfying user-specified error probability constraints.


\paragraph{Contributions} Our work introduces (i) a novel relaxation of the standard PAC framework for bandit exploration, (ii) an asymptotic anytime-valid confidence sequence for determining the best arm, and (iii) an algorithm that leverages our confidence sequences for BAI. We expand on each contribution below.

\begin{itemize}
    \item \textit{Novel Problem Formulation}: Bandit exploration problems assume that an experiment can be stopped at \textit{any time} during the experiment. In contrast, our approach leverages a burn-in parameter $t_0$ that provides a minimum sample size for the experiment. Our methods ensure the desired level of error control as the parameter $t_0$ grows large, ensuring \textit{asymptotic} error control. 

    \item \textit{Confidence Sequences for the Best Arm}: To construct our BAI approach, we leverage a novel, asymptotic anytime-valid confidence sequence over arm indices to determine (i) when to stop the experiment and (ii) which arm to return as best. We construct our asymptotic confidence sequences by leveraging weighted sums of unbiased scoring functions, generalizing doubly robust estimators for the purposes of BAI. Our weighting procedure corresponds to maximizing the signal-to-noise (SNR) of our test processes and is constructed using a simple concave fractional program. In a simple setting with no contexts, we show that our weighting scheme implicitly corresponds to Kullback–Leibler (KL) projection.
    
    \item \textit{Sample Complexity Benefits}: To optimize our confidence sequence approach, we provide a sampling scheme based on projected subgradient descent that minimizes the asymptotic sample complexity of our method. Under convergence assumptions that allow for \textit{any} rate of convergence, we show that the \textit{worst-case} asymptotic sample complexity of our method is no worse than the \textit{best-case} complexity for Gaussian BAI with known arm variances. Our results demonstrate that under our relaxed error guarantees, (i) nonparametric BAI (without contexts) is no harder than Gaussian BAI with known variances and (ii) contextual information can yield sample complexities \textit{strictly} less than that of Gaussian BAI. We connect our approach to semi-parametric efficient estimation to show that our approach \textit{efficiently} leverages contextual information for sampling, stopping, and arm selection.  
\end{itemize}

\paragraph{Outline} Our work proceeds as follows. In the remainder of this section, we provide an overview of related work, focusing on existing works for best-arm identification and asymptotic anytime valid inference approaches. In Section \ref{sec:problem_statement}, we introduce our modeling assumptions and inference goals, focusing on our asymptotic relaxation of error control. Section \ref{sec:av_tests} introduces the framework of our BAI algorithm, which builds upon a novel, asymptotic anytime-valid confidence sequence over arm indices. We demonstrate how to construct our confidence sequences and provide both information-theoretic and testing-by-betting \citep{testing_as_betting} interpretations for our approach. In Section \ref{sec:sampling}, we propose a sampling scheme that minimizes the expected sample complexity of our confidence sequence-based approach via projected sub-gradient descent. We provide results on the asymptotic expected stopping time of our procedure and compare our results with known lower bounds for the standard BAI problem in common parametric models. Section \ref{sec:experiments} presents our experiments, and we provide our concluding remarks and future extensions in Section \ref{sec:conclusion}.

\subsection{Related Works}

We present a brief overview of existing works that closely relate to our proposed method. In particular, we focus on existing approaches for best arm identification, the design of their stopping rules, and asymptotic anytime-valid inference based on strong invariance principles.

\paragraph{Best Arm Identification.} The goal of identifying the option with largest mean response has been studied extensively in the pure exploration bandit literature. In the fixed budget setting~\citep{fixed_budget, bai}, the experimenter aims minimize the error of recommending a suboptimal arm, given a fixed budget of samples. In the fixed confidence setting~\citep{garivier2016optimalbestarmidentification, russo2018simplebayesianalgorithmsbest, bai_healthcare}, the experimenter aims to minimize the number of samples needed to recommend an arm as best, given an error level constraint. Our work builds upon the fixed confidence regime under a relaxed error level constraint. In contrast with our setting, existing best arm identification approaches require exact error control, often do not assume access to covariates, or allow complex, nonparametric distributions. For example, works such as \citet{garivier2016optimalbestarmidentification} focus on responses belonging to a known, exponential family and provide lower bounds on the expected stopping time. The closest works to our setting are \cite{kato2023bestarmidentificationcontextual} and \cite{kato2024rolecontextualinformationbest}. However, \cite{kato2023bestarmidentificationcontextual} studies the fixed-budget BAI problem with contextual information under an asymptotic regime, where the limit is with respect to the sampling budget. In contrast, our setting is the fixed-confidence regime, and our limits are with respect to the error tolerance rather than sampling budget. While \cite{kato2024rolecontextualinformationbest} study fixed-confidence BAI with contexts, their proposed approach focuses on the standard PAC guarantee, resulting in methods that only achieve their sample complexity bounds in limited parametric contexts: (i) two-armed BAI under responses and contexts that jointly follow a multivariate Gaussian distribution, and (ii) BAI with finite-cardinality contexts and responses generated under an exponential family. In contrast, our approaches are readily applicable to a wide variety of settings, including continuous contexts and nonparametric response distributions.

\paragraph{Anytime-Valid Inference.} For fixed confidence bandit exploration problems, the decision of when to stop the experiment are based on anytime-valid confidence sequences and sequential tests ~\citep{garivier2016optimalbestarmidentification,kaufmann2021mixturemartingalesrevisitedapplications,cho2024rewardmaximizationpureexploration, cho2024peekingpeaksequentialnonparametric, Howard_2021}. Anytime-valid inference approaches control error levels uniformly across repeated testing for all time points by leveraging the martingale maximal inequality of \citet{Ville1939}. This provides a natural approach for fixed confidence exploration, which tests for the best arm at each time point to determine when to stop. For example, the Track-and-Stop approaches by \cite{garivier2016optimalbestarmidentification} use composite sequential likelihood ratios as their stopping criteria for BAI, while \cite{cho2024rewardmaximizationpureexploration} leverages the generalized Bernoulli $e$-process for threshold tests. However, these methods are often hindered by two limitations in practice: (i) requiring knowledge or assuming bounds on the moment generating function (MGF) of the response distributions, and (ii) conservative performance when these MGF bounds are loose. Because analysts tend to specify larger bounds on distributions to maintain valid error control (such as the sub-gaussian factor of $\sigma^2 = 1/4$ for $[0,1]$-bounded random variables), these limitations hinder the practical performance of existing approaches for best arm identification under exact error control requirements. The conservative performance is well documented in works such as \citet{garivier2016optimalbestarmidentification} and \citet{cho2024rewardmaximizationpureexploration}, which use anytime valid inference approaches as their stopping criteria.

\paragraph{Asymptotic Anytime-Valid Inference.} To overcome the limitations of standard anytime-valid approaches, more recent works have proposed the notion of \textit{asymptotic} anytime valid inference~\citep{waudbysmith2024timeuniformcentrallimittheory, bibaut2024nearoptimalnonparametricsequentialtests} to calibrate anytime-valid testing procedures. In particular, our work leverages a stronger notion of an asymptotic confidence sequence and sequential test presented in \citet{bibaut2024nearoptimalnonparametricsequentialtests}, which ensures error control beyond a prespecified burn-in time. Like previous works, our approach builds upon semi-parametrically efficient scoring functions \citep{bickel1998efficient, chernozhukov2024doubledebiasedmachinelearningtreatment, cook2024semiparametricefficientinferenceadaptive, oprescu2025efficientadaptiveexperimentationnoncompliance} generalized for our sequential setting. In contrast to these works, however, our goal is not to provide valid inference on the \textit{value} of any given arm (or differences thereof), but to label it as suboptimal as quickly as possible. To this end, our approach combines the efficient scores used to estimate arm mean differences using a novel, sequential weighting scheme that maximizes the signal-to-noise ratio (SNR), tailored to the composite null hypothesis that a given arm is the best arm. Furthermore, we provide both (i) an analysis of our method's sample complexity and (ii) a sampling scheme that minimizes its upper bound. Among all existing asymptotic anytime-valid methods, only the work of \cite{bibaut2024nearoptimalnonparametricsequentialtests} characterizes the expected sample complexity of their procedure, and no previous work provides a corresponding sampling scheme to minimize the expected sample complexity for their testing procedure.

\section{Problem Formulation}\label{sec:problem_statement}

In this section, we first provide our modeling assumptions on the data-generating process. We then define the best arm identification problem as defined in the literature~\citep{garivier2016optimalbestarmidentification}, and provide our relaxation that provides a limiting notion of error control with respect to a sequence of BAI algorithms.

\subsection{Modeling Assumptions}

We define the set of all collected information up to time $T$ as $H_T=(X_t, A_t, Y_t)_{t=1}^{T}$, where $(X_t, A_t, Y_t)$ denote the context, arm, and outcome observed at time $t$. We set $H_0 = \{\emptyset\}$ as the empty set. We denote the canonical filtration at time $T$ as $\Fcal_{T} = \sigma\left((X_t, A_t, Y_t)_{t=1}^T\right)$, with $\Fcal_0$ as the trivial, empty sigma field.

Our data-generating process (DGP) proceeds in the following sequential manner.  At each time $t$, the learner observes a context $X_t \in \Xcal$, where $X_t$ is distributed according to a fixed, unknown distribution $P_X$. After observing the context $X_t$, the learner selects a treatment $A_t \in [K]$, where $[K] \equiv \{1,...,K\}$ denotes a discrete, finite set of $K$ arms. The choice of arm is specified by the policy $\pi: (H_{t-1},\Xcal) \rightarrow \Delta^K$, where $\Delta^K$ denotes the $K$-dimensional probability simplex. The learner then observes outcome $Y_t \in \RR$, where $Y_t$ is distributed according to a fixed, unknown distribution $P_{Y| A,X}$. Overloading notation, we denote $\pi_t(x, a) = P(A_t = a| X_t=x, H_{t-1} )$ as the conditional probability of selecting the option $a \in [K]$ given the current context $X_t$ and history $H_{t-1}$. We denote vectors in bold as $\bm{w}$, with the $i$-th component of vector $\bm{w}$ denoted as $w(i)$ and $\bm{w}$ with the $i$-th component removed as $\bm{w}(-i)$. We define the set $\Delta(a)\coloneqq \{\bm{w} \in \RR^K: w(a) = -1, \ \bm{w}(-a) \in \Delta^{K-1}\}$ as the set of all vectors with the $a$-th component equal to $-1$, and the remaining components lying in the $K-1$ simplex. We use vectors $\bm\mu  = [\mu(1),...,\mu(K)] \in \RR^K$ and $\bm{\sigma}^2 = \left[ \sigma^2(1), \dots, \sigma^2(K) \right] \in \RR^K$ to denote the vectors of arm means and variances, where the $a$-th component of each vector corresponds to 
\begin{align*}
    \mu(a) \coloneqq \EE_{P_X}\left[\EE_{P_{Y|A,X}}[Y|A=a, X] \right], \quad 
    \sigma^2(a) \coloneqq \EE_{P_X}\left[\EE_{P_{Y|A,X}}\left[ \left(Y - \mu(a)\right)^2 | A=a, X \right]\right].
\end{align*}
To denote conditional means and variances given an arm $A=a$ and context $X=x$, we define
\begin{align*}
    g(x,a) \coloneqq \EE_{P_{Y|A,X}}\left[Y|A=a, X=x\right], \quad v(x,a) \coloneqq  \EE_{P_{Y|A,X}}\left[\left(Y - g(x,a)\right)^2 | A=a, X=x \right] 
\end{align*}
as our conditional mean and variance functions. We define $\| f\|_{L_q(P)} \coloneqq \EE_{P}[|f|^q]^{1/q}$ as the $L_q$ norm with respect to the distribution $P$. Note that $P$ can be simple marginal distributions $P_X$, or more complex conditional distributions such as $P_{X|A=a, H_{t-1}}$. We make three standard assumptions on our DGP below.

\begin{assumption}[Unique Optimal Arm]\label{assump:unique_optimal_arm}
    There exists a unique treatment option $a^* = \argmax_{a \in [K]} \mu(a)$. 
\end{assumption}
Assumption \ref{assump:unique_optimal_arm} is a common assumption in the best arm identification literature to ensure that the problem is well-defined. Without Assumption \ref{assump:unique_optimal_arm}, existing approaches for BAI have infinite sample complexity \citep{bai, garivier2016optimalbestarmidentification}, and will not terminate in finite time. 

\begin{assumption}[Nondegenerate Variances]\label{assump:nonzero_variances}
    For all $x \in \Xcal$, $a \in [K]$, $v(x,a)$ is positive.
\end{assumption}

Assumption \ref{assump:nonzero_variances} ensures that our sample complexities do not degenerate towards zero and avoids trivial cases for BAI. This assumption is likely to hold in practice. 
Note that Assumption \ref{assump:nonzero_variances} ensures that marginal variances $\bm{\sigma}^2$ are also positive by the law of total variation.

\begin{assumption}[Boundedness of Outcomes]\label{assump:bounded_outcomes}
    There exists a constant $B$ such that $|Y_t| < B$ for all $t \in \NN$. 
\end{assumption}

In most common applications, Assumption \ref{assump:bounded_outcomes} is likely to hold, even if the maximum magnitude of the outcome variable is unknown in advance. We emphasize that this constant $B$ does not need to be known in advance, estimated, or assumed to be any certain value across our methods. It only plays a role in our theoretical guarantees, and is not an input to any component of our BAI algorithm. In contrast, existing pure exploration methods \citep{garivier2016optimalbestarmidentification, cho2024rewardmaximizationpureexploration} require as in input an upper bound on this constant $B$ (such as $[0,1]$-bounded outcomes) or moment bounds (e.g., $1$-sub-gaussian) in order to maintain valid error control. Other than the assumptions provided above, we make \textit{no further assumptions} on the DGP. Outcomes do not have to follow parametric modeling assumptions (e.g., exponential family such as \citealp{garivier2016optimalbestarmidentification}), and conditional regression functions are not assumed to follow simple parametric models (e.g., linear functions with a link function such as \citealp{kazerouni2019bestarmidentificationgeneralized}).

\subsection{Best Arm Identification}

A BAI algorithm $\Bcal = ( \pi, \xi, \hat{a})$ consists of (i) a sampling scheme $\pi: (H_{t-1},X_t) \rightarrow \Delta^K$ that determines the arm selection at each time $t$, (ii) a stopping rule $\xi: H_t \rightarrow \{0,1\}$, which returns the binary decision to stop at time $t$, and (iii) an answer $\hat{a} \in [K]$ that returns the arm index deemed to be the largest mean arm at termination (that is, $\hat a$ is measurable with respect to $H_{\inf\{t:\xi(H_t)=1\}}$). An algorithm $\Bcal$ is $\alpha$-correct if it terminates almost surely and returns the correct answer with probability at least $1-\alpha$.

\begin{definition}[$\alpha$-Correctness]\label{defn:alpha_correctness}
   An algorithm $\Bcal = (\pi, \xi, \hat{a})$ is $\alpha$-correct if (i) the algorithm $B$ terminates almost surely, i.e. $P\left(\exists t < \infty: \xi(H_t) = 1\right) = 1$, and (ii) the probability of returning the best arm $a^*$ is at least $1-\alpha$, i.e. the probability of returning a suboptimal arm satisfies $P(\hat{a} \neq a^*) \leq \alpha$.
\end{definition}

This definition of $\alpha$-correctness is the standard requirement for BAI across all existing work in the fixed confidence setting. In contrast, we propose a relaxation of error control for a sequence of BAI algorithms $\left(\Bcal_{t_0}\right)_{t_0 \in \NN_0}$, defined with respect to an index parameter $t_0$. 

\begin{definition}[Asymptotic $\alpha$-correctness]\label{defn:asymp_alpha_correctness}
A sequence of BAI algorithms $\left(\Bcal_{t_0}\right)_{t_0 \in \NN_0} = \left((\pi_{t_0}, \xi_{t_0}, \hat{a}_{t_0})\right)_{t_0 \in \NN_0}$ is asymptotically $\alpha$-correct if (i) for each fixed $t_0$, $B_{t_0}$ terminates almost surely, i.e. $P\left(\exists t < \infty: \xi_{t_0}(H_t) = 1\right)=1$, and (ii) the probability of returning the optimal arm $a^*$ converges to at least $1-\alpha$ as $t_0 \rightarrow \infty$, i.e. $\limsup_{t_0 \rightarrow \infty} P(\hat{a}_{t_0} \neq a^*) \leq \alpha$. 
\end{definition}

Definition \ref{defn:asymp_alpha_correctness} is a strict relaxation of the $\alpha$-correctness property in Definition \ref{defn:alpha_correctness} by only requiring the sequence of algorithms $\left(\Bcal_{t_0}\right)_{t_0 \in \NN_0}$ to satisfy error control as the index $t_0$ diverges to infinity. Any algorithm $\Bcal$ satisfying $\alpha$-correctness implicitly satisfies asymptotic $\alpha$-correctness by using the trivial sequence $\Bcal_{t_0} = \Bcal$ for all $t_0 \in \NN_0$.

In our work, the index parameter $t_0$ takes the role of a \textit{burn-in time}, where algorithm $\Bcal_{t_0}$ does not stop before any time $t < t_0$, i.e. $\xi_{t_0}(H_{t}) = 0$ for all $t < t_0$. Equivalently, the burn-in time parameter $t_0$ represents a \textit{minimum} sample size for the experiment. This choice aims to match common scenarios in practice: weak signal strength (i.e. small gaps between the best arm and its alternatives), stringent error requirements, and/or post-experiment inference considerations often result in long experiment horizons, corresponding to the setting with where $t_0$, the minimum sample size of an experiment, diverges towards infinity.

\begin{remark}[Choice of Index as a Burn-in Time]
    While Definition \ref{defn:asymp_alpha_correctness} does not require the index parameter $t_0$ to enforce a minimum sample complexity, we set the index parameter $t_0$ as a burn-in time to match the guarantees of asymptotic anytime valid inference, as defined in \cite{bibaut2024nearoptimalnonparametricsequentialtests} and Theorem 2.8 of \cite{waudbysmith2024timeuniformcentrallimittheory}. Our decision to parameterize an explicit minimum sample size $t_0$ plays a minimal role in our algorithm beyond controlling asymptotic error rates. In Section \ref{sec:av_tests}, we provide a choice of burn-in time $t_0(\alpha)$ with respect to the error tolerance $\alpha$ that ensures (i) the sequence of burn-in times $t_0(\alpha)$ satisfy $t_0(\alpha) \rightarrow \infty$ as $\alpha \rightarrow 0$ and (ii) asymptotic sample complexities (with respect to sequences $\left(\Bcal_{t_0(\alpha)}\right)_{\alpha \in (0,1)}$ as $\alpha \rightarrow 0$) match or outperform well-known existing sample complexities for BAI.  
\end{remark}



\section{Exploration with Confidence Sequences}\label{sec:av_tests}

To determine when to stop and which arm to declare best, our approach leverages confidence sequences $\left(C_t\right)_{t=1}^\infty$ over the arm indices $[K]$ that satisfy asymptotic anytime-valid error guarantees, i.e.
\begin{equation}\label{eq:conf_seq_requirement}
    \limsup_{t_0 \rightarrow \infty}P\left(\exists t \geq t_0: a^* \not\in  C_t(t_0, H_t, \alpha) \right) \leq \alpha.
\end{equation}
Confidence sequences $(C_t)_{t=1}^\infty$ satisfying Equation \eqref{eq:conf_seq_requirement} ensure that the best arm $a^*$ is uniformly contained in $C_t$ for all $t$ greater than the burn-in time $t_0$ with high probability. Naturally, this implies a simple strategy for our BAI procedure: whenever the confidence sequence $C_t$ contains a single arm at \textit{any} time step $t \geq t_0$, one can immediately conclude the experiment and return the remaining arm as best. 

To construct our asymptotic anytime-valid confidence sequences, we proceed in the following manner. For each arm $a$, we construct a test process $(\hat\psi_t(a))_{t \in \NN}$ adapted to the filtration $(\Fcal_t)_{t \in \NN}$. Each test process corresponds to the composite null $\Hcal_a: \mu(a) = \max_{b \in [K]} \mu(b)$, the set of distributions where arm $a$ is the best arm. When the null $\Hcal_a$ is true, its associated test process $\hat\psi_t(a)$ has non-positive drift at each time $t \in \NN$, which enables us to reject $\Hcal_a$ if the cumulative drift is deemed positive. Our confidence set sequentially removes the arms $a$ whose corresponding test process $\hat\psi_t(a)$ drift is deemed positive by an asymptotic anytime-valid test, resulting in a confidence sequence with our desired statistical guarantees.

\subsection{Constructing Test Processes}

To construct our arm-specific test processes $(\hat\psi_t(a))_{t \in \NN}$, we first begin with arm-specific score processes $(\phi_t(a))_{t=1}^\infty$ in Definition \ref{defn:score_process} that serve as unbiased estimates for the mean of arm $a$. 

\begin{definition}[Score Process]\label{defn:score_process}
    For each $a \in [K]$, let $\left(\phi_t(a)\right)_{t \in \NN}$ be a process adapted to the filtration $\left(\Fcal_t\right)_{t \in \NN}$. For each $t \in \NN$, let $\phi_t(a)$ denote the function $\phi_t(a) \coloneqq g_t(X_t, a) + \frac{\mathbf{1}[A_t = a](Y_t - g_t(X_t, a))}{\pi_t(X_t, a)}$, where $\pi_t(X_t, a) = P(A_t=a|X=X_t,H_{t-1})$, and $g_t: \Xcal\times [K] \rightarrow \RR$ is an $\Fcal_{t-1}$-measurable function. 
\end{definition}
The time-varying function $g_t(X,a)$ corresponds to the best estimate of the true conditional expectation function $g$, using the observations collected until time $t-1$. The function $g_t$ can be estimated with complex algorithmic regressors, such as random forests \citep{random_forests}, neural networks \citep{pmlr-v70-shalit17a}, and boosting algorithms \citep{K_nzel_2019}, under mild convergence conditions. Crucially, regardless of our choice of $g_t$, the score processes $\phi_t(a)$ acts as an \textit{conditionally unbiased estimator} for mean of arm $a$. Because functions $g_t$ and $\pi_t$ are $\Fcal_{t-1}$-measurable and therefore fixed conditional on history $H_{t-1}$, our score processes satisfy $\EE[\phi_{t}(a)|H_{t-1}] = \mu(a)$ for each $a \in [K]$ and $t \in \NN$, regardless of the choice of regression function $g_t$. 

Our confidence sequences build on the score processes of Definition \ref{defn:score_process} by constructing the test process $(\hat\psi_t(a))_{t \in \NN}$, a weighted combination of score processes $(\phi_t(a))_{t \in \NN}$, for each composite null hypothesis $\Hcal_a$. For the null hypothesis $\Hcal_a$, we define 
\begin{equation}
    \hat\psi_t(a) = \frac{1}{t}\sum_{b \in [K]} w_t^a(b) \phi_t(b)
\end{equation}
where $\bm{w}_t^a \in \Delta(a)$ is an $\Fcal_{t-1}$-measurable vector for all $t \in \NN$. The arm-specific test process $(\hat\psi_t(a))_{t \in \NN}$ corresponds to a normalized process with \textit{non-positive} drift under the null $\Hcal_a$. Specifically, due to the fact that $\bm{w}_t^a \in \Delta(a)$, $\bm{w}_t^a$ is $\Fcal_{t-1}$-measurable, and score processes $(\phi_t(a))_{t \in \NN}$ are conditionally unbiased, the non-normalized test process $t\hat\psi_t(a)$ satisfies the following for every distribution $P \in \Hcal_a$:

\begin{equation}
    \EE_P\left[t\hat\psi_t(a)- (t-1)\hat\psi_t(a) | H_{t-1}\right] = \sum_{b \in [K]} w_t^a(b)\mu(b) = \left(\sum_{b \neq a} w_t^a(b) \mu(b) \right) - \mu(a) \leq 0.
\end{equation}

\begin{figure*}[t]
    \includegraphics[width=\linewidth]{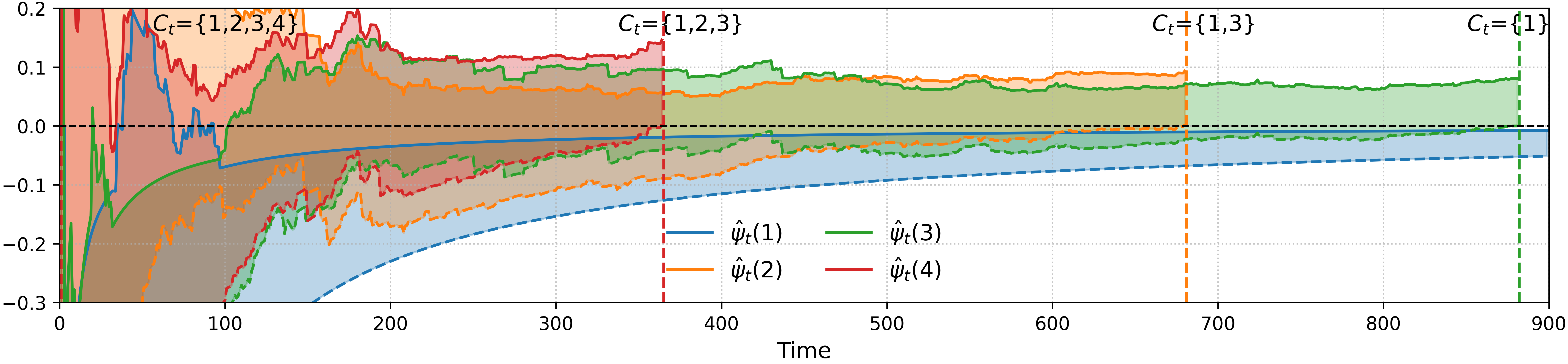}
    \caption{Visualization of Confidence Sequence Approach. Solid lines plot score process $\hat\psi_t(a)$, and dotted lines plot asymptotic anytime-valid lower bounds $L_t^a(H_t, \alpha, \rho)$. Arm $a$ is removed from $C_t$ when $L_t^a(H_t, \alpha, \rho) > 0$. }
    \label{fig:cs_visualization}
\end{figure*}
\begin{algorithm}[t]
\begin{algorithmic}[1]
\Procedure{BestArmId}{$\pi, \alpha, \rho, t_0$}
    \State Set $C_0(t_0,H_0, \alpha)\gets [K]$, $t\gets 0$, $H_0 \gets \{\emptyset\}$. 
    \While{$ | C_t(t_0,H_t, \alpha)| > 1$} 
        \State Increment time index $t \gets t+1$.
        \State Observe $X_t$, and sample $A_t \sim \pi_t(X_t, \cdot)$.
        \State Observe $Y_t$, set $H_t \gets H_{t-1} \cup (X_t, A_t, Y_t) $. 
        \State Update $C_t(t_0, H_t, \alpha)$ according to Equation \eqref{eq:ci_w_lower_bounds}.
    \EndWhile
    \State \textbf{return} arm $\hat{a} \in C_t(t_0,H_t, \alpha)$ if $|C_t(t_0, H_t, \alpha)|=1$, else $\hat{a} \in \argmin_{a \in [K]}{\hat\psi_t(a)}- {\hat\sigma_t(a)}\ell_{t,\alpha, \rho}\left(\hat\sigma_t(a)\right)$ .   
\EndProcedure
\end{algorithmic}
\caption{Asymp-BAI}\label{alg:bai_short}
\end{algorithm}

Thus, to determine whether arm $a$ can be removed from confidence sequence $C_t$ (i.e. arm $a$ is not the best arm), it suffices to test whether the drift of its test process $\hat\psi_t(a)$ is positive. To test the sign of $\hat\psi_t(a)$'s drift while maintaining the guarantees of Equation \eqref{eq:conf_seq_requirement}, we construct asymptotic, anytime-valid lower confidence bounds $L_t^a(H_t, \alpha, \rho)$ based on Gaussian mixture martingales:
\begin{align}\label{eq:lower_bounds}
    L_t^a(H_t, \alpha, \rho) = &\hat\psi_t(a) - \hat\sigma_t(a) \ell_{t,\alpha, \rho}(\hat\sigma_t(a)), \quad \hat\sigma_t^2(a)=\frac{1}{t}\sum_{i=1}^t\left(\sum_{b \in [K]} w_i^a(b)(\phi_i(b) - \hat\mu_i(b)) \right)^2,\\
    &\ell_{t,\alpha,\rho}(x) = t^{-1/2}\sqrt{\frac{2(\rho^2+1/tx^2)}{\rho^2}\log\left(1+\frac{\sqrt{t x^2\rho^2 + 1}}{2\alpha}\right)}
\end{align}
The term $\hat\sigma_t^2(a)$ denotes the estimated cumulative conditional variance of score process $(\hat\psi_i(a))_{i=1}^t$, and $\ell_{t,\alpha,\rho}(x)$ corresponds to an asymptotic anytime-valid bound based on strong invariance principles ~\citep{waudbysmith2024timeuniformcentrallimittheory} and Gaussian mixture martingales ~\citep{kaufmann2021mixturemartingalesrevisitedapplications}. The process $L_t^a(H_t, \alpha, \rho)$ serves as a time-uniform, high-probability lower bound for the running  drift of $\hat\psi_t(a)$. When $L_t^a(H_t, \alpha, \rho)$ crosses above zero at any $t \geq t_0$, the asymptotic anytime-valid guarantees for $L_t^a$ ensure that one can conclude that $\hat\psi_t(a)$ has positive drift with high probability. Our confidence sequences $(C_t)_{t=1}^\infty$ follow from this logic, where 
\begin{equation}\label{eq:ci_w_lower_bounds}
    C_t(t_0, H_t, \alpha) = \{a \in [K]: \sup_{t_0\leq i \leq t} L_i^a(H_i, \alpha, \rho) \leq 0 \}
\end{equation}
is simply the set of all arms $a$ such that  $L_i^a(H_i, \alpha, \rho)$ has never crossed above zero for any $i \leq t$.

We provide both pseudocode for our BAI approach in Algorithm \ref{alg:bai_short} and a simple visualization in Figure \ref{fig:cs_visualization}. Our confidence sequences $C_t(t_0,H_t, \alpha)$ determine (i) when to stop and (ii) which arm to return as best. At each time $t_0$, we construct our arm-specific test processes $\hat\psi_t(a)$ (shown in solid lines in Figure \ref{fig:cs_visualization}) and their corresponding lower bounds $L_t^a(H_t, \alpha, \rho)$ (shown in dotted lines in Figure \ref{fig:cs_visualization}). As soon as the lower bound $L_t^a(H_t,\alpha,\rho)$ lies above zero at any time $t \geq t_0$, we remove the arm $a$ from our confidence set $C_t(t_0, H_t, \alpha)$. When our confidence set $C_t(t_0, H_t, \alpha)$ contains at most one arm, our BAI algorithm terminates and returns last remaining arm in $C_t(t_0, H_t, \alpha)$. In the case where $|C_t(t_0, H_t,\alpha)| = 0$ (i.e. last remaining arms eliminated at the same time), Algorithm \ref{alg:bai_short} returns the arm $\hat{a}$ with the \textit{smallest} lower confidence bound for $\hat\psi_t(a)$.

\begin{remark}[Selection of $\rho$ Parameter]
    The lower bounds $L_t^a(H_t, \alpha, \rho)$ introduce a new parameter $\rho$, which governs where our lower bounds are the tightest across time with respect to an \textit{intrinsic} time $t_*$. In this work, we provide error guarantees and stopping time results for all fixed $\rho >0$ specified in advance of testing. We discuss selecting $\rho$ based on user preferences over the hardness of the BAI instance in Appendix \ref{app:hyperparameters}. 
\end{remark}

\subsection{Maximizing the Signal-to-Noise Ratio}
To ensure suboptimal arms $a \neq a^*$ are removed from $C_t(t_0, H_t, \alpha)$ over time (i.e. Algorithm \ref{alg:bai_short} terminates), we require their corresponding lower bounds in Equation \eqref{eq:lower_bounds} grow above zero. To do so, for each $ a\in[K]$, we select the sequence of $\Fcal_{t-1}$-measurable weight vectors $\bm{w}_t^a$ that maximizes the \textit{signal-to-noise ratio} (SNR) for each test process $(\hat\psi_t(a))_{t \in \NN}$. In Algorithm \ref{alg:snr_maximization}, we propose our weight construction scheme, which aims to maximize $\hat\psi_t(a)/\hat\sigma_t(a)$, the ratio of the test process drift and its cumulative conditional standard deviation. 

\begin{algorithm}[t]
\begin{algorithmic}[1]
\Procedure{SNRMax}{$a, H_{t-1}, \bm{w}_0^a$}
    \State Initialize the weight vector $\bm{w}_t^a = \bm{w}_0^a$.
    \State Compute $\hat\mu_{t-1}(b) = \frac{1}{t-1}\sum_{i=1}^{t-1} \phi_i(b)$ and $\tilde\sigma_t^2(b) = \frac{1}{t-1}\sum_{i=1}^{t-1} (\phi_i(b) - \hat\mu_{i}(b))^2$ for each $b \in [K]$. 
    \State Compute the set $\Acal_t^* = \argmax_{b \in [K]} \hat\mu_{t-1}(b)$. 
    \If{$a \not\in \Acal_t^*$ and $\min_{b \in [K]}\tilde\sigma_t^2(b) > 0$} 
        \State Set $\bm{w}_t^a$ as the weight vector $\bm{w} \in \Delta(a)$ that maximizes the estimated SNR: 
     $$\bm{w}_t^a \in \argmax_{\bm{w} \in \Delta(a)} \frac{\left(\sum_{b \in [K]} w(b) \hat\mu_{t-1}(b)\right)}{\hat\sigma_{t-1}(\bm{w})}, \quad \hat\sigma_{t-1}^2(\bm{w}) = \frac{1}{t-1} \sum_{i=1}^{t-1} \left( \sum_{b \in [K]} w(b)\left( \phi_i(b) - \hat\mu_{i}(b) \right)  \right)^2.$$\label{l:snr_maximization_problem}
    \EndIf
    \State \textbf{return} arm-specific weight vector $\bm{w}_t^a$.   
\EndProcedure
\end{algorithmic}
\caption{Signal-to-Noise Ratio (SNR) Maximization}\label{alg:snr_maximization}
\end{algorithm}

Our weight selection procedure in Algorithm \ref{alg:snr_maximization} provide a simple approach for selecting the weight vector $\bm{w}_t^a$ for each $a \in [K]$ across all $t \in \NN$. For arms $a$ that appear suboptimal at time $t$ (i.e. $a\not\in \argmax_{b \in [K]}\hat\mu_{t-1}(b)$), our approach solves for the SNR-Maximizing weight vector $\bm{w}_t^a$ in hindsight, using previous observations $H_{t-1}$. When arm $a$ appears optimal at time $t$ (i.e. $a \in \argmax_{b \in [K]} \hat\mu_{t-1}(b)$), our weight scheme defaults to a weight vector $\bm{w}_0^a \in \Delta(a)$ specified before observing any data. 
To avoid infinite objective function values in the maximization problem, our procedure also defaults to $\bm{w}_t^a = \bm{w}_0^a$ when there exists estimated arm variances $\tilde\sigma_t(b)$ equal to zero. Note that the procedure in Algorithm \ref{alg:snr_maximization} is run for each $a \in [K]$ at time $t$ in order to construct the corresponding weight sequences $(\bm{w}_t^a)_{t \in \NN}$ for each test process $\hat\psi_t(a)$. 

Our choice of weight sequences follow from  the structure of our confidence bounds $L_t^a(H_t, \alpha, \rho)$. Recall that we reject $\Hcal_a$ and remove $a$ from $C_t$ whenever $L_t^a(H_t, \alpha, \rho) > 0$ for any $t \geq t_0$. Rearranging $L_t^a(H_t, \alpha, \rho)$, we obtain that $a$ is removed from $C_t$ when $\hat\psi_t(a)/\hat\sigma_t(a) \geq \ell_{t, \alpha, \rho}(\hat\sigma_t(a))$ for any $t \geq t_0$, i.e.
\begin{equation}
    \frac{\hat\psi_t(a)}{\hat\sigma_t(a)} \geq t^{-1/2}\sqrt{\frac{2\left(\rho^2+1/t\hat\sigma^2_t(a)\right)}{\rho^2}\log\left(1+\frac{\sqrt{t \hat\sigma^2_t(a)\rho^2 + 1}}{2\alpha}\right)} 
\end{equation}
Ignoring logarithmic terms, $\ell_{t,\alpha,\rho}(\hat\sigma_t(a))$ is a term converging to zero at the rate $\tilde{O}(1/\sqrt{t})$ for any fixed $\rho > 0$. Thus, weights that maximize the ratio $\hat\psi_t(a)/\hat\sigma_t(a)$ roughly correspond to minimizing the time $t$ at which $\Hcal_a$ can be rejected and arm $a$ can be removed from our confidence sequence $C_t(t_0,H_t, \alpha)$. 

\subsubsection{Information-Theoretic Interpretation}

Beyond the particular structure of our confidence sequence, our SNR-maximizing weighting scheme also has a direct information-theoretic interpretation. For each $a \neq a^*$, the maximized SNR corresponds to the Gaussian KL-projection of the true mean vector $\bm{\mu}$ onto the distributional set $\Hcal_a$. We formalize this result below in Lemma \ref{lem:kl_projection}, focusing on the classical multi-armed bandit setup with no contexts.

\begin{lemma}[SNR Maximization as KL-Projection]\label{lem:kl_projection}
    Assume that the context set $\Xcal$ is empty and $a \neq a^*$. Let $\pi \in \Delta^K$ denote a vector on the $K$-dimensional probability simplex bounded away from zero. Let $d_{\sigma}(x,y) = \frac{(x-y)^2}{2\sigma^2}$ denote the KL divergence function between two Gaussian distributions with equal variances $\sigma^2$. Let $\bm{w}_*^a$ denote a solution to the oracle SNR-maximization problem with true arm means $\bm{\mu}$ and variances $\bm\sigma^2$, i.e. $\bm{w}_*^a = \argmax_{\bm{w} \in \Delta(a)} \frac{\sum_{b \in [K]}w(b) \mu(b)}{\sqrt{\sum_{b \in [K]}w(b)^2 \sigma^2(b)/\pi(b)}}$. Then, the squared SNR achieved by $\bm{w}_*^a$ is half of the minimum KL divergence between the composite null $\Hcal_a$ and the true mean vector $\bm{\mu} \not\in \Hcal_a$, i.e. 
    \begin{equation}
        \frac{1}{2}\left( \frac{\sum_{b \in [K]}w_*^a(b) \mu(b)}{\sqrt{\sum_{b \in [K]} w_*^a(b)^2\sigma^2(b)/\pi(b)}}\right)^2 = \inf_{\tilde{\bm{\mu}} \in \Hcal_a} \sum_{b \in [K]}\pi(b)  d_{\sigma(b)}\left(\mu(b), \tilde{\mu}(b)\right). 
    \end{equation}
\end{lemma}

The results of Lemma \ref{lem:kl_projection} show that under any policy $\pi \in \Delta^K$, our SNR maximization procedure is equivalent to targeting the mean vector $\tilde{\bm{\mu}}$ most difficult to distinguish from the true mean vector $\bm{\mu}$. Recall that to reject the composite null $\Hcal_a$, \textit{every} possible distribution with mean vector $\tilde{\bm{\mu}} \in \Hcal_a$ must be rejected. The oracle SNR-maximizing weights implicitly target the hardest hypotheses $\tilde{\bm\mu} \in \Hcal_a$ to reject, allowing one to reject the whole composite null $\Hcal_a$ and remove arm $a$ from $C_t(t_0, H_t, \alpha)$ when $L_t^a(H_t, \alpha, \rho) >0$ for any $t \geq t_0$. Put succinctly, Lemma \ref{lem:kl_projection} demonstrates that our SNR maximiziation procedure corresponds to standard composite null testing procedures with KL divergence in parametric families, generalized to nonparametric settings with auxiliary information, such as contexts. 

\begin{remark}[Connections to Testing-by-Betting.] As a final interpretation of our SNR maximization procedure, we consider our approach through the "testing-by-betting" lens discussed by \cite{testing_as_betting}. Standard approaches to anytime-valid testing \citep{testing_by_betting_means, cho2024peekingpeaksequentialnonparametric, cho2024rewardmaximizationpureexploration} often leverage a rich connection between maximizing power against a given null and maximizing the returns of a betting system. Our SNR maximization approach in Algorithm \ref{alg:snr_maximization} shares a similar connection to a different problem in mathematical finance: maximizing a portfolio's Sharpe ratio~\citep{Sharpe1994}. For each arm $a$, we construct our test by maximizing the Sharpe ratio against the baseline performance of arm $a$. Each of the $K-1$ arm difference $\hat\mu_{t-1}(b) - \hat\mu_{t-1}(a)$ corresponds to the estimated difference in asset returns adjusted for the benchmark arm $a$, and our weights corresponds to the distribution of capital invested across the assets $b \neq a$ against our benchmark asset $a$. Under this framing, maximizing the Sharpe ratio, i.e. the ratio of risk over return, is equivalent to maximizing the SNR as constructed in Algorithm \ref{alg:snr_maximization}. 
\end{remark}

\subsubsection{Convex Reformulation for Optimization} The procedure presented in Algorithm \ref{alg:snr_maximization} requires us to solve the empirical SNR problem in line \ref{l:snr_maximization_problem}. To solve for our SNR-maximizing weights with standard methods, we provide a convex formulation in Lemma \ref{lem:ccs_snr}.
\begin{lemma}[Charnes-Cooper-Schaible Transform]\label{lem:ccs_snr} For each time $t$ such that $\min_{b \in [K]}\tilde\sigma_t(b) > 0$, for all arm indices $a \not\in \argmax_{b \in [K]} \hat\mu_{t-1}(b)$, there exists a vector $\bm{w}_t^a  \in \argmax_{\bm{w} \in \Delta(a)} \frac{\sum_{b \in [K]} w(b)\hat\mu_{t-1}(a') }{\hat\sigma_{t-1}(\bm{w})}$ has entries $w_t^a(b) = \tilde{{w}}_t^a(b)/\sum_{a' \neq a} \tilde{w}_t^a(a')$ for all $b \neq a$, where
\begin{align}
    {\tilde{\bm{w}}}_t^a \in \argmax_{\bm{w} \in \RR^{K-1}_+} \quad& \sum_{a' \neq a}w(a')\left(\hat\mu_{t-1}(a') - \hat\mu_{t-1}(a) \right)\\
    \text{s.t. } 
    & \hat{\sigma}_{t-1}({\bm{w}}) \leq 1,
\end{align}
and $\hat\sigma_{t-1}(\bm{w})$ is as defined in line \ref{l:snr_maximization_problem} of Algorithm \ref{alg:snr_maximization}. 
\end{lemma}
Lemma \ref{lem:ccs_snr} provides a simple, convex reformulation for obtaining $\bm{w}_t^a$ for all seemingly suboptimal arms $a \not\in \argmax_{b \in [K]}\hat\mu_{t-1}(b)$. To avoid solving for the fractional SNR maximization objective, our reformulation in Lemma \ref{lem:ccs_snr} uses a Charnes-Cooper-Schaible transform \citep{charnes_cooper} to recast our problem as a linear objective function with second-order constraints. To solve for $\tilde{\bm{w}}_t^a$, one can pick among the plethora of modern standard second-order cone program (SOCP) solvers (\cite{MOSEK, diamond2016cvxpy}). 

\subsection{Theoretical Guarantees}

To ensure our confidence sequences control satisfy the guarantees of Equation \eqref{eq:conf_seq_requirement}, we provide mild, sufficient conditions under which our confidence sequences $C_t(t_0, H_t, \alpha)$ provide asymptotic error control.

\begin{theorem}[Type I Error Control]\label{thm:error_control} Let Assumptions \ref{assump:unique_optimal_arm}, \ref{assump:nonzero_variances}, and \ref{assump:bounded_outcomes} be in full force, and let the following assumptions hold in an almost-sure sense with respect to trajectories $(H_t)_{t \in \NN}$: 
\begin{itemize}
    \item[(A1)] Convergent Sampling with Strict Positivity: $\exists \pi_\infty$ such that $\|\pi_t(x,a)- \pi_\infty(x,a)\|_{L_2(P_{X|H-1})} = o(1)$ for all $a \in [K]$, and there exists a $\kappa < \infty$ s.t. $1/\pi_t(x,a) \leq \kappa$ for all $t \in \NN, x\in\Xcal, a \in [K]$.
    \item[(A2)]  Convergent, Bounded Regression Function: $\exists g_\infty$ such that $\|g_t(x,a) - g_\infty(x,a)\|_{L_2(P_{X|H_{t-1}})}^2 = o(1)$ for all $a \in [K]$, and there exists $B$ such that $|g_t(x,a)| \leq B$ for all $t \in \NN,x\in\Xcal, a \in [K]$.
    \item[(A3)] Invertibility of Limiting Covariance Matrix: Assume that the limiting covariance matrix $\Sigma_\infty$ is invertible, where the $(i,j)$-th entry of $\Sigma_\infty$ is $\Sigma_\infty(i,j) = \EE_{P_\infty}\left[ \left(\phi_\infty(i) - \mu(i)\right) \left(\phi_\infty(j) - \mu(j)\right) \right]$, $\phi_\infty(a) = g_\infty(X,a) + \frac{\mathbf{1}[A=a](Y - g_\infty(X,a))}{\pi_\infty(X,a)}$, and $P_\infty = P_X \times P_{A \sim \pi_\infty(X, \cdot)} \times P_{Y|A,X}$ denotes the limiting distribution. 
\end{itemize}
Then, for every $\rho >0$, $\alpha \in (0,1)$, and $\bm{w}_0^a \in \Delta(a)$ for all $a\in [K]$,  the confidence sequence $C_t(t_0, H_t,\alpha) = \{a \in [K]: \sup_{t_0\leq i \leq t} L_i^a(H_i, \alpha, \rho) \leq 0 \}$ provides asymptotic anytime-valid error control, i.e. 
\begin{equation}
    \limsup_{t_0 \rightarrow \infty} P\left(\exists t \geq t_0 : a^* \not\in C_t(t_0, H_t, \alpha) \right)\leq \alpha.
\end{equation}

\end{theorem}

Theorem \ref{thm:error_control} provides standard regularity conditions to ensure our confidence sequences $(C_t)_{t=1}^\infty$ protect error rates as intended. Condition $(A1)$ corresponds to standard positivity and convergence constraints on the sampling schemes, similar to existing approaches based on scores $\phi_t(b)$ \citep{cook2024semiparametricefficientinferenceadaptive, kato2025efficientadaptiveexperimentaldesign}. Condition $(A2)$ requires conditional regression functions $g_t$ to remain bounded, which naturally follows from Assumption \ref{assump:bounded_outcomes}, and the existence of an $L_2$ almost-sure limit for $g_t$. Note that $g_\infty$ does not need to be the true conditional regression function $g$ for Theorem \ref{thm:error_control} to hold. Lastly, condition $(A3)$ provides sufficient conditions for our SNR-Maximizing weights $\bm{w}_t^a$ to converge almost surely to a limiting weight $\bm{w}_\infty^a$ for each $a \in [K]$.

In particular, the convergence of our weight sequence $\bm{w}_t^a$ guarantees that our procedure will reject all suboptimal arms $a \neq a^*$ at some time $t < \infty$ for all fixed choices of $t_0 \in \NN$. As a result, we obtain that under the same conditions, Algorithm \ref{alg:bai_short} with SNR-maximizing weights terminates in finite time for all fixed $t_0 \in \NN$. Combined with the error control of Theorem \ref{thm:error_control}, this implies that our confidence sequence-based BAI approach in Algorithm \ref{alg:bai_short} satisfies the asymptotic $\alpha$-correctness requirements of Definition \ref{defn:asymp_alpha_correctness}.

\begin{lemma}[Asymptotically Valid BAI]\label{lem:asymp_valid_bai}
    Assume that all conditions of Theorem \ref{thm:error_control} hold. Then, for every $\rho >0$, $\alpha \in (0,1)$, and any choice of $\bm{w}_0^a \in \Delta(a)$ for each $a\in [K]$, Algorithm \ref{alg:bai_short} with $\bm{w}_t^a$ set by Algorithm \ref{alg:snr_maximization} is an asymptotically $\alpha$-correct BAI algorithm, where the sequence of algorithms $\{\Bcal_{t_0}\}_{t_0 \in \NN}$ is Algorithm \ref{alg:bai_short} initialized with parameter $t_0$ and all other parameters ($\rho, \alpha, \{\bm{w}_0^a\}_{a \in [K]}$) fixed. 
\end{lemma}

Beyond valid error control, Lemma \ref{lem:asymp_valid_bai} states that Algorithm \ref{alg:bai_short} terminates almost surely for any fixed choice of $t_0 \in \NN$. To better characterize the sample complexity of Algorithm \ref{alg:bai_short}, we present an upper bound in Theorem \ref{thm:sample_complexity_general}, which holds both almost surely and in expectation.

\begin{theorem}[Sample Complexity Under General $\pi$]\label{thm:sample_complexity_general}
Let the assumptions of Theorem \ref{thm:error_control} be in full force. Let $t_0(\alpha)$ be a sequence of burn-in times that satisfy (i) $t_0(\alpha) \rightarrow \infty$ as $\alpha \rightarrow 0$ and (ii) $\lim_{\alpha \rightarrow 0} t_0(\alpha)/\log(1/\alpha) = 0$. Let $\bm{w}_\infty^{a^*}= \bm{w}_0^{a^*}$, and $\forall a \neq a^*$, let $\bm{w}_\infty^a = \argmax_{\bm{w} \in \Delta(a)}\sum_{b \in [K]} {w}(b)\mu(b)/\sigma_\infty(\bm{w})$, where we denote the limit variance $\sigma^2_\infty(\bm{w}) = \EE_{P_\infty}\left[ \left(\sum_{b \in [K]} w(b)\left(\phi_\infty(b) - \mu(b)\right) \right)^2\right]$ with $\phi_\infty(b), P_\infty$ defined as in Theorem \ref{thm:error_control}. Let $\tau_{t_0(\alpha)}$ denote the (random) number of samples before Algorithm \ref{alg:bai_short} with $t_0 = t_0(\alpha)$ terminates, and $\Gamma_1 = \left(\min_{a \neq a^*}\frac{\sum_{b \in [K]} w_\infty^a(b)\mu(b)}{\sigma_\infty(\bm{w}_\infty^a)}\right)^{-2}$ denote twice the squared inverse of the minimum SNR across all suboptimal arms $a \neq a^*$. Then, for all fixed choices of $\rho >0, \bm{w}_0^a \in \Delta(a)$ for $a \in [K]$, 
\begin{equation}
    \lim_{\alpha \rightarrow 0}\frac{\EE_P[\tau_{t_0(\alpha)}]}{\log(1/\alpha)} \leq 2\Gamma_1, \quad P\left(\lim_{\alpha \rightarrow 0} \frac{\tau_{t_0(\alpha)}}{\log(1/\alpha)} \leq 2\Gamma_1\right) = 1.
\end{equation} 
\end{theorem}

Theorem \ref{thm:sample_complexity_general} establishes that the normalized number of samples $\tau/\log(1/\alpha)$ is asymptotically bounded both in expectation and almost surely by $2\Gamma_1$, twice the squared inverse of the \textit{smallest SNR ratio} across all suboptimal arms $a \neq a^*$. Crucially, Theorem \ref{thm:sample_complexity_general} provides a natural choice for our sampling scheme. By using a policy $\pi$ that maximizes the minimum SNR across all suboptimal arms, we equivalently minimize $\Gamma_1$, the asymptotic sample complexity bounds for our confidence sequence approach.

\begin{remark}[Asymptotic Order of Burn-in Times]
    To obtain the guarantees of Theorem \ref{thm:sample_complexity_general}, we place two restrictions on the burn-in time $t_0$. The first condition requires the burn-in time parameter $t_0 \rightarrow \infty$ as error tolerance $\alpha \rightarrow 0$, which ensures that $\limsup_{\alpha \rightarrow 0} P\left(\hat{a}_{t_0(\alpha)} \neq a^* \right) = 0$. This follows from the results of Theorem \ref{thm:error_control}, which ensures uniform error control for $\alpha \in (0,1)$ when the burn-in time parameter $t_0$ diverges towards infinity. Our second condition requires $t_0(\alpha)$ to be of order $o\left(\log(1/\alpha)\right)$, which ensures that the burn-in time $t_0(\alpha)$ is negligible with respect to the sample complexity bounds, which are of order $\log(1/\alpha)$. 
\end{remark}

\section{Optimized Sampling for Exploration}
\label{sec:sampling}

Given the results of Theorem \ref{thm:sample_complexity_general}, the natural choice of sampling scheme $\pi$ aims to minimize $\Gamma_1$, the inverse minimum squared signal-to-noise ratio across all suboptimal arms $a \neq a^*$. To characterize the optimal solution, we first rewrite the bound $\Gamma_1$ as the objective function $G(\pi)$, making our dependence on $\pi$ explicit:
\begin{align}\label{eq:G_pi}
    &G(\pi) = \max_{a \neq a^*} F_a(\pi), \quad F_a(\pi) = \min_{\bm{w} \in \Delta(a), \bm{w}^\top \bm{\mu} \geq 0} f(\pi, \bm{w}), \\
    f(\pi, \bm{w}) =& \frac{\EE_{P_\infty}\left[\left( \sum_{b \in [K]} w(b)\left(g_\infty(X,b) + \frac{\mathbf{1}[A=b](Y - g_\infty(X,b))}{\pi(X,b)} -\mu(b)\right) \right)^2 \right]}{\left(\sum_{b \in [K]} w(b) \mu(b)\right)^2}.
\end{align}

The function $f(\pi, \bm{w})$ corresponds to the squared inverse SNR ratio for a fixed weight and policy. The function $F_a(\pi)$ then minimizes $\bm{w}$ for that fixed policy $\pi$ for a given arm $a$. Lastly, $G(\pi)$, our objective function, is the maximum inverse squared SNR (equivalently, inverse of the minimum squared SNR) across all suboptimal arms $a \neq a^*$, matching the almost-sure and expected sample complexity bound $\Gamma_1$. 

Our optimization problem involves minimizing the functional $G: \Pi \rightarrow \RR_+$ with respect to the function $\pi$, where $\Pi \coloneqq \left\{\pi(x,b) \geq 0, \ \sum_{b \in [K]}\pi(x,b) = 1 \ P_X\text{-a.s.} \right\}$ denotes the set of all valid policies.\footnote{Our definition of the policy class $\Pi$ may be replaced with a stricter policy class that enforces $\pi(x,b)> 0$. However, under Assumption \ref{assump:nonzero_variances}, the optimal solution $\pi_*$ satisfies $\pi(x,b) > 0$ for all $x \in \Xcal, b \in [K]$; otherwise, the objective value diverges towards infinity due to  $\pi(x,b)^{-1}$ terms. Therefore, we allow our policy class $\Pi$ to include zero propensity scores.} To reduce the space of functions $\pi \in \Pi$, we first establish that the objective function $G(\pi)$ is a strictly convex functional with respect to the function $\pi$ and therefore has a unique minimizing $\pi_*$.

\begin{lemma}[Strict Convexity of $G(\pi)$]\label{lem:convexity_policy}
    Let Assumptions \ref{assump:unique_optimal_arm}, \ref{assump:nonzero_variances}, and \ref{assump:bounded_outcomes} hold. Then, the function $G(\pi)$ is strictly convex with respect to $\pi \in \Pi$, i.e. $G(\pi)$ has a unique minimizing $\pi_* \in \Pi$.    
\end{lemma}

\begin{proof}[Proof Sketch of Lemma \ref{lem:convexity_policy}] The strict convexity of $G(\pi)$ follows from a four step argument. First, we derive the Fréchet Hessian $D^2_{\pi}f_a(\pi, \bm{w})[u,h]$, where $u,h \in L_2(P_X:\RR^K)$ are square integrable functions with respect to the norm $\| f\|_{L_2(P_X:\RR^K)} \coloneqq \sqrt{\int_{x} \sum_{b \in [K]} |f(x,b)|^2 dP_X(x)}$. Second, we establish that for any \textit{fixed} $\pi$, for all $a \in [K]$, the weight vector $\bm{w} \in \{\bm{w}' \in \Delta(a): \bm{w}'^\top \bm{\mu} \geq 0\}$ that minimizes the function $f(\pi, \bm{w})$ is unique. Third, we apply Danskin's Theorem \citep{bonnans2000perturbation} on the function $F_a(\pi)$ to obtain the Fréchet derivative of $F_a(\pi)$ with respect to $\pi$. Using this derivative, we show that $F_a(\pi)$ has a positive definite Hessian on the interior of $\Pi$, and is therefore strictly convex. To conclude, we note that the maximum of strictly convex functions is strictly convex, and therefore $G(\pi) = \max_{a \neq a^*} F_a(\pi)$ is strictly convex. Because the optimal minimizing $\pi_*$ lies in the interior of the policy set $\Pi$, it follows that $\pi_*$ must be unique. 
\end{proof}

The strict convexity results of Lemma \ref{lem:convexity_policy}, paired with the fact that our set $\Pi$ is defined with only linear equality/inequality constraints, ensures Slater's condition holds. Thus, the Karush–Kuhn–Tucker (KKT) conditions characterize the optimal solution. From the KKT conditions, we obtain that the optimal policy $\pi_*$ reduces into a simple form that only depends on the conditional variance function $v(a,x)$, residual errors $r_\infty(x, a)$, and a real-valued vector $\bm{\theta} \in \RR^K$. We provide the structure of optimal policy $\pi_*$ in Lemma \ref{lem:optimal_policy_structure} below.

\begin{lemma}[Structure of Optimal Policy]\label{lem:optimal_policy_structure}
    Let Assumptions \ref{assump:unique_optimal_arm}, \ref{assump:nonzero_variances}, and \ref{assump:bounded_outcomes} hold, and assume conditions $(A2), (A3)$ of Theorem \ref{thm:error_control} hold. Then, $\exists \bm{\theta}_* \in \RR^K$ with $K$-th coordinate ${\theta_*}(K) = 0$ such that $\pi_* = \argmin_{\pi \in \Pi} G(\pi)$ satisfies 
    \begin{equation}
        \pi_*^{-1}(x,b) = \sum_{a \in [K]} \sqrt{\frac{{V(x,a)}}{{V(x,b)}}} \exp(\theta_*(a) - \theta_*(b)),  
    \end{equation}
    where $V(x,a)= v(x,a) + r_\infty^2(x,a)$, and $r_\infty(x,a)=g_\infty(x,a) - g(x,a)$ denotes the conditional error between the limiting regression function $g_\infty$ and the true regression function $g$. 
\end{lemma}

Lemma \ref{lem:optimal_policy_structure} provides an explicit characterization of the optimal policy $\pi_*$ that substantially simplifies our policy learning task. Under our assumptions, learning the optimal policy function $\pi$ reduces to estimating (i) conditional variances $v(x,a)$, (ii) limiting residual error function $r_\infty(x,a)$, and (iii) the vector $\bm{\theta} \in \RR^K$. In the following section, we provide a sampling scheme that minimizes the empirical objective function $G_t(\pi)$ at each time $t$. Our empirical objective function leverages $\Fcal_{t-1}$-measurable running estimates of conditional variances and regression function error and projected subgradient descent (PSGD) to estimate $\bm{\theta}_*$. Under mild convergence conditions, we demonstrate that our sampling scheme $\pi$ satisfies the regularity conditions of Theorem \ref{thm:error_control}, ensuring asymptotic error control and sample complexity upper bounds via Theorem \ref{thm:sample_complexity_general}.

\subsection{Sampling via Projected Subgradient Descent}
Following the optimal policy structure provided in Lemma \ref{lem:optimal_policy_structure}, our proposed sampling scheme $\pi_t$ takes the form
\begin{equation}
    \pi_t^{-1}(x,b) = \sum_{a \in [K]} \sqrt{\frac{V_t(x,a)}{V_t(x,b)}} \exp(\theta_t(a) - \theta_t(b)),
\end{equation}
where $\bm{\theta}_t$ and $V_t$ denote $\Fcal_{t-1}$-measurable estimates of the function $V(x,a)$ and $\bm{\theta}_*$ as defined in Lemma \ref{lem:optimal_policy_structure}. Our policy learning approach proceeds in the following two-step procedure, with pseudocode provided in Algorithm \ref{alg:proj-subgradient}. At each time $t$, we first construct the function $V_{t}: \Xcal \times [K] \rightarrow \RR_{++}$, an estimate for the sum of the conditional variance $v(a,x)$ and limiting residual error $r_\infty(x,a)$ using previous observations $H_{t-1}$. To obtain $\bm{\theta}_t$, we then run projected subgradient descent on $G_t(\bm{\theta})$, which substitutes unknown quantities with $\Fcal_{t-1}$-measurable estimates. Below, we expand on each step of our procedure, beginning with our function $V_t$. 

\begin{algorithm}[t]
\caption{Sampling Policy via Subgradient Descent}
\label{alg:proj-subgradient}
\begin{algorithmic}[1]

\Procedure{PolicyLearning}{$H_{t-1}, S, \bm\theta_0, N, \epsilon, g_t$}

\State \textbf{Require:} $\epsilon > 0$, $S\geq 0$, $\bm\theta_{0}\in [-S,S]^K$, $\theta(K) = 0$, $N\in\mathbb{N}$.

\Comment{\texttt{Step 1: Conditional Variance Estimation}}

\State Compute $\tilde{Y}_i = (Y_i - g_t(X_i,A_i))^2$, the squared residual between outcomes and regression function $g_t$. 

\State Regress squared residuals $(\tilde{Y}_i)_{i \in [t-1]}$ with respect to $(X_i, A_i)_{i \in [t-1]}$ to obtain $\tilde{V}_t$. 

\State Truncate $\tilde{V}_t$ to obtain $V_t(x,a) = \max\left(\tilde{V}_t(x,a), \epsilon\right)$ for $x \in \Xcal, a \in [K]$. 

\Comment{\texttt{Step 2: Projected Subgradient Descent}}

\For{$n \in [N]$} 

    \State Compute 
    $\displaystyle 
    \bm{w}^{a}_{n} = 
    \arg\min_{\substack{\bm{w}\in\Delta(a), \bm{w}^\top \hat{\boldsymbol{\mu}}_{t-1}\ge 0}}
    f_{t}(\theta_n,\bm{w})
    \quad \text{for all } a \not\in \arg\max_{b\in[K]} \hat{\mu}_{t-1}(b)$.\label{l:compute_weights}

    \State Compute the active arms set
    $\displaystyle 
    \mathcal{A}_n = \{\, a\in [K] : F_{a,t}(\theta_n)=\max_{b\in[K]} F_{b,t}(\theta_n)\,\}$. \label{l:compute_active_set}

    \State Choose subgradient
    $\displaystyle
    \bm{d}_n = \frac{1}{|\mathcal{A}_n|}
    \sum_{a\in\mathcal{A}_n} \nabla_{\theta} f_{t}(\bm\theta_n,\bm{w}^{a}_{n})$.\label{l:compute_subgradient}

    \State Set $\displaystyle \theta_{n+1} \gets 
    \Pi_{[-S, S]^{K-1}}\!\left(\theta_n - \frac{1}{n \| \bm{d}_n \|_2} \bm{d}_n\right)$. \label{l:projection_step}

\EndFor
\State Set $\displaystyle  \bm{\theta}_t = \arg\min_{i\in[N]} G_t(\bm\theta_i)$.

\State \Return $\displaystyle \pi_t(x,b)$, where $\pi_t^{-1}(x,b) = \sum_{a \in [K]} \sqrt{\frac{V_t(x,a)}{V_t(x,b)}}\exp\left(\theta_t(a) - \theta_t(b)\right)$.

\EndProcedure
\end{algorithmic}
\end{algorithm}

\subsubsection{Construction of Conditional Variance Estimator}

Our conditional regression function $V_t(x,a)$ aims to estimate the function $V(x,a) = v(x,a) + r_\infty^2(x,a)$ by first constructing pseudo-outcomes $\tilde{Y}_i = (Y_i - g_t(X_i, A_i))^2$. Assuming that $g_t(\cdot, x)$ converges to $g_\infty(\cdot, x)$ in $L_2(P_X)$ almost surely for all $a \in [K]$, the pseudo-outcomes $\tilde{Y}_i$ correspond to observations with conditional expectation $V_t(x,a)= v(x,a) + r_\infty^2(x,a)$ as $t$ diverges towards infinity, i.e. 
\begin{align}
    \lim_{t \rightarrow \infty}\EE_{P_{Y|A,X}}\left[\left(Y - g_t(x,a)\right)^2 | X=x, A=a\right] = v(x,a) + r_\infty(x,a)^2 = V(x,a).
\end{align}
After constructing our pseudo-outcomes $\tilde{Y}_i$, we regress $(\tilde{Y}_i)_{i < t}$ on observed contexts and arm indices $(X_i, A_i)_{i < t}$ to obtain the function $\tilde{V}_t$. Similar to our regression function $g_t$, our regression function $\tilde{V}_t$ may be estimated with flexible machine learning models, including random forests, neural networks, or boosting algorithms. Lastly, we enforce a \textit{minimum} value $\epsilon > 0$ on the function $\tilde{V}_t$ to obtain $V_t$, i.e. 
\begin{equation}
    V_t(x,a) = \begin{cases}\tilde{V}_t(x,a) & \text{if } \  \tilde{V}_t(x,a) \geq \epsilon \\ \epsilon  & \text{if } \ \tilde{V}_t(x,a) < \epsilon \end{cases}.
\end{equation}

\begin{remark}[Truncation of Conditional Variance Estimator]
    One may wonder why the additional truncation step is necessary for our estimates $V_t$. The truncation of our initial estimate $\tilde{V}_t(x,a)$ by a strict margin $\epsilon$ not only avoids degenerate values in our empirical objective function $G_t(\bm\theta)$, but also (i) simplifies subgradient computation, (ii) ensures strict positivity on our sampling probabilities $\pi_t$, and (iii) ensures convergence of subgradient descent for estimating our parameter $\bm{\theta}_t$. We elaborate on the role of truncation in Appendix \ref{app:proofs}.
\end{remark}

\subsubsection{Parameter Estimation via Projected Subgradient Descent}

Using our estimated functions $(V_t)_{t\in \NN}$, we run projected subgradient descent (PSGD) on the empirical objective function $G_t(\bm{\theta})$, which substitutes unknown quantities with $\Fcal_{t-1}$-measurable estimates. Below, we define $G_t(\bm{\theta})$, our empirical analogue to the true objective function $G(\pi)$, parameterized with respect to $\bm{\theta}$:
\begin{align}\label{eqs:empirical_counterpart}
    G_t(\bm{\theta}) &=  \max_{a: \hat\mu_{t-1}(a) < \max_{b \in [K]} \hat\mu_{t-1}(b)} F_{a,t}(\bm\theta), \\
    F_{a,t}(\bm\theta) &= \min_{\bm{w} \in \Delta(a), \bm{w}^\top \hat{\bm\mu}_{t-1} \geq 0} f_{t}(\bm{\theta}, \bm{w}), \\
    f_{t}(\bm\theta, \bm{w}) &= \frac{ \sum_{b \in [K]} \frac{w(b)^2}{t} \sum_{i=1}^t\left[V_i(X_i,b)\sum_{a \in [K]}\frac{ \sqrt{V_i(X_i,a) } }{\sqrt{V_i(X_i,b) } }\exp\left(\theta(a) - \theta(b)\right)\right] + l_t(\bm{w})}{\left(\sum_{b \in [K]} w(b)  \hat\mu_{t-1}(b) \right)^2} \\
    l_t(\bm{w}) &= \frac{1}{t}\sum_{i=1}^t\left[ \left(\sum_{b \in [K]} w(b) (g_i(X_i,b)- \hat\mu_{t-1}(b)) \right)^2 \right], \label{eqs:empirical_counterpart_2}
\end{align}
where $\hat\mu_{t-1}(a) = \frac{1}{t-1}\sum_{i=1}^{t-1} \phi_i(a)$ denotes our $\Fcal_{t-1}$-measurable mean estimate. To parse our projected subgradient descent approach in Algorithm \ref{alg:proj-subgradient}, we first show that (i) our objective function $G_t(\bm{\theta})$ is strictly convex with respect to $\bm{\theta}$ and (ii) the subgradient set of $G_t(\bm{\theta})$ is characterized as follows.

\begin{lemma}[Subgradient Set of $G_t(\bm{\theta})$]\label{lem:subgradient_lemma}
    Let $V_t$ be constructed as in Algorithm \ref{alg:proj-subgradient}. Then, $G_t(\bm{\theta})$ is a strictly convex function with respect to $\bm{\theta}$, and the subdifferential set of $G_t(\bm\theta)$ at $\bm{\theta}$ is given by 
    \begin{equation}
        \partial G_t(\bm{\theta}) = \text{conv}(\{\nabla_{\bm\theta} F_{a,t}(\bm{\theta})\}_{a \in \Acal_t(\bm{\theta})}), 
    \end{equation}
    where $\text{conv}(\{\bm{x}_i\}_{i \in \Acal(\bm{\theta})})$ denotes the convex hull of vectors $\bm{x}_i$, $\Acal_t(\bm{\theta}) = \{a \in [K]: F_{a,t} = G_t(\bm{\theta})\}$, and $\nabla_{\bm\theta} F_{a,t}(\bm{\theta}) \in \RR^{K-1}$ is the gradient of function $F_{a,t}(\bm{\theta})$ evaluated at $\bm{\theta}$. The gradient is characterized by \begin{equation}\label{eq:danskin_application}
        \nabla_{\bm\theta} F_{a,t}(\bm{\theta}) = \nabla_{\bm{\theta}} f_{t}(\bm{\theta}, \bm{w}^a_{\bm{\theta}})
    \end{equation}
    where $\bm{w}^a_{\bm{\theta}}$ is the unique vector $\bm{w} \in \Delta(a)$ such that $f_{t}(\bm{\theta}, \bm{w}_{\bm\theta}^a) = F_{a,t}(\bm{\theta})$, and $\nabla_{\bm{\theta}} f_{t}(\bm{\theta}, \bm{w}^a_{\bm{\theta}}) \in \RR^{K-1}$ has $c$-th entry $\frac{\partial}{\partial \theta(c)} f_{t}(\bm\theta, \bm{w}_{\bm\theta}^a) =  \sum_{b \in [K]}\frac{\frac{1}{t}\sum_{i=1}^t\sqrt{V_i(X_i,b)V_i(X_i,c)} }{\left(\sum_{b \in [K]} w_{\bm{\theta}}^a(b) \hat\mu_{t-1}(b)\right)^2} \left(  w_{\bm\theta}^a(b)^2\exp(\theta(c) - \theta(b)) - w_{\bm\theta}^a(c)^2 \exp(\theta(b)-\theta(c))\right)$.
\end{lemma}

Lemma \ref{lem:subgradient_lemma} states that the subgradient set of our empirical objective $G_t$ is simply the convex hull of vectors $\{\nabla_{\bm\theta}f_{t}(\bm{\theta}, \bm{w}_{\bm\theta}^a)\}_{a \in \Acal_t(\bm{\theta})}$. The vector $\nabla_{\bm\theta}f_{t}(\bm{\theta}, \bm{w}_{\bm\theta}^a) \in \RR^{K-1}$ corresponds to the gradient of functions $f_t(\bm{\theta}, \bm{w}_{\bm\theta}^a)$ with respect to $\bm{\theta}$, evaluated at SNR-maximizing weights $\bm{w}_{\bm\theta}^a$. These result follow from a similar approach to the proof of Lemma \ref{lem:convexity_policy}. First, by combining the uniqueness of $\bm{w}_{\bm\theta}^a$ for each fixed $\bm{\theta}$ and Danskin's Theorem, we obtain that the function $F_{a,t}(\bm{\theta})$ has a unique gradient equivalent to $\nabla_{\bm{\theta}} f_t(\bm{\theta}, \bm{w}_{\bm\theta}^a)$ for all $a \in [K]$. Because $G_t(\bm{\theta})$ selects the maximum $F_{a,t}(\bm{\theta})$ over indices $a \not\in \argmax_{b \in [K]}\hat\mu_{t-1}(b)$, the convex hull of all gradients $\nabla_{\bm\theta} F_{a,t}(\bm{\theta})$ that satisfy $F_{a,t}(\bm{\theta}) = G_t(\bm{\theta})$ characterizes our subgradient set. 

Importantly, these results provide a recipe for PSGD on our empirical objective $G_t(\bm{\theta})$. Our subgradient computation is provided in lines \ref{l:compute_weights}-\ref{l:compute_subgradient} of Algorithm \ref{alg:proj-subgradient}. In line \ref{l:compute_weights}, we estimate the SNR-Maximizing weight $\bm{w}_n^a$ with respect to $\bm{\theta}_n$, the current value of $\bm{\theta}$ at the $n$-th iterate of PSGD. Note that we only compute $\bm{w}_n^a$ for all $a \not\in \argmax_{b \in [K]}\hat\mu_{t-1}(b)$ due to the fact that any $a \in \argmax_{b \in [K]}\hat\mu_{t-1}(b)$ cannot achieve the minimum SNR, and $\bm{w}_n^a$ can be computed with SOCP solvers as in Lemma \ref{lem:ccs_snr} using the objective value $\max_{\bm{w} \in \Delta(a)}f_t^{-1/2}(\bm\theta, \bm{w})$, which corresponds to the SNR-maximization problem using estimated conditional variances $V_t$. In line \ref{l:compute_active_set}, we construct the set $\Acal_n$, the set of all arm indices $a$ that achieve $F_{a,t}(\bm{\theta}_n) = G_t(\bm{\theta}_n)$. Lastly, in line \ref{l:compute_subgradient}, we select the subgradient $\bm{d}_n$ that uniformly weights all gradients $\nabla_{\bm\theta}F_{a,t}(\bm{\theta})$ across $a \in \Acal_n$, and move in the opposite direction of this subgradient. Our projection step, shown in line \ref{l:projection_step}, occurs after updating our current estimate $\bm{\theta}_n$ in the direction $d_n$ with step size $1/\sqrt{N}$. Our projection operator $\Pi_{[-S, S]^K}$ merely enforces our boundedness constraints $\bm{\theta}(-K) \in [-S, S]^{K-1}$, where $\bm{\theta}_{n+1}$ has the following entries for all $a \in [K-1]$:
\begin{equation}
    \theta_{n+1}(a) = \min\left(S, \max\left({\theta}_n(a) + \frac{d_n(a)}{n \| \bm{d}_n \|_2}, -S \right)\right). 
\end{equation}
Similar to the truncation of the conditional variance estimator, our coordinate-wise bounds $[-S, S]$ ensure (i) strict positivity of the sampling scheme $\pi_t$ and (ii) bounds on the norm of each gradient $g_n$. In particular, the second result ensures that our PSGD procedure converges to the unique optimal $\bm{\theta}_t^*$ that maximizes $G_t(\bm{\theta})$ over the set $\bm{\Theta} = \left\{\bm{\theta} \in \RR^K: \theta(K) = 0, \bm{\theta}(-K) \in [-S,S]^{K-1}\right\}$ as the number of iterations $N$ approaches infinity.

\subsection{Theoretical Guarantees with Adaptive Sampling}\label{subsec:theoretical_guarantees}

Our choice of step size $(n \|\bm{d}_n\|_2)^{-1}$, truncated variance estimator $V_t$, and coordinate-wise bounds $\theta(a) \in [-S, S]$ for all $a \in [K-1]$ ensures that Algorithm \ref{alg:proj-subgradient} converges almost surely to a limiting $\bm{\theta}_\infty$. In Theorem \ref{thm:projected_subgradient_descent_convergence}, we provide mild conditions regarding the boundedness and convergence of $V_t$ that ensure our sampling policy sequence $(\pi_t)_{t \in \NN}$ converges almost surely to a limiting policy ${\pi}_\infty$.

\begin{theorem}[Convergence of Learning Policy]\label{thm:projected_subgradient_descent_convergence}
    Let Assumptions \ref{assump:unique_optimal_arm}, \ref{assump:nonzero_variances}, \ref{assump:bounded_outcomes} and condition $(A2), (A3)$ of Theorem 1 hold. Furthermore, assume  $\exists B < \infty$ such that $|V_t(x,a)| \leq B^2 $and $\exists V_\infty$ such that $\lim_{t \rightarrow \infty}\|V_t(\cdot, a) - V_\infty(\cdot, a) \|_{L_2\left(P_{X|H_{t-1}}\right)} = 0$ almost surely for all $a \in [K]$. Let ${\Theta} = \left\{\bm{\theta} \in \RR^K: \theta(K) = 0, \bm{\theta}(-K) \in [-S,S]^K\right\}$. 
    Let $\pi_\infty$ be the policy with entries $\pi_\infty(x,b) = \left(\sum_{a \in [K]}\sqrt{\frac{V_\infty(x,a)}{V_\infty(x,b)}}\exp\left(\theta_\infty(a) - \theta_\infty(b)\right)\right)^{-1}$, where $\bm{\theta}_\infty$ is the unique vector that minimizes the function $G_\infty(\bm{\theta})= \max_{a \neq a^*} F_{a, \infty}(\bm{\theta})$, and 
    \begin{align}
        F_{a,\infty}(\bm{\theta}) &= \min_{\bm{w} \in \Delta(a), \bm{w}^\top \bm{\mu} \geq 0} f_{\infty}(\bm\theta, \bm{w}), \\
        f_\infty(\bm{\theta, \bm{w}}) &= \frac{\EE_{P_X}\left[ \sum_{b \in [K]} \left(w^2(b) V_\infty(X,b) \sum_{a \in [K]}\sqrt{\frac{V_\infty(X,a)}{V_\infty(X,b)}}\exp\left(\theta(a) - \theta(b)\right)\right) \right] + l_\infty(\bm{w})}{\left(\sum_{b \in [K]}w(b)\mu(b) \right)^2}, \\
        l_\infty(\bm{w}) &= {\EE_{P_X}\left[ \left(\sum_{b \in [K]} w(b)\left(g_\infty(X, b) - \mu(b) \right) \right)^2 \right]}.
    \end{align}
     Let the number of descent iterations $N(t)$ be an increasing function of $t$, such that $N(t) \rightarrow \infty$ as $t \rightarrow \infty$. Then, for all $\epsilon > 0$, $S \geq 0$, and $\bm{\theta}_0 \in {\Theta}$, (i) there exists a $\kappa > 0$ such that $\pi_t(x,a) \geq 1/\kappa$ for all $t \in \NN$, $x \in \Xcal$, $a \in [K]$, and (ii) $\lim_{t\rightarrow \infty}\|\pi_t(\cdot,a) - \pi_\infty(\cdot,a) \|_{L_2\left(P_{X|H_{t-1}}\right)} = 0$ almost surely. 
     
     Furthermore, if $V(a,x) \geq \epsilon$ for all $a \in [K]$, $x \in \Xcal$ $P_X$-almost surely, $\bm{\theta}_* \in \Theta$, where $\bm{\theta}_*$ is defined as in Lemma \ref{lem:optimal_policy_structure}, and the limiting function $V_\infty$ equals $V$, then $\pi_\infty = \pi_* = \argmin_{\pi \in \Pi} G(\pi)$, i.e. $\pi_\infty$ converges to the optimal policy $\pi$ that minimizes the sample complexity bound $\Gamma_1$ in Theorem \ref{thm:sample_complexity_general}. 
\end{theorem}

Beyond previous assumptions, Theorem \ref{thm:projected_subgradient_descent_convergence} requires that  $V_t(x,a)$ is uniformly bounded by some constant $B^2 < \infty$, and there exists an $L_2$ almost-sure limit $V_\infty$ for the random sequence $(V_t)_{t \in \NN}$. These conditions are analogous to condition $(A2)$ in Theorem \ref{thm:error_control} on the regression function $g_t$. Under these assumptions, Theorem \ref{thm:projected_subgradient_descent_convergence} states that the policy $\pi_t$, estimated with $N(t)$ descent iterations at each time $t$, satisfies the necessary conditions for Theorem \ref{thm:error_control}. Our condition that $N(t) \rightarrow \infty$ as $t \rightarrow \infty$ ensures that the parameter $\bm{\theta}_t \rightarrow \bm{\theta}_\infty$ almost surely for some $\bm{\theta}_\infty \in \Theta$, ensuring that our policy $\pi_t$ converges in $L_2$ to some policy $\pi_{\infty}$. 

\begin{remark}[Comparison of Sampling Guarantees to Existing Work]
In contrast to the contextual sampling scheme for BAI proposed in \cite{kato2024rolecontextualinformationbest}, we establish conditions under which our sampling scheme converges to the optimal solution of the minimax optimization problem implied by our sample-complexity bound. The method in \cite{kato2024rolecontextualinformationbest}, by comparison, relies on off-the-shelf sequential least squares programming and does not provide guarantees on the convergence of its sampling policy or on optimal sampling complexity. To the best of our knowledge, our policy-learning procedure in Algorithm \ref{alg:proj-subgradient} is the first contextual sampling scheme for BAI that offers provable convergence guarantees to the optimal policy.    
\end{remark}

By satisfying the conditions of Theorem \ref{thm:error_control}, our BAI procedure in Algorithm \ref{alg:bai_short}, paired with our sampling scheme $\pi_t$ provided in Algorithm \ref{alg:proj-subgradient}, satisfies asymptotic $\alpha$-correctness (Lemma \ref{lem:asymp_valid_bai}), with asymptotic sample complexities characterized by Theorem \ref{thm:sample_complexity_general}. In Theorem \ref{thm:results_general}, we show that under the same conditions as Lemma \ref{thm:projected_subgradient_descent_convergence}, Algorithm \ref{alg:bai_short} paired with sampling policy $\pi_t$ in Algorithm \ref{alg:proj-subgradient} is asymptotically $\alpha$-correct.

\begin{lemma}[Asymptotic $\alpha$-Correctness under Algorithm \ref{alg:proj-subgradient}]\label{thm:results_general}
    Let all assumptions of Theorem \ref{thm:projected_subgradient_descent_convergence} hold, and define $\left(\Bcal_{t_0}\right)_{t_0 \in \RR_+}$ as the sequence of BAI algorithms with burn-in time $t_0$ and $\pi_t$ in Algorithm \ref{alg:proj-subgradient}, parameterized with $\epsilon > 0$, $S \geq 0$, iteration number $N(t)$, and $\bm\theta_0(t) \in \Theta$, where $\Theta$ is as defined in Theorem \ref{thm:projected_subgradient_descent_convergence}. Assume that the sequence of descent iterations $N(t) \rightarrow \infty$ as $t \rightarrow \infty$. Then, for all fixed $\rho > 0$, $\alpha \in (0,1)$, $\epsilon > 0$, $S \geq 0$, and initialization sequence $\{\bm{\theta}_0(t)\}_{t \in \NN}$, the sequence $(\Bcal_{t_0})_{t_0 \in \NN}$ is asymptotically $\alpha$-correct.
\end{lemma}

The conditions of Lemma \ref{thm:projected_subgradient_descent_convergence} also ensure that the results of Theorem \ref{thm:sample_complexity_general} hold, allowing for an explicit characterization of asymptotic sample complexities under our proposed sampling scheme using the limiting sampling policy $\pi_\infty$. To connect our results to (i) existing BAI sample complexity bounds and (ii) semi-parametric efficiency in average treatment effect estimation, we provide results under additional assumptions. 

\paragraph{Connections with Existing BAI Bounds}

Under stronger assumptions that assume the limiting functions $g_\infty = g$ and $V_\infty = v$, we provide minimax results that demonstrate the \textit{worst-case} sampling complexity of our approach is no larger than the \textit{best-case} sample complexity of canonical Gaussian BAI.

\begin{theorem}[Minimax Sample Complexities under Algorithm \ref{alg:proj-subgradient}]\label{thm:sampling_complexity_final}
Let all assumptions of Theorem \ref{thm:projected_subgradient_descent_convergence} hold, and assume that $g_\infty = g$, and ${V}_\infty = v$. Let $\left(\Bcal_{t_0(\alpha)}\right)_{\alpha \in (0,1)}$ be the sequence of algorithms $\Bcal_{t_0(\alpha)}$, with $\Bcal_{t_0(\alpha)}$ as defined in Lemma \ref{thm:results_general}. Let $t_0(\alpha)$ denote a sequence of burn-in times such that $t_0(\alpha) \rightarrow \infty$ and $t_0(\alpha) = o(\log(1/\alpha))$ as $\alpha \rightarrow 0$. Let $\bm{\theta}_*$ be defined as in Lemma \ref{lem:optimal_policy_structure}. Let $\Pcal(\bm{\mu}, \bm{\sigma}^2)$ denote the set of all arm distributions with means $\bm{\mu}$ and arm variances $\bm{\sigma}^2$ satisfying our assumptions, and $\Gamma_2(\bm{\mu}, \bm{\sigma}^2)$ denote
\begin{equation}\label{eq:gamma_2}
    \Gamma_2(\bm\mu, \bm\sigma^2) = \left(\sup_{\pi \in \Delta^K} \inf_{\tilde\mu \not\in \Hcal_{a^*}} \sum_{a \in [K]} \pi(a) d_{N(\cdot, \sigma^2(a))}\left(\mu(a), \tilde\mu(a)\right) \right)^{-1}.
\end{equation}
where $d_{N(\cdot, z)}$ denotes the Gaussian KL divergence function as defined in Lemma \ref{lem:kl_projection}. Let $\tau_{t_0(\alpha)}$ denote the (random) number of samples before Algorithm $\Bcal_{t_0(\alpha)}$ terminates. Then, for all $\epsilon >0$ such that $\epsilon \leq \min_{x \in \Xcal, b \in [K]}v(x,b)$ $P_X$-a.s., all $S \geq 0$ such that $\max_{b \in [K]} |\theta_*(b)| \leq S$, and $\rho >0$, we obtain   
\begin{equation}\label{eq:stopping_time_empirical_bound}
    \lim_{\alpha \rightarrow 0} \frac{\EE[\tau_{t(\alpha)}]}{\log(1/\alpha)} \leq \Gamma_2(\bm\mu, \bm\sigma^2), \quad P\left(\lim_{\alpha \rightarrow 0}\frac{\tau_{t(\alpha)}}{\log(1/\alpha)} \leq \Gamma_2(\bm{\mu}, \bm\sigma^2)\right) = 1.
\end{equation}
for any $P \in \Pcal(\bm{\mu}, \bm{\sigma}^2)$. 
Furthermore, for any $P \in \Pcal(\bm{\mu}, \bm{\sigma}^2)$ where there exists $a,b \in [K]$ and $\widetilde{\Xcal} \subseteq \Xcal$ with $P_X(X \in \widetilde\Xcal) > 0$ such that $(g(x,a)-\mu(a))(g(x,b)-\mu(b)) < 0$ for $x \in \widetilde{\Xcal}$,
\begin{equation}
    \lim_{\alpha \rightarrow 0} \frac{\EE[\tau_{t_0(\alpha)}]}{\log(1/\alpha)} < \Gamma_2(\bm\mu, \bm\sigma^2), \quad P\left(\lim_{\alpha \rightarrow 0}\frac{\tau_{t_0(\alpha)}}{\log(1/\alpha)} < \Gamma_2(\bm{\mu}, \bm\sigma^2)\right) = 1.
\end{equation}
\end{theorem}

Theorem \ref{thm:sampling_complexity_final} characterizes the worst-case sample complexity of our approach over all distributions with mean $\bm{\mu}$ and arm variances $\bm{\sigma}^2$ under the assumption that $g_t$ and $V_t$ converge to the true conditional mean and variance functions $g$ and $v$. Our condition for strict inequality corresponds to the $X$-specific heterogeneity of conditional means $g(x,a)$ relative to the marginal mean $\mu(a)$. In particular, if there exists some set $\widetilde\Xcal \subseteq \Xcal$ with positive measure where two arms achieve larger and smaller average outcomes relative to their population mean, our condition is satisfied, and our stopping time is strictly smaller than the upper bound $\Gamma_2(\bm\mu,\bm\sigma^2)$. We note that when contextual information $X$ is (i) uninformative of outcomes $Y$ or (ii) unavailable, as in the standard multi-armed bandit (MAB) setting, our strict inequality condition fails, resulting in equality in Equation \eqref{eq:stopping_time_empirical_bound}. Importantly, our strict inequality demonstrates that when conditional outcomes are heterogeneous relative to the population average, our approach strictly improves upon the best possible performance bound for standard Gaussian BAI, even with known variances.  

\begin{remark}[Connections with Existing Sample Complexities]\cite{garivier2016optimalbestarmidentification} show that the upper bound $\Gamma_1(\bm{\mu}, \bm\sigma^2)$ corresponds to the \textit{best possible sampling complexity} for $\alpha$-correct BAI (as in Definition \ref{defn:alpha_correctness}) in the setting where (i) the conditional distribution $P(Y|A=a)$ is equivalent to $N(\mu(a), \sigma^2(a))$ and (ii) arm-specific variances $\sigma^2(a)$ are known for each arm $a \in [K]$. The results of Theorem \ref{thm:sampling_complexity_final} demonstrate the benefits of our relaxed notion of error control for BAI. By relaxing the error control requirement from $\alpha$-correctness to asymptotic $\alpha$-level correctness, Theorem \ref{thm:sampling_complexity_final} demonstrates that even without contexts, best-arm identification (BAI) under the bounded outcome assumption—with unknown bounds and variances—is no more difficult than exact $\delta$-correct Gaussian BAI with known arm variances. Our conditions for strict inequality highlight the role of contextual information. In heterogeneous settings, where conditional means $g(x,a)$ differ from marginal arm means $\mu(a)$, our contextual information enables our approach to achieve strictly smaller expected sample complexities than the best possible sample complexity for Gaussian BAI without contexts.
\end{remark}

\paragraph{Connections with Adaptive Treatment Effect Estimation} Under the same assumptions as Theorem \ref{thm:error_control}, our procedure is analogous to semi-parametric efficient inference for treatment effect estimation \citep{cook2024semiparametricefficientinferenceadaptive} in the two-armed case. We demonstrate this connection in Lemma \ref{lem:two_armed} by providing closed-form expressions for the limiting sampling policy $\pi_\infty$ and the asymptotic sample complexity.

\begin{lemma}[Closed-Form Limits in the Two-Armed Case]\label{lem:two_armed}
    Let all assumptions of Theorem \ref{thm:sampling_complexity_final} hold, and let $K=2$. Let $\left(\Bcal_{t_0(\alpha)}\right)_{\alpha \in (0,1)}$ be defined as in Lemma \ref{thm:results_general}, and let the sequence $t_0(\alpha)$ satisfy $t_0(\alpha) \rightarrow \infty$ and $t_0(\alpha) = o(\log(1/\alpha))$ as $\alpha \rightarrow 0$. Then, for all $\epsilon >0$ such that $\epsilon \leq \min_{x \in \Xcal, b \in [K]}v(x,b)$ $P_X$-a.s., all $S \geq 0$ such that $\max_{b \in [K]} |\theta_*(b)| \leq S$, and $\rho >0$, the limiting sampling policy $\pi_{\infty}$ corresponds to the function 
    \begin{equation}
        \pi_\infty(x,a) = \frac{\sqrt{v(x,a)}}{\sqrt{v(x,1)} + \sqrt{v(x,2)}},
    \end{equation}
    and the asymptotic sample complexity of our approach satisfies 
    \begin{equation}
        \lim_{\alpha \rightarrow 0} \frac{\EE[\tau_{t(\alpha)}]} {\log(1/\alpha)} \leq \Gamma_2, \quad P\left(\lim_{\alpha \rightarrow 0} \frac{\tau_{t(\alpha)}} {\log(1/\alpha)} \leq \Gamma_2 \right) = 1,
    \end{equation}
    where $\tau_{t(\alpha)}$ is as defined in Theorem \ref{thm:sampling_complexity_final}, and $\Gamma_2$ is defined as 
    \begin{equation}
        \Gamma_2 = 2\left(\frac{\EE_{P_X}\left[ \left( \sqrt{v(X,1)} + \sqrt{v(X,2)} \right)^2 \right] + \EE_{P_X}\left[ \left( \ g(X,1) - \mu(1)) - (g(X, 2)- \mu(2)) \ \right)^2 \right]}{\left(\mu(1) - \mu(2)\right)^2}\right).
    \end{equation}
\end{lemma}

Lemma \ref{lem:two_armed} demonstrates that in the two-armed case, our procedure closely corresponds to adaptive estimation for semi-parametric efficient inference on the difference between arm means $\mu(1) - \mu(2)$, referred to as the treatment effect in the causal inference literature. Our limit policy $\pi_\infty$ in Lemma \ref{lem:two_armed} corresponds to the \textit{optimal} sampling policy for semi-parametric efficient inference of the treatment effect, shown by \cite{han}. The numerator of $\Gamma_2$, our asymptotic sample complexity bound, corresponds to the \textit{minimum possible variance} for a treatment effect estimator with data-dependent sampling, as shown by \cite{cook2024semiparametricefficientinferenceadaptive}.\footnote{To be precise, the numerator of $\Gamma_2$ corresponds the minimum possible variance over (i) all possible sampling policies $\pi$ and (ii) the class of regular and asymptotically linear (RAL) estimators for the treatment effect. We refer to \cite{vanderVaart1998} for a more detailed discussion on the class of RAL estimators.}

The results of Lemma \ref{lem:two_armed} shed light on how our BAI approach exploits contexts to achieve better sample complexity. Recall that our general sample complexity bound, $\Gamma_1$, is inversely proportional to the squared minimum signal-to-noise ratio (SNR) of the test processes $\hat\psi_t(a)$ for all suboptimal arms $a \neq a^*$. Thus, reducing the variances of these test processes directly improves the sample complexity of BAI. Under the regularity conditions stated above, Lemma \ref{lem:two_armed} shows that in the two-armed setting, our method minimizes these variances to the \textit{lowest possible value} permitted by our nonparametric statistical model. From this perspective, our BAI framework can be seen as a generalization of adaptive sampling techniques used for efficient treatment effect estimation, with the goal of identifying the highest mean arm instead of improving the precision of treatment effect estimates.

\begin{remark}[Additional Assumptions in Theorem \ref{thm:sampling_complexity_final} and Lemma \ref{lem:two_armed}]
    Beyond our assumptions that our limit functions satisfy $g_\infty = g$ and $V_\infty = v$, both Theorem \ref{thm:sampling_complexity_final} and Lemma \ref{lem:two_armed} require that (i) the truncation parameter $\epsilon$ is strictly smaller than the minimum conditional variance $v(a,x)$ and (ii) the optimal $\bm{\theta}_*$ has coordinates $\theta(a) \in [-S,S]$ for all $a \in [K-1]$. Note that due to Assumption \ref{assump:nonzero_variances}, there exists both an $\epsilon_* >0$ and $S_* < \infty$ that satisfies these conditions. The existence of $\epsilon_* > 0$ follows directly from Assumption \ref{assump:nonzero_variances}, and the existence of $S_* < \infty$ follows from $\pi(x,K) \rightarrow 0$ for all $x \in \Xcal$ as $\max_{a \in [K]} \bm{\theta}(a) \rightarrow \infty$, leading to an infinite value for our sample complexity. 
\end{remark}

In conclusion, for bounded outcome bandit models, our theoretical results suggest that our BAI approach provides a robust, efficient procedure for nonparametric BAI. Theorem \ref{thm:sampling_complexity_final} demonstrates that even without contexts, knowledge of outcome bounds, and arm-specific variances, asymptotic $\alpha$-correct BAI is no harder than Gaussian BAI under exact $\alpha$-correct constraints and known variances. In settings with $X$-specific heterogeneity across outcomes, our results demonstrate that asymptotic $\alpha$-correct BAI is strictly easier than Gaussian BAI with exact $\delta$-correct constraints and known variances. Lemma \ref{lem:two_armed} provides valuable insight on how our approach achieves reduced sample complexities. By leveraging contexts and adaptive sampling to achieve the \textit{smallest} possible variance on our test processes, our method generalizes semi-parametric efficient adaptive designs in causal effect estimation to the setting of BAI, resulting in \textit{efficient} sample complexities that make full use of the available contexts.

\section{Experiments}\label{sec:experiments}

To highlight the benefits of our approach, we compare our approach both with and without contexts to existing BAI approaches. In our first experiment, we compare our approach under differing mean vectors where baselines are \textit{known} to be asymptotically optimal for the given DGP. In our second experiment, we consider the case where the underlying distribution is unknown, and demonstrate that our approaches naturally adapt to the difficulty of the instance. For all experiments, we track (i) the average number of samples $\tau$ collected before declaring an arm as best and (ii) the empirical probability that the returned arm is suboptimal. 

\subsection{Experiment Setup}

\paragraph{Choice of Hyperparameters/Solvers} We set our $\bm{\theta}$ bounds as $S = 100$, the variance estimate truncation constant as $\epsilon = 0.01$, the descent iterations as $N(t) = 10 + \log(t + 1)$ for each $t \in \NN$, and the burn-in time $t_0 = 100$. For all conditional mean and variance estimates, we use probit regression as implemented in \cite{seabold2010statsmodels}. To solve the convex optimization problem necessary to obtain $\bm{w}_t^a$ for both our test processes and subgradient calculations, we use SOCP solvers \texttt{CLARABEL} \citep{Clarabel_2024}, \texttt{ECOS} \citep{bib:Domahidi2013ecos}, and \texttt{SCS} \citep{odonoghue:21} at each $t$, and take the best solution as our weight. We set $\rho = 0.06$. For all methods, we set $\alpha = 0.1$.

\paragraph{Baselines}
As baselines for our approach, we compare existing fixed-confidence BAI methods. For non-contextual methods (i.e. methods that do not leverage contexts for stopping and sampling), we test algorithms Track-and-Stop (T\&S) \citep{garivier2016optimalbestarmidentification} with $D$-tracking, Chernoff stopping with top-two sampling (ChernBC) \citep{pmlr-v30-Kaufmann13}, Chernoff Racing \citep{garivier2016optimalbestarmidentification}, and ChernT3C \citep{3tc}. For contextual methods, we test contextual Track-and-Stop (CT\&S) \citep{kato2024rolecontextualinformationbest}, which provides nonasymptotic $\alpha$-correct guarantees under the assumptions of known arm variances (or upper bounds), parametric arm distributions, and finite, discrete contexts. To apply CTaS to our setting, we discretize our context space into 4 bins $\widetilde{\Xcal}= [4]$  with equal probability.\footnote{Our choice of bins is due to the relative instability of the CT\&S algorithm when the cardinality of the context set is large. Because the CT\&S algorithm estimates conditional means and variances for each context-arm pair, a large number of contexts degrades the performance of the approach significantly.} To learn the policy, we use the estimation approach used in \cite{kato2024rolecontextualinformationbest}, where the policy is estimated with sequential least squares programming (\texttt{SLSQP}) as implemented by \cite{kraft1988software}. For all methods, we test the variant corresponding to Bernoulli outcomes across all simulations, as the stopping methods for Bernoulli outcomes offer error control for the $[0,1]$-bounded outcome setting.

\paragraph{Synthetic Data Generating Processes} We test synthetic data-generating processes that vary (i) arm distributions, (ii) access to covariates, and (iii) choice of arm means. For all experiments, we use a 4-dimensional context vector $X \in \mathbb{R}^4$, with the marginal context distribution $P_X$ set as the standard multivariate normal distribution $N(0, I_4)$. Matching the experimental set-up of \cite{garivier2016optimalbestarmidentification}, we test the arm mean vectors $\bm{\mu}_1 = [0.5, 0.45, 0.43, 0.4]$ and $\bm{\mu}_2 = [0.3, 0.21, 0.2, 0.19, 0.18]$. For our conditional distributions $P_{Y|A,X}$, we consider both Bernoulli and mixture-Beta outcomes, with three distinct conditional distributions for each distribution type. Our Bernoulli and mixture-Beta outcomes denote the high and low variance settings respectively. For our Bernoulli and mixture-Beta settings, we set $P_{Y|A,X}$ as
\begin{align}
    P_{Y|A,X} &= \text{Bern}\left(\Phi\left(c(A) + \sum_{i=1}^4 X(i)\right)  \right), \\  P_{Y|A,X} &= \text{Beta}\left( \Phi\left(c(A) + \sum_{i=1}^4 X(i)\right), 1-\Phi\left(c(A) + \sum_{i=1}^4 X(i)\right)\right),
\end{align}
respectively, where $\bm{c}_1 = [0, -0.28, -0.39, -0.57]$ and $\bm{c}_2 = [-1.17, -1.80, -1.88, -1.96, -2.05]$ correspond to mean vectors $\bm{\mu}_1$ and $\bm{\mu}_2$ and $\Phi(\cdot)$ denotes the CDF of the standard normal distribution. To assess the value of covariates in contextual BAI approach, we run our method both with and without contexts, allowing for fair comparison across our contextual and non-contextual baselines respectively. 

\subsection{Discussion of Results}

In Figure \ref{fig:exp_1_bernoulli}, we provide the average number of samples for each method for mean vectors $\bm{\mu}_1$ and $\bm{\mu}_2$ under the Bernoulli and Beta setting, with standard deviations of our estimates shown in the error bar. Across all methods and distributions, the realized error rate reached a maximum of 0.02, well below the nominal level $\alpha = 0.1$, including our asymptotic approaches with burn-in time of $t_0=100$. These results suggest that even with relatively small burn-in times, the realized error rate remains far below the nominal level. 

\begin{figure}[h]
    \centering
    \includegraphics[width=\linewidth]{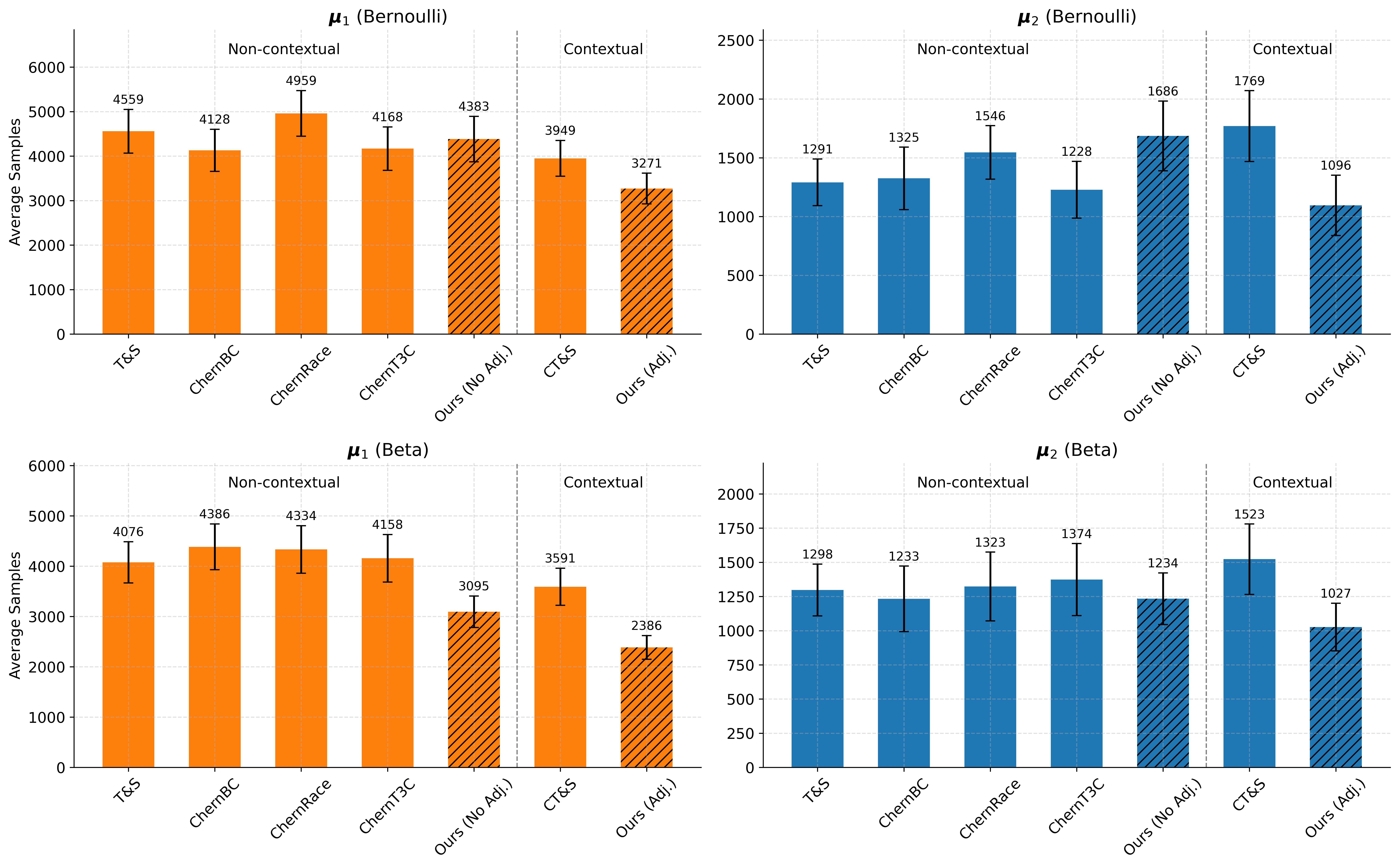}
    \caption{Average number of samples under Bernoulli and Beta conditional outcome distributions. Error bars are $\pm1$ standard deviation for estimated average sample complexity over 100 simulations.}
    \label{fig:exp_1_bernoulli}
\end{figure}

\paragraph{Comparison with Existing Optimal Approaches.}
We use our Bernoulli outcome results to test our approach against asymptotically optimal BAI approaches. Note that in this Bernoulli setting, the T\&S and CT\&S baselines obtain asymptotically optimal sample complexity for non-asymptotic BAI without contexts and finite context set $\widetilde\Xcal$ respectively. The results in the top row of Figure \ref{fig:exp_1_bernoulli} demonstrate that our asymptotic approaches provide comparable (if not better) sample complexities to existing asymptotically optimal methods, with larger reductions in average samples under more difficult arm instances. 

For both experiments, our approach without contexts performs similarly to the best non-contextual methods, including the asymptotically optimal T\&S approach. Our approach with contexts achieve the smallest average sample complexity across both mean vectors and all methods. For $\bm{\mu}_1$ and $\bm{\mu}_2$, our approach reduces sample complexity by roughly 20\% and 10\% relative to the next best method respectively. This result suggests that our approach offers the most benefits for \textit{harder} instances. Relative to $\bm{\mu}_2$, the mean vector $\bm{\mu}_1$ has both higher arm variances and smaller arm gaps, resulting in long horizons that allow our nuisances (e.g. weights, sampling policy, conditional mean/variance estimates) to stabilize over time.

\paragraph{Adapting to the Underlying Distribution.} A key benefit of our asymptotic approach is its ability to \textit{adapt} to the underlying arm distributions. To demonstrate the benefits of relaxed error guarantees, we test our methods under Mixture-Beta arm distributions. For our baselines, we assume that the experimenter knows outcomes are bounded between $[0,1]$, ensuring the validity of our baselines using Bernoulli stopping rules. For our asymptotic approaches, our approaches do not depend on knowledge of outcome bounds/moments, and leverage running estimates of arm variances for the sampling and stopping rules. 

Our results presented in the second row of Figure \ref{fig:exp_1_bernoulli} demonstrate that our asymptotic approaches naturally adapt to the difficulty of the instance. Compared to the Bernoulli instances, note that the conditional Beta distributions have reduced variance, resulting in a smaller sample complexity lower bound. Among non-contextual methods, our non-contextual approach achieves the smallest sample complexity, with a reduction of up to 25\% in sample complexity relative to the best baseline. Note that this reduction is achieved solely by our asymptotic error relaxation, which enables learned variances. In contrast, non-asymptotic methods assume worst-case variance bounds to ensure valid error control (specified by outcome/moment bounds). As a result, non-asymptotic BAI approaches require larger sample complexities than necessary when the underlying distribution is not worst-case. By leveraging contexts, our approach achieves the smallest sample complexity across all tested approaches. Compared to all non-contextual baselines (excluding our approach), our approach with contexts provides up to a 50\% reduction in samples; compared to CT\&S, our approach provides up to a 33\% reduction in samples. 

Similar to our Bernoulli experiments, we observe the largest improvements with $\bm{\mu}_1$, demonstrating that our approach offers the most practical benefit when the underlying instance is more difficult. Our non-contextual and contextual approaches provide significant sample complexity reductions for $\bm{\mu_1}$, resulting in 33\% smaller sample complexities compared to the best baseline. In contrast, our non-contextual approach achieves similar performance to the non-contextual baselines for $\bm{\mu}_2$, while our contextual approach achieves a 17\% reduction in average samples compared to the best baseline. As in the Bernoulli case, more difficult instances allow for our nuisances to converge, enabling our approach to achieve larger gains in performance.

\section{Conclusion and Future Directions}\label{sec:conclusion}

In this work, we propose a new framework for best-arm identification that relaxes classical fixed-confidence guarantees to hold only beyond a growing burn-in period, reflecting the long-horizon nature of practical experiments. Building on this relaxation, we develop novel asymptotic anytime-valid confidence sequences over arm indices, enabling efficient elimination of suboptimal arms under fully nonparametric outcome models with unknown contextual structure. To complement these stopping rules, we propose a sampling procedure based on projected subgradient descent that allocates samples to minimize asymptotic stopping time. Relative to existing approaches in the BAI literature, our asymptotic approach can seamlessly incorporate infinite-dimensional contextual information and does not require parametric (e.g. exponential family) assumptions. 

Our theoretical results show that, under mild convergence assumptions, the worst-case sample complexity of our method matches the sample complexity lower bound for Gaussian BAI with known variances. Under stronger assumptions of conditional mean consistency, conditional variance consistency, and informative covariates, the asymptotic sample complexity of our approach is \textit{strictly} smaller than that of Gaussian BAI. Empirical evaluations demonstrate sample efficiency gains up to 33\% over existing methods, particularly for bandit instances that require larger horizon experiments. 

Our work provides both (i) immediate results for similar exploration problems in bandits and (ii) future directions of investigation. We list several of these implications and future directions below. 

\begin{itemize}
    \item \textbf{Applications to Alternative Exploration Problems}: Our asymptotic framework for exploration can immediately be applied to similar bandit problems, such as threshold identification \citep{cho2024rewardmaximizationpureexploration}. By leveraging asymptotic, anytime-valid confidence intervals for the mean of each arm, similar results, such as sample complexity reduction using contextual information and worst-case bounds matching Gaussian sample complexity lower bounds, follow directly from the proofs provided.
    \item \textbf{Computationally Lightweight Variants}: While leveraging pretrained models and batched updates may reduce computational costs in terms of model training, our procedure requires us to leverage optimization methods for finding the weight sequences and sampling parameters. To reduce computation costs further, a closed-form, heuristic choice of weights $\bm{w}_t$ and sampling parameter $\bm{\theta}_t$ may be desirable. 
    \item \textbf{Extensions to Continuous Actions/Policies}: Beyond discrete action spaces, one may wish to find the best action in a continuous or infinite-dimensional set, such as the best \textit{personalized} policy with continuous contexts. We believe such an extension is possible by allowing our weights to be a function, and relaxing our best-arm condition to $\epsilon$-best. We leave this direction for future work.
\end{itemize}

As a cautionary note, in settings where the outcomes follow parametric assumptions, experiment horizons are typically short, and exact guarantees are desired, we note that our method does not guarantee the best arm at the nominal level and may have worse performance than existing methods. However, in many modern applications, such as digital experiments, horizons are typically long, contexts are collected, and outcomes follow unknown, nonparametric distributions. For such settings, our approach provides a tailored solution for bandit exploration that provides both theoretical and empirical performance gains.

\newpage

\bibliographystyle{abbrvnat}
\bibliography{main}

\newpage
\begin{appendices}
\renewcommand{\thesection}{A.\arabic{section}} 

\section*{Appendix}

\section{Notation}\label{app:notation_table}

\begin{description}[
  leftmargin=2.2cm,
  labelwidth=2cm,
  labelsep=.2cm,
  parsep=0mm,
  itemsep=2pt
  ]
\item[$\RR$] the set of all real numbers
\item[$\RR_+$] the set of all nonnegative real numbers
\item[$\RR_{++}$] the set of all strictly positive real numbers
\item[$\NN$] the set of all natural numbers
\item[$\Fcal_t$] the canonical filtration at time $t$; $\Fcal_t =\sigma ((A_i, X_i)_{i=1}^t)$, where $\Fcal_0$ denotes the empty sigma field
\item[$\alpha$] error tolerance parameter, where $\alpha \in [0,1]$
\item[$H_T$] set of all observations $(X_i, A_i, Y_i)_{i=1}^T$ collected up to time $T \in \NN$, where $H_0$ is the empty set. 
\item [$P_X$] fixed, unknown distribution that characterizes the distributions of contexts $X_t$ for all $t \in \NN$.
\item [$P_{Y|A,X}$] fixed, unknown distribution that characterizes the conditional distribution of $Y_t$ for all $t \in \NN$
\item [$P$] the instance of the bandit problem, defined by $P = (P_X, P_{Y|A,X})$. 
\item [$\Xcal$] set of possible contexts, allowed to be the empty set
\item[${[}K{]}$] set of integers $1,\ldots, K$, where $K$ is the total number of arms
\item[$\Delta^K$] probability simplex over the K arms
\item[$\bm{w}(-i)$] the vector $\bm{w} \in \RR^K$ with the $i$-th component removed
\item[$\Delta(a)$] the set of vectors $\{\bm{w} \in \RR^K: w(a) = -1, \bm{w}(-a) \in \Delta^{K-1}\}$
\item[$\pi$] the mapping $(H_{t-1}, \Xcal) \rightarrow \Delta^K$ that determines sampling probabilities at time $t$.
\item[$\pi_t$] the conditional sampling policy at time $t$, i.e. $\pi_t(x,a) = P(A_t = a| X_t = x, H_{t-1})$
\item[$\bm\mu$] the vector of arm means, where $\mu(a)= \EE_{P_X}\left[ \EE_{P_{Y|A,X}}\left[Y| A= a, X\right] \right]$.
\item[$\bm\sigma^2$] the vector of arm variances, where $\sigma^2(a) = \EE_{P_X}\left[\EE_{Y|A,X}\left[\left(Y - \mu(a) \right)^2 | A=a, X \right] \right]$.
\item[$g(x,a)$] the expectation of outcomes conditional on context $X=x$ and arm $A=a$, i.e. $g(x,a) = \EE_{P_{Y|A,X}}\left[Y | A=a, X=x\right]$
\item[$v(x,a)$] the variance of outcomes conditional on context $X=x$ and arm $A=a$, i.e. $v(x,a) = \EE_{P_{Y|A,X}}\left[ \left(Y - g(x,a)\right)^2 | A=a, X=a\right]$
\item [$a^*$] the unique arm $a^* \in [K]$ such that $a^* = \argmax_{a \in [K]}\mu(a)$
\item[$\|f \|_{L_q(P_{H_{t-1}})}$] the conditional $L_q$ norm, where $\|f \|_{L_q(P_{H_{t-1}})} = \EE\left[ |f|^q | H_{t-1} \right]$ 
\item[$\Bcal$] fixed-confidence best arm identification algorithm $\Bcal = (\pi, f, \hat{a})$, where $\pi$ denotes the sampling scheme, $f: H_t \rightarrow \{0,1\}$ denotes a binary decision to stop at each time $t \in \NN$, $\hat{a} \in [K]$ denotes the estimated best arm returned when $f(H_t) = 1$ (i.e. procedure stops). 
\item[$\phi_t(b)$] unbiased score function for arm $b \in [K]$ at time $t$, where $\phi_t(b) = g_t(X_t, b) + \frac{\mathbf{1}[A_i = b](Y_t - g_t(X_t, b))}{\pi_t(X_t, b)}$ and $g_t: \Xcal \times [K] \rightarrow \RR$ is an $\Fcal_{t-1}$-measurable function. 
\item [$\hat\mu_t(a)$] running estimate of the mean of arm $a$, where $\hat\mu_t(a) = \frac{1}{t}\sum_{i=1}^t \phi_i(a)$
\item [$\hat\sigma_t^2(\bm{w})$] cumulative conditional variance estimate $\hat{\sigma}_t^2(\bm{w}) = \frac{1}{t}\sum_{i=1}^t \left( \sum_{b\in [K]} w(b)\left(\phi_i(b) - \hat\mu_t(b)\right) \right)^2$ for fixed weight vector $\bm{w} \in \RR^K$ up to time $t$. 
\item [$(\bm{w}_t^a)_{t=1}^\infty$] signal-to-noise (SNR) maximizing weights, where $\bm{w}_t^a \in \argmax_{\bm{w}\in\Delta(a)} \frac{\sum_{b \in [K]}w(b) \hat\mu_{t-1}(b)}{\hat\sigma_{t-1}(\bm{w})}$ if $\hat\mu_{t-1}(a) < \max_{b \in [K]}\hat\mu_{t-1}(b)$, and $\bm{w}_t^a = \bm{w}_0^a$ otherwise, where $\bm{w}_0^a \in \Delta(a)$ is specified in advance. 
\item [$\left(\hat\psi_t(a)\right)_{t=1}^\infty$] arm-specific score process adapted to $\left(\Fcal_t\right)_{t=1}^\infty$, where $\hat\psi_t(a) = \frac{1}{t}\sum_{i=1}^t \sum_{b \in [K]}w_i^a(b) \phi_i(b)$.
\item[$\hat\sigma^2_t(a)$] the estimated cumulative conditional variance for the score process $\left(\hat\psi_i(a)\right)_{i=1}^t$ corresponding to arm $a$, i.e. $\hat\sigma^2_t(a) = \frac{1}{t}\sum_{i=1}^t \left(\sum_{b \in [K]} w_i^a(b) \left(\phi_i(b) - \hat\mu_i(b)\right)\right)^2$ 
\item[$\ell_{t,\alpha,\rho}(x)$] asymptotic anytime-valid lower bound $\ell_{t,\alpha,\rho}(x) = t^{-1/2}\sqrt{\frac{2(\rho^2+1/tx^2)}{\rho^2}\log\left(1+\frac{\sqrt{t x^2\rho^2 + 1}}{2\alpha}\right)}$, where $\alpha \in [0,1]$, $\rho > 0$,  and $t \in \NN$. 
\item [$\Hcal_a$] the set of distributions $P = (P_X, P_{Y|A,X})$ such that arm $a$ achieves the largest mean, i.e. $\EE_{P_X}\left[ \EE_{P_{Y|A,X}}\left[Y|A=a, X\right]\right] = \max_{b \in [K]}\EE_{P_X}\left[ \EE_{P_{Y|A,X}}\left[Y|A=b, X\right]\right]$. 
\item [$\kappa$] inverse of minimum sampling probability at each time $t \in \NN$, where $1/\pi_t(x,a) \leq \kappa$. 
\item [$\bm\phi_T$] score matrix $\bm\phi_T \in \RR^{T\times K}$, where the $(t,k)$-th entry corresponds to centered score $\phi_t(a) - \hat\mu_T(a)$.  
\item[$\Gamma_1$] the squared maximum of the limiting inverse SNR across all suboptimal arms
\item[$\pi_t^\kappa$] the proposed sampling scheme in Algorithm \ref{alg:bai_short}
\item[$\Pcal(\bm{\mu}, \bm{\sigma}^2)$] the set of all distributions $P = (P_X, P_{Y|A,X})$ with arm means/variances $\bm{\mu}$ and $\bm{\sigma}^2$ respectively
\item[$\Gamma_2(\bm{\mu}, \bm{\sigma}^2)$] the minimum sampling complexity for Gaussian BAI under $\alpha$-correct error constraints, where $\Gamma_2(\bm{\mu}, \bm{\sigma}^2) = \left(\sup_{\pi \in \Delta^K} \inf_{\tilde\mu \not\in \Hcal_{a^*}} \sum_{a \in [K]} \pi(a) \frac{\left(\mu(a) - \tilde\mu(a)\right)^2}{2\sigma^2(a)} \right)^{-1}$

\end{description}

\newpage

\section{Proofs}\label{app:proofs}
In this section, we provide proofs for all theorems and lemmas presented in the main body of the paper. We begin with preliminary lemmas used in the steps of our proofs, and then provide proofs for our main results. 

\subsection{Preliminary Lemmas}\label{appsec:existing_lemmas}

To recast as our SNR-maximization problem as a simple convex optimization problem, we leverage the Charnes-Cooper-Schaible Transform below. We apply this transform to obtain the results of Lemma \ref{lem:ccs_snr}. 

\begin{lemma}[Charnes-Cooper-Shaible Transform \citep{schaible2016fractional}]\label{lemma:ccs_transform}
Assume that $\Xcal \subseteq \RR^n$ is a convex set, and let $f$ and $g$ be nonnegative concave and strictly positive convex functions respectively on the set $\Xcal$. Let $h$ denote our constraints, such that the feasible region is defined as $S = \{\bm{x} \in \Xcal: h(\bm{x}) \leq 0\}$. Then, the maximization problem $\sup_{x \in S} f(x) / g(x)$ is equivalent to the following:
\begin{align}
    \sup_{t \in \RR, y \in \RR^n} \quad& t f(y/t)\\
    \text{s.t. } 
    & th(y/t) \leq 0,  \\
    & tg(y/t) \leq 1, \\ 
    & y/t \in \Xcal, \\
    & t > 0. 
\end{align}
\end{lemma}

To show that estimated sequences (such as our SNR-maximizing weights) converge almost surely, we leverage a version of Theorem 3.2.2 by \cite{vandervaart1996weak} under the conditions of \cite{white1984asymptotic}, replacing the convergence in distribution condition with almost sure convergence. 

\begin{lemma}[Strong Consistency of Argmax \citep{vandervaart1996weak}]\label{lem:conv_argmax}
    Let $\Theta \subset \RR^{K-1}$ be compact. Then, assume there exists a sequence of random functions $f_n$ and a deterministic function $f$ such that $\sup_{\theta \in \Theta} |f_n(\theta) - f(\theta) | \rightarrow 0$ almost surely, each $\hat\theta_n \in \argmax_{\theta \in \Theta} f_n(\theta)$ for all $n \in \NN$, and $\theta_* \in \argmax_{\theta \in \Theta} f(\theta)$ is unique. Then, $\hat\theta_n \rightarrow \theta_*$ almost surely.  
\end{lemma}
In our proofs, we leverage Lemma \ref{lem:conv_argmax} to ensure that (i) our SNR-maximizing weights $(\bm{w}_t^a)_{t=1}^\infty$ converges almost surely to the limiting weight vector $\bm{w}_\infty^a$ and (ii) our sampling scheme $\pi_t$ converges almost surely to $\pi_\infty$. We leverage Theorem 2.8 of \cite{waudbysmith2024timeuniformcentrallimittheory} to establish our asymptotic error control. Below, we provide a succinct version of their results adapted for our set-up. 

\begin{lemma}[Theorem 2.8 of \cite{waudbysmith2024timeuniformcentrallimittheory}]\label{thm:waudby_smith}
        Let $(Z_t)_{t=1}^\infty$ be a sequence of random variables with conditional means $\mu_t \coloneqq \EE[Z_t|(Z_i)_{i=1}^{t-1}]$ and conditional variances $\sigma_t^2 \coloneqq \text{Var}(Z_t|(Z_i)_{i=1}^{t-1})$. Let $\tilde\sigma_t^2$ be an estimator of cumulative variances $\frac{1}{t}\sum_{i=1}^t \sigma_t^2$. Assume that the following conditions (B1), (B2), and (B3) hold in an almost-sure sense:
        \begin{itemize}
            \item[(B1)] Cumulative Variance Divergence: $\sum_{t=1}^T \sigma_t^2 \rightarrow \infty$,
            \item[(B2)] Bounded $2+\delta$ Moment: $\exists \delta >0, \ \ell < \infty$ s.t. $\EE\left[|Z_t - \mu_t|^{2+\delta}|(Z_i)_{i=1}^{t-1}\right] \in [1/\ell, \ell)$ for all $t \in \NN$,
            \item[(B3)] Polynomial rate variance estimation: $\exists \eta \in (0,1)$ s.t. $\tilde\sigma_t^2 - \frac{1}{t}\sum_{i=1}^t \sigma_i^2 = o\left(\frac{(\sum_{i=1}^t \sigma_i^2)^\eta}{t}\right)$.
        \end{itemize}
        Then, 
        $\lim_{t_0 \rightarrow \infty}P\left(\exists t \geq t_0, \ \frac{1}{t}\sum_{i=1}^t\mu_i \leq \frac{1}{t}\sum_{i=1}^t Z_t - \sqrt{\frac{2(t\tilde\sigma_t^2\rho_{t_0}^2 + 1)}{t^2 \rho_{t_0}^2}\log\left(\frac{\sqrt{t\tilde\sigma_t^2 \rho_{t_0}^2+1}} {2\alpha} \right)}\right) \leq \alpha.$ 
\end{lemma}

We introduce two additional results regarding the convergence of martingale difference sequences and Cesaro means, which ensure that the conditions of Lemma \ref{thm:waudby_smith} are satisfied. 

\begin{lemma}[Martingale Law of Iterated Logarithm \citep{stout1974almost}]\label{lem:stout_lil}
        Let $\{Z_i, \Fcal_i\}_{i\in \NN}$ be a sequence of martingale differences, where $S_t = \sum_{i=1}^t Z_i$ is the martingale and $V_t = \sum_{i=1}^t\EE[Z_i^2|\Fcal_{i-1}]$ is the predictable quadratic variation. Assume that $\EE[Z_t|\Fcal_{t-1}]= 0$, $V_t \rightarrow \infty$ and there exists a $\delta>0$ such that $\sum_{i=1}^t \frac{\EE[|Z_i|^{2+\delta}|\Fcal_{i-1}]}{V_t^{1+\delta/2}} \rightarrow 0$ as $t \rightarrow \infty$ almost surely. Then, $\limsup_{t \rightarrow \infty} |S_t|/\sqrt{2V_t\log\log V_t} = 1$ almost surely. 
\end{lemma}

Lemma \ref{lem:stout_lil} provides mild conditions for controlling the behavior of our score processes. To provide analogous guarantees for the estimated variance of the score processes, we leverage a classical result from \cite{hall2014martingale}. We provide a simplified version of this result in Lemma \ref{lem:hall_martingale} below. 

\begin{lemma}[Theorem 2.18 of \cite{hall2014martingale}]\label{lem:hall_martingale}
Let $\{S_n = \sum_{i=1}^t X_i, \Fcal_t, t\geq  1\}$ be a martingale with conditionally zero-mean increments, and assume there exists a $\beta > 1/2$ such that $\lim_{t \rightarrow 0}\sum_{i=1}^t \frac{1}{i^{2\beta}} \EE[X_i^2|\Fcal_{i-1}] \leq \infty$ almost surely. Then, $\lim_{t \rightarrow \infty} \frac{1}{t^{\beta}} S_t =0$ almost surely. 
\end{lemma}

To ensure our running means and variances match the behavior of limiting process, we require control over cesaro means. To do so, we leverage Lemma \ref{lem:cesaro_convergence} below.

\begin{lemma}[Almost-sure convergence of Cesaro Means (Proposition 3 of \cite{bibaut2020sufficientinsufficientconditionsstochastic})]\label{lem:cesaro_convergence}
    If $t^\beta X_t \rightarrow 0$ almost surely, then for $\bar{X}_t  \coloneqq \frac{1}{t}\sum_{i=1}^t X_i$, $t^\beta\bar{X}_t \rightarrow 0$ almost surely. 
\end{lemma}

Lemma \ref{lem:cesaro_convergence} enables the rates of Lemma \ref{lem:stout_lil} to apply directly to our running mean sums, which will be applied to show that Condition $(A3)$ of Lemma \ref{thm:waudby_smith} holds for our setup. To ensure that our sampling scheme in Algorithm \ref{alg:proj-subgradient} converges, we leverage Lemmas \ref{lem:convergence_bandit} and \ref{lem:mslln} below. 

\begin{lemma}[Fact E.1, \citealp{fact_e1}]\label{lem:convergence_bandit}
        Suppose that $Y_n \rightarrow Y$ a.s. as $n \rightarrow \infty$, and $N(t) \rightarrow \infty$ a.s. as $t \rightarrow \infty$. Then $Y_{N(t)} \rightarrow Y$ a.s. as $t \rightarrow \infty$.
\end{lemma}

\begin{lemma}[Martingale Strong Law of Large Numbers \cite{hall2014martingale}]\label{lem:mslln}
    Let $(X_t, \Fcal_t)_{t \in \NN}$ denote a discrete-time martingale difference sequence, where $\EE[X_t|\Fcal_{t-1}] = 0$ for all $t \in \NN$. If $\lim_{t \rightarrow \infty}\sum_{i=1}^t \EE[X_i^2]/t^2 < \infty$, then $\lim_{t \rightarrow \infty}\frac{1}{t}\sum_{i=1}^t X_i = 0$ almost surely.  
\end{lemma}

Lastly, we use Lemma \ref{lem:unique_optima} below to ensure that our weights $(\bm{w}_t^a)_{t=1}^\infty$ converge to the limiting weight $\bm{w}_\infty^a$ in Theorem \ref{thm:sample_complexity_general}. For completeness, we provide a compact proof of Lemma \ref{lem:unique_optima} below.

\begin{lemma}[Unique Optima of Ratio Function]\label{lem:unique_optima}
Let $\theta \in \Theta $ be a compact, convex set. Let $f(\theta)$ be affine, and $g(\theta)$ be strictly convex and positive. Then, $\theta_* = \argmax_{\theta \in \Theta}  f(\theta)/g(\theta)$ is unique whenever $\max_{\theta \in \Theta}  f(\theta)/g(\theta)>0$. 
\end{lemma}

\begin{proof}[Proof of Lemma \ref{lem:unique_optima}.]
   We prove this result by contradiction. Note that because $\Theta$ is a compact set, there exists a maximizer for the expression $h(\theta) = f(\theta)/g(\theta)$. Assume there exists two maximizers $\theta_1 , \theta_2 \in \Theta$ such that $h(\theta_1) = h(\theta_2) = M$, where $M=\max_{\theta \in \Theta}h(\theta) > 0$. By convexity of our set $\Theta$, note that for any $\lambda \in (0,1)$, $\theta_{\lambda} \coloneqq \lambda\theta_1 + (1-\lambda)\theta_2 \in \Theta$. By $f(\theta)$ being affine, we have $f(\theta_\lambda) = \lambda f(\theta_1) + (1-\lambda)f(\theta_2)$, and by definition of $h(\theta)$, $f(\theta_{\lambda}) = M(\lambda g(\theta_1) + (1-\lambda)g(\theta_2))$. Because $g(\theta)$ is strictly convex and positive, $g(\theta)$ satisfies
   \begin{equation}
       g(\theta_{\lambda}) < \lambda g(\theta_1) + (1-\lambda)g(\theta_2).
   \end{equation}
   Evaluating the function $h$ at $\theta_\lambda$, we obtain the contradiction $h(\theta_\lambda) > M = \max_{\theta \in \Theta}h(\theta)$,
   $$ h(\theta_{\lambda}) = \frac{f(\theta_\lambda)}{g(\theta_\lambda)} > \frac{M(\lambda g(\theta_1) + (1-\lambda)g(\theta_2))}{\lambda g(\theta_1) + (1-\lambda)g(\theta_2)} = M.$$
   Therefore, there cannot exist two solutions to $\max_{\theta \in \Theta} h(\theta)$, and the maximizing value $\theta$ is unique. 
\end{proof}

Using our preliminary lemmas, we prove all lemmas and theorems presented in the main body of our work. 

\subsection{Proof of Lemma \ref{lem:kl_projection}}

\begin{proof}[Proof of Lemma \ref{lem:kl_projection}]
    To get the desired equality, we first re-express our original maximization problem as its Lagrangian dual form. Note that our original problem takes the form
    \begin{equation}
        \bm{w}_*^a = \argmax_{\bm{w} \in \Delta^{K-1}} \frac{\sum_{a'\neq a} w(a')\left(\mu(a') - \mu(a) \right)}{\sqrt{\frac{\sigma^2(a)}{\pi_*(a)}+\sum_{a' \neq a}\frac{w^2(a')\sigma^2(a')}{\pi_*('a)} }}.
    \end{equation}
    To prove our equality, we first establish basic properties about the KL-divergence minimization problem. Note that the minimization objective given by the KL-divergences expands to
    $$
            \inf_{\tilde{\bm{\mu}} \in \Hcal_a} \sum_{b \in [K]} \pi_*(b) d_{\sigma(b)}(\mu(b), \tilde\mu(b)) = \inf_{\tilde{\bm{\mu}} \in \Hcal_a}\sum_{b \in [K]} \pi_*(b)\frac{(\mu(b) - \tilde{\mu}(b))^2}{2\sigma^2(b)}, 
    $$
    which is a convex optimization problem bounded from below that satisfies Slater's conditions. As a result, we obtain that this problem has no duality gap, i.e. its primal is equal to its dual. Thus, we can re-express the primal minimization problem with its Lagrangian dual, which is equivalent to
    $$
    g(\bm{\gamma}) = \min_{\tilde{\bm{\mu}} \in \RR^K}\Lcal(\tilde{\bm\mu}, \bm{\gamma}) = \min_{\tilde{\bm{\mu}} \in \RR^K} \left(\sum_{b \in [K]}\frac{(\tilde\mu(b) - \mu(b))^2}{2\sigma^2(b)/\pi_*(b)} + \sum_{a' \neq a}\gamma(a')(\tilde\mu(a')- \tilde\mu(a))\right).
    $$
    To solve this minimization problem, we use the first order conditions of this problem, given by:
    $$
        \frac{\partial}{\partial\mu(b)} \Lcal(\tilde{\bm{\mu}}, \bm\gamma) = \frac{\tilde\mu(b) - \mu(b)}{\sigma^2(b)/\pi_*(b)} + \mathbf{1}[b\neq a]\gamma(b) - \mathbf{1}[b=a]\sum_{a' \neq a} \gamma(a') = 0. 
    $$
    Solving this inequality, we obtain that $\gamma(b) = -\frac{\tilde\mu(b) - \mu(b)}{\sigma^2(b)/\pi(b)}$ for all $b \neq a$, and $\frac{\tilde\mu(a) - \mu(a)}{\sigma^2(a) / \pi(a)} = \sum_{a' \neq a} \gamma(a')$. Subbing these expressions back into our original expression, we obtain the following expression:
    $$
    g(\bm{\gamma}) = \sum_{a' \neq a} \gamma(a')\left(\mu(a') - \mu(a)\right) -\frac{\sigma^2(a)}{2\pi_*(a)}\left( \sum_{a' \neq a} \gamma(a') \right)^2 - \left(\sum_{a' \neq a} \frac{\sigma^2(a')}{2\pi_*(a')} \gamma^2(a') \right).
    $$

Now, we show that that the maximization of the dual function, i.e. $\max_{\bm\gamma \geq 0} g(\bm\gamma)$, is equivalent to our original SNR-maximizing weight problem. First, we set $w(a') = \frac{\gamma(a')}{\sum_{a' \neq a}\gamma(a')}$ and set $S = \sum_{a' \neq a}\gamma(a')$, resulting in the following maximization problem over $\bm{w} \in \Delta^{K-1}$ and $S \in \RR$:
$$
    g(\bm\gamma) = g(\bm{w}, S) = S\sum_{a' \neq a} w(a')(\mu(a') - \mu(a)) - S^2 \left(\frac{\sigma^2(a)}{2\pi_*(a)} + \sum_{a' \neq a}\frac{\sigma^2(a')}{2\pi_*(a')}w^2(a')\right).
$$
Now, for a fixed $\bm{w} \in \Delta^{K-1}$, we note $g(\bm{w}, S)$ is a negative quadratic with respect to $S$. Then, the maximum of $g(\bm{w}, S)$ is attained when $S$ satisfies the following first-order equations:
$$
S = \frac{\sum_{a' \neq a} w(a')\left(\mu(a') - \mu(a)\right)}{\left(\frac{\sigma^2(a)}{\pi_*(a)} + \sum_{a' \neq a}\frac{\sigma^2(a')}{\pi_*(a')}w^2(a')\right)}.
$$
Plugging the result above back into $g(\bm{w}, S)$, we obtain the following equivalence:
$$ 
\max_{S \in \RR, \bm{w} \in \Delta^{K-1}} g(\bm{w}, S) = \max_{\bm{w} \in \Delta^{K-1}} \frac{1}{2}\left(\frac{\sum_{a' \neq a} \left(w(a') \left(\mu(a') - \mu(a) \right) \right)}{\left(\frac{\sigma^2(a)}{\pi_*(a)} + \sum_{a' \neq a}\frac{\sigma^2(a')}{\pi_*(a')}w^2(a')\right)}\right)^2,
$$
which is the exact statement of Lemma \ref{lem:kl_projection}. We thus conclude this proof. 
\end{proof}

\subsection{Proof of Lemma \ref{lem:ccs_snr}}

To prove Lemma \ref{lem:ccs_snr}, we first split our optimization problem into two cases: (i) the optimal solution $\bm{w}_t^a$ lies in a set where $\hat\sigma^2(\bm{w}) > 0$, and (ii) the optimal solution $\bm{w}_t^a$ lies in a set where the estimated variance $\hat\sigma^2(\bm{w}) = 0$. 

\paragraph{Case (i): Nondegenerate Solution}
Our SNR optimization problem takes the form
$$
\bm{w}_t^a \coloneqq \argmax_{\bm{w} \in \Delta(a)} \frac{\sum_{b \in [K]} w(b)\hat\mu_{t-1}(b) }{\hat\sigma_{t-1}(\bm{w})}.
$$ 
To use Lemma \ref{lemma:ccs_transform}, we first note that that the numerator is affine (and therefore concave) with respect to $\bm{w}$, and the denominator is the L2 norm with respect to the empirical measure at time $t$, and is therefore convex. By our nondegeneracy assumption, $\hat\sigma^2(\bm{w}_t^a)$ is strictly greater than zero. To satisfy the conditions of Lemma \ref{lemma:ccs_transform}, we restrict our choice of $\bm{w}$ to the region where the numerator is nonnegative, resulting in the following optimization problem:

\begin{align}
    \max_{\beta \in \RR, \bm{\gamma} \in \RR^{K-1}} \quad& \sum_{a' \neq a}\gamma(a')\left(\hat\mu_{t-1}(a') - \hat\mu_{t-1}(a) \right)\\
    \text{s.t. } 
    & \sum_{a' \neq a} \gamma(a') = \beta  \\
    & \sum_{a' \neq a}\gamma(a')\left(\hat\mu_{t-1}(a') - \hat\mu_{t-1}(a) \right) \geq 0, \label{l:unncessary}\\
    & \tilde{\sigma}_{t-1}(\bm{\gamma}) \leq 1, \\ 
    & \beta >0, \ \gamma(a') \geq 0 \quad \forall a' \neq a.
\end{align}

Note that our additional domain constraint on line \eqref{l:unncessary} to ensure non-negativity of the numerator can be removed, as the maximizer of the objective above has the same solution and value with or without the constraint in line \eqref{l:unncessary}. Additionally, note that $\beta$ is a free variable greater than or equal to zero under our constraints, reducing to the following problem:

\begin{align}
    \max_{\beta \in \RR, \bm{\gamma} \in \RR^{K-1}} \quad& \sum_{a' \neq a}\gamma(a')\left(\hat\mu_{t-1}(a') - \hat\mu_{t-1}(a) \right)\\
    \text{s.t. } 
    & \sum_{a' \neq a} \gamma(a') > 0 \label{l:unncessary_2} \\
    & \tilde{\sigma}_{t-1}(\bm{\gamma}) \leq 1, \\ 
    & \gamma(a') \geq 0 \quad \forall a' \neq a.
\end{align}

Finally, note that under the assumption that there exists an $a' \neq a$ such that $\hat\mu_{t-1}(a') - \hat\mu_{t-1}(a) > 0$, the constraint in line \eqref{l:unncessary_2} is redundant. Under the optimization problem where line \eqref{l:unncessary_2} is removed, assume that the optimal solution is when $\sum_{a' \neq a}\gamma(a') \leq 0$, which implies $\sum_{a' \neq a}\gamma(a') = 0$ by our negativity constraint. Note that this solution cannot be optimal, as one can set $\gamma(a') > 0$ until the variance is equal to one for any $a' \neq a$ such that $\hat\mu_{t-1}(a') - \hat\mu_{t-1}(a) > 0$.  This will strictly have a larger objective value, while maintaining feasibility. Therefore, we remove line \eqref{l:unncessary_2}, resulting in the desired formulation given by Lemma \ref{lem:ccs_snr}.

\paragraph{Case (ii): Degenerate Solution} In the case where the optimal solution $\bm{w}_t^a$ lies in a set $\Wcal \subseteq \Delta(a)$ where $\hat\sigma^2(\bm{w}) = 0$, our result still holds. Let $\tilde\sigma_{t-1}^2(b) = \frac{1}{t-1}\sum_{i=1}^{t-1} \left(\phi_i(b) - \hat\mu_{t-1}(b)\right)^2$ for all $b \in [K]$. Then, if $\hat\sigma^2(\bm{w}) = 0$ at the maximum SNR, then it must be that (i) $\tilde\sigma_{t-1}^2(a)= 0$ and (ii) $\exists b \neq a$ such that $\tilde\sigma_{t-1}^2(b) = 0$ and $\hat\mu_{t-1}(b) > \hat\mu_{t-1}(a)$. Let $\Acal_t^+(a) = \argmax_{b \in [K]\setminus\{a\}: \tilde\sigma^2(b) = 0} \hat\mu_{t-1}(b)$ denote the set of largest mean arms with an estimated variance of zero. By our assumption that the optimal solution $\bm{w}_t^a$ lies in a set $\Wcal \subseteq \Delta(a)$ where $\hat\sigma^2(\bm{w}) = 0$, $|\Acal_t^+(a)| \geq 1$ must hold. The optimal solution sets for $\bm{w}_t^a$ and $\tilde{\bm{w}}_t^a$ can be characterized as 
\begin{align}
    \Wcal_t^a &= \{\bm{w} \in \Delta(a): w(b) > 0 \ \forall b \in \Acal_t^+(a), w(b) = 0  \ \forall b \not\in  \Acal_t^+(a)\}\\
    \widetilde\Wcal_t^a &= \{\tilde{\bm{w}} \in \RR_+^{K-1}: \tilde{w}(b) = \infty \ \forall b \in \Acal_t^+(a), w(b) = 0  \ \forall b \not\in  \Acal_t^+(a)\}
\end{align}
respectively. For any $\bm{w} \in \Wcal_t^a$, one can construct the corresponding sequence of weight vector $\tilde{\bm{w}}_x \in \RR^{K-1}$, 
\begin{equation}
    \tilde{w}_x(b) = \begin{cases} 0 & \text{if }w(b) = 0 \\  w(b)/x & \text{if }w(b) \neq 0\end{cases},
\end{equation}
where the limit (with respect to $x \rightarrow 0$) corresponds to $\tilde{w}$, i.e. $\lim_{x \rightarrow 0}\tilde{\bm{w}}_x = \tilde{\bm{w}} \in \widetilde{\Wcal}_t^a$. By normalizing entries of vector $\tilde{\bm{w}} \in \widetilde{\Wcal}_t^a$, we obtain $\tilde{w}(b)/\sum_{b \in \Acal_t^+(a)} \tilde{w}(b) = \lim_{x \rightarrow 0} w(b)/\sum_{b \in \Acal_t^+(a)} w(b) = w(b)$, as desired. 

\subsection{Proof of Theorem \ref{thm:error_control}}
We leverage the results of Lemma \ref{thm:waudby_smith}, and show that our testing procedure satisfies all three conditions sufficient for Lemma \ref{thm:waudby_smith} to hold. To begin our proof, we first utilize the structure of our score processes $\left( \hat\psi_t(a) \right)_{t=1}^\infty$. The non-normalized score process $t\hat\psi_t(a)$ corresponds to the sum of random variables $\sum_{i=1}^t Z_i(a)$, where $Z_i(a) = \left(\sum_{b \in [K]} w_t^a(b) \phi_t(b) \right)$. We first derive the condition mean and variance for our terms $Z_t(a)$. By definition of $\phi_i(b)$ and $\bm{w}_i^a \in \Delta(a)$,
\begin{equation}
    \mu_i(a) \coloneqq \EE\left[ Z_i(a) | H_{i-1}\right] = \left(\sum_{b \neq a} w_i^a(b)\mu(b)\right) - \mu(a). 
\end{equation}
The conditional variance of $Z_i(a)$, denoted as $\sigma_i^2(a)$, is defined as 
\begin{equation}
    \sigma_i^2(a) \coloneqq \EE\left[\left(\sum_{b \in [K]} w_i^a(b) \left(\phi_i(b) - \mu(b)\right)\right)^2 \big| H_{i-1}  \right].
\end{equation}
Under the null $\Hcal_a$, note that $\mu_i(a) \leq 0$ for all $i \in \NN$. Assuming conditions $(B1)$-$(B3)$ in Lemma \ref{thm:waudby_smith} holds, for all $P \in \Hcal_a$, by definition of $\hat\psi_t(a), \hat\sigma^2_t(a)$, and $\ell_{t,\alpha,\rho}(x)$, 
\begin{equation}\label{eq:desired}
    \limsup_{t_0 \rightarrow \infty}P\left(\exists t \geq t_0, \  \frac{\hat\psi_T(\Wcal_t)}{\tilde\sigma_t(\Wcal_t)}\geq \ell_{t,\alpha, \rho}\left(\tilde\sigma_t(\Wcal_t)\right) \right) \leq \alpha,
\end{equation}
which closely resembles our test (with an additional burn-in time parameter $t_0$). Under Theorem \ref{thm:error_control}'s conditions, we demonstrate conditions $(B1)$-$(B3)$ of Lemma \ref{thm:waudby_smith} are satisfied, ensuring that Equation \eqref{eq:desired} holds.

\paragraph{Condition $(B1)$}


First, we expand the conditional variance term to obtain
\begin{align}
        \sigma_t^2(a) &= \sum_{b \in [K]}w_t^a(b)^2 \EE_{P_X}\left[\frac{v(x,b)}{\pi_t(x,b)} | H_{t-1}\right]\label{l:expected_vars} \\
        & + \EE_{P_X}\left[\left(\sum_{b \in [K]}w_t^a(b) \left(g(x,b) - \mu(b)\right) \right)^2 \right]\label{l:conditional_vars}\\
        & + \sum_{b \in [K]} w_t^a(b)^2\EE_{P_X}\left[\frac{1 - \pi_t(x,b)}{\pi_t(x,b)}r_t(x,b)^2 | H_{t-1}\right]\label{l:rem_1} \\
        & - 2\sum_{b < K, c > b} w_t^a(b)w_t^a(c)\EE_{P_X}\left[r_t(x,b)r_t(x,c)|H_{t-1}\right]\label{l:rem_2},
\end{align}
where $r_t(x,b) = g_t(x,b) - g(x,b)$ denotes the residual error of estimated conditional expectations $g_t$ from the ground truth conditional expectation function $g$. We first leverage a simple Cauchy-Schwartiz inequality to show that the sum of lines \eqref{l:rem_1} and \eqref{l:rem_2} is strictly nonnegative. We then leverage Condition $(A3)$ in Theorem \ref{thm:error_control} to show that $\sigma_t^2(a)$ is strictly larger than a constant bounded away from zero, ensuring that the cumulative sum of conditional variances $\sigma_t^2(a)$ diverges to infinity.

Let $\bm{z}_a, \bm{\gamma}_a \in \RR^K$, where $z_a(b) = \frac{w_t^a(b)r_t(x,b) }{\sqrt{\pi_t(x,b)}}$ and $\gamma_a(b) = \sqrt{\pi_t(x,b)}$. Using the Cauchy-Schwartz inequality,
\begin{equation}
     \left(\sum_{b \in [K]} w_t^a(b) r_t(x,b)\right)^2 = \left(\sum_{b \in [K]} z_a(b) \gamma_a(b) \right)^2\leq  \left(\sum_{b \in [K]} z_a^2(b)\right) \left(\sum_{b \in [K]} \gamma_a^2(b)\right) = \sum_{b \in [K]}\frac{w_t^a(b)^2r_t^2(x,b)}{\pi_t(x,b)}.
\end{equation}
Taking the expectation with respect to conditional distribution $P_{X|H_{t-1}}$, we obtain the inequality
\begin{equation}
    \sum_{b \in [K]} w_t^a(b)^2\EE_{P_X}\left[ r_t^2(x,b) |H_{t-1}\right] + 2\sum_{b< K, c > b }w_t^a(b) w_t^a(c)\EE_{P_X}[r_t(x,b)r_t(x,c)|H_{t-1}] \leq \sum_{b \in [K]}w_t^a(b)^2\EE\left[\frac{r_t^2(x,b)}{\pi_t(x,b)} | H_{t-1} \right],
\end{equation}
which ensures that the sum of the terms in lines \eqref{l:rem_1} and \eqref{l:rem_2} is strictly nonnegative. As a result, we obtain
\begin{equation}\label{eq:variance_lower_bound}
    \sigma_t^2(a) \geq \sum_{b \in [K]}w_t^a(b)^2 \EE_{P_X}\left[\frac{v(x,b)}{\pi_t(x,b)} | H_{t-1}\right] 
         + \EE_{P_X}\left[\left(\sum_{b \in [K]}w_t^a(b) \left(g(x,b) - \mu(b)\right) \right)^2 \right].
\end{equation}
To demonstrate that our conditional variance $\sigma^2_t(a)$ diverges remains bounded above zero, we leverage (i) a simple expansion using the law of total variance and (ii) the fact that $\sigma^2(a) > 0$ for all $a\in [K]$. We first construct a random variable $\tilde{Y} = \sum_{b \in [K]} w_t^a(b) Y(b)$, where $Y_b \sim P_X \times P_{Y|A=b, X}$ denotes an independent random variable, and $\bm{w}_t^a \in \Delta(a)$ is independent of $Y(b)$. By independence, the variance of $\tilde{Y}$ is 
\begin{equation}
    \text{Var}(\tilde{Y}) = \sum_{b \in [K]} w_t^a(b)^2 \sigma^2(b) > \sigma^2(a)> 0,
\end{equation}
where our inequalities follows from the fact that $\bm{w}_t^a \in \Delta(a)$ and $\sigma^2(b) > 0$ for all $b\in [K]$. By the law of total variance, we can re-express $\text{Var}(\tilde{Y})$ in a similar form to Equation \eqref{eq:variance_lower_bound}, resulting in 
\begin{equation}
    \text{Var}(\tilde{Y}) = \sum_{b \in [K]}w_t^a(b)^2 \EE_{P_X}\left[v(x,b)\right] 
         + \EE_{P_X}\left[\left(\sum_{b \in [K]}w_t^a(b) \left(g(x,b) - \mu(b)\right) \right)^2 \right] > \sigma^2(a) > 0.
\end{equation}

Because $\pi_t(x,a) \in [1/\kappa, 1)$ for all $x \in \Xcal$, $a\in [K]$, we obtain 
$$\sum_{b \in [K]}w_t^a(b)^2 \EE_{P_X}\left[\frac{v(x,b)}{\pi_t(x,b)}\right] \geq  \sum_{b \in [K]}w_t^a(b)^2 \EE_{P_X}\left[v(x,b)\right].$$ 

Thus, the conditional variance of our score process $\sigma_t^2(a)$ in Equation \eqref{eq:variance_lower_bound} is no less than the arm-specific variance $\sigma^2(a) > 0$.  Because $\sigma^2(a) > 0$ is a fixed constant independent of $t \in \NN$, we obtain $\sum_{i=1}^t \sigma^2_t(a) \geq t \sigma^2(a)$, and therefore $\sum_{i=1}^t \sigma^2_t(a)$ diverges to infinity as $t \rightarrow \infty$.

\paragraph{Condition $(B2)$}

We provide time-uniform upper and lower bounds on the $2+\delta$ moment of $Z_t(a)$ for some $\delta > 0$ to show that that condition $(B2)$ is satisfied. Our upper bound immediately follows from Assumption \ref{assump:bounded_outcomes} and Conditions $(A1), (A2)$ of Theorem \ref{thm:error_control}. The $2+\delta$ moment of  $Z_t(a)$ takes the form 
\begin{equation}
    \EE\left[\left|Z_t(a) - \mu_t(a) \right|^{2+\delta}  \big| H_{t-1} \right] = \EE\left[ \left|\sum_{b \in [K]} w_t^a(b) \left(g_t(x,b)+\frac{\mathbf{1}[A_t = b] (Y_t - g_t(x,b))}{\pi_t(x,b)} - \mu(a)\right) \right|^{2+\delta} \big| H_{t-1}\right].
\end{equation}
By the fact that $|g_t(x,b)| \leq B$, $|Y_t| \leq B$, $\bm{w}_t^a \in \Delta(a)$, and $1/\pi_t(x,b) \leq \kappa$, we obtain
\begin{align}
    \EE\left[\left|Z_t(a) - \mu_t(a) \right|^{2+\delta}  \big| H_{t-1} \right]
    &\leq \EE\left[ \left|\sum_{b \in [K]} w_t^a(b) \left(B+\mathbf{1}[A_t = b] 2B\kappa + B\right) \right|^{2+\delta} \big| H_{t-1}\right]\\
    &\leq \left(2\left(B + 2\kappa B + B\right)\right)^{2+\delta} \\
    &= (4B(1+\kappa))^{2+\delta}.
\end{align}

To construct our lower bound, recall for any probability measure $P$, $\| f\|_{L_p(P)} \leq \|f \|_{L_q(P)}$ for $p \leq q$. We can use the conditional variance to lower bound the $2+\delta^*$ moment, resulting in 
\begin{equation}
    \EE\left[\left| \sum_{b \in [K]} w_t^a(b)\left(\phi_t^{b} - \mu(b)\right) \right|^{2+\delta^*}  \biggr{|} H_{t-1}\right] \geq \sigma_t^{2+\delta^*}(a) \geq \sigma^{2+\delta}(a),
\end{equation}
where the last inequality follows from the conditional variance bounds derived for Condition $(B1)$. Thus, setting $\delta = 1/2$, the choice of $\ell = \max\left\{ (4B(1+\kappa))^{5/2}, \sigma^{5/2}(a) \right\}$ satisfies Condition $(B2)$.

\paragraph{Condition $(B3)$} To prove this condition, we first show that $\sigma_t^2(a)$ is bounded above. By following our bounds on the $2+\delta$ centered moment, we obtain
\begin{equation}
    \sigma_t^2(a)  = \EE\left[\left|Z_t(a) - \mu_t(a) \right|^{2}  \big| H_{t-1} \right] \leq (4B(1+\kappa))^{2},
\end{equation}
which is finite. By establishing an upper bound on the conditional variance $\sigma_t^2(a)$, Condition $(B3)$ reduces to showing that $\hat\sigma_t^2(a) - \frac{1}{t}\sum_{i=1}^t \sigma^2_i(a) = o(1/t^{1-\eta})$ for some $\eta \in (0,1)$. Defining $\tilde\mu_t(a) \coloneqq \sum_{b \in [K]}w_t^a(b) \hat\mu_t(a)$ as weighted sum of estimated arm means, we expand $\hat\sigma_t^2(a) - \frac{1}{t}\sum_{i=1}^t \sigma^2_i(a)$ to obtain
\begin{align}\label{eq:terms_i_ii}
        \hat\sigma_t^2(a) - \frac{1}{t}\sum_{i=1}^t \sigma_i^2(a) &= \frac{1}{t} \sum_{i=1}^t (Z_i(a) - \tilde\mu_i(a))^2 - \sigma_i^2(a)\\
        &=\underbrace{\frac{1}{t} \sum_{i=1}^t \left(Z_i(a) - \tilde\mu_i(a)\right)^2 - \left(Z_i(a) - \mu_i(a)\right)^2}_{(i)} \label{l:final_i} \\ 
        &\quad+\underbrace{\frac{1}{t} \sum_{i=1}^t(Z_i(a) - \mu_i(a))^2 - \sigma_i^2(a)}_{(ii)}\label{l:final_ii}.
\end{align}
We now show that terms $(i)$ and $(ii)$ vanish at appropriate rates satisfying Condition $(B3)$, using Lemmas \ref{lem:stout_lil} and \ref{lem:cesaro_convergence}, beginning with term $(i)$. Using the Cauchy-Schwartz inequality, we obtain
\begin{align}
    (i) &= \frac{1}{t}\sum_{i=1}^t (\tilde\mu_i(a) - \mu_i(a))^2 + 2(\tilde\mu_i(a) - \mu_i(a))(\mu_i(a) - Z_i(a))\\
    & \leq \frac{1}{t}\sum_{i=1}^t (\tilde\mu_i(a) - \mu_i(a))^2 \\ 
    &+2\left(\frac{1}{t}\sum_{i=1}^t(\tilde\mu_i(a) - \mu_i(a))^2 \right)^{1/2}\left(\frac{1}{t}\sum_{i=1}^t (\mu_i(a) - Z_i(a))^2\right)^{1/2}.
\end{align}
We now upper bound the terms $(\tilde\mu_i(a) - \mu_i(a))^2$ and $(\mu_i(a) - Z_i(a))^2$ for all $i \in \NN$. By definition of $\tilde\mu_i(a)$, $|w_i^a(b)| \leq 1$ for all $b \in [K]$, $i \in \NN$, and the Cauchy-Schwartz inequality,
\begin{equation}
    (\tilde\mu_i(a) - \mu_i(a))^2 = \left(\sum_{b \in [K]} w_i^a(b)\left(\hat\mu_i(b) - \mu(b)\right)\right)^2
\leq K\sum_{b \in [K]} \left(\hat\mu_i(b) - \mu(b) \right)^2,
\end{equation}
resulting the a simplified upper bound for term $(i)$ independent of the weight vector $\bm{w}_t^a$:
\begin{equation}\label{eq:term_i}
    (i) \leq \frac{K}{t}\sum_{i=1}^t \sum_{b \in [K]}\left( \hat\mu_i(b) - \mu(b) \right)^2 + 2\left(\frac{K}{t}\sum_{i=1}^t \sum_{b \in [K]} (\hat\mu_i(b) - \mu(b))^2 \right)^{1/2}\left(\frac{1}{t}\sum_{i=1}^t (\mu_i(a) - Z_i(a))^2\right)^{1/2}.
\end{equation}
To show that each term on the RHS of Equation \eqref{eq:term_i} vanishes at the appropriate rate, we apply Lemmas \ref{lem:stout_lil} and \ref{lem:cesaro_convergence} by leveraging the martingale structure of $t\left( \phi_t(b) - \mu(b) \right)$ for all $b \in [K]$. To apply Lemma \ref{lem:stout_lil}, we first verify its conditions. By definition, $\EE[\phi_t(b) - \mu(b)|H_{t-1}] = 0$. Each term in its corresponding conditional variance process $V_t(b) = \sum_{i=1}^t \EE[(\phi_t(b) - \mu(b))^2| H_{t-1}]$ is lower bounded by $\sigma^2(b)$ due to  
\begin{equation}
    \EE[(\phi_i(b) - \mu(b))^2| H_{i-1}] \geq \EE_{P_X}\left[ \frac{v(x,b)}{\pi_i(x,b)} | H_{i-1}\right] + \EE_{P_X}\left[ \left(g(x,b) - \mu(b)\right)^2 | H_{i-1}\right] \geq \sigma^2(b) > 0,
\end{equation}
where the inequalities above follow from the proof of Condition $(A1)$. As such $V_t(b) \geq t\sigma^2(b)$, and therefore $V_t(b) \rightarrow \infty$ almost surely as $t \rightarrow \infty$. Lastly, to satisfy the Lyapunov-style condition, note that
\begin{equation}
    \sum_{i=1}^t \frac{\EE[|\phi_i(b) - \mu(b)|^{2+\delta} |H_{i-1}]}{V_t^{1+\delta/2}(b)} \leq \sum_{i=1}^t\frac{  (2B(\kappa+1))^{2+\delta}}{t^{1+\delta/2}\sigma^{2+\delta}(b)}  = \frac{(2B(\kappa+1))^{2+\delta}}{\sigma^{2+\delta}(b)}\frac{1}{t^{\delta/2}},
\end{equation}
where the upper bound on the numerator follows from the boundedness conditions of Assumption \ref{assump:bounded_outcomes} and Theorem \ref{thm:error_control}. As $t \rightarrow \infty$, it follows that $\sum_{i=1}^t \frac{\EE[|\phi_i(b) - \mu(b)|^{2+\delta} |H_{i-1}]}{V_t^{1+\delta/2}(b)} \rightarrow 0$ almost surely for $\delta = 1$, satisfying the Lypapunov-style condition. Given that our martingale $t(\hat\mu_t(b) - \mu(b))$ satisfies Lemma \ref{lem:stout_lil}'s conditions and $V_t(b) \leq t(2B(\kappa + 1))^2$ by our boundedness assumptions, it follows that 
\begin{equation}
    1 = \limsup_{t \rightarrow \infty}\frac{|t(\hat\mu_t(b) - \mu(b))|}{\sqrt{2V_t \log\log V_t }} \geq \limsup_{t \rightarrow \infty} \frac{|\hat\mu_t(b) - \mu(b)|}{\sqrt{2(2B(\kappa + 1))^2\log\log(t(2B(\kappa + 1))^2)/t }} 
\end{equation}
Thus, $|\hat\mu_t(b) - \mu(b) |$ is of asymptotic order $O\left(\sqrt\frac{\log\log t}{t}\right)$. For any $\eta \in (1/2, 1)$, this implies $|\hat\mu_t(b) - \mu(b) | = o(1/t^{1-\eta})$ and $(\hat\mu_t(b) - \mu(b))^2 = o(1/t^{2-2\eta})$. By Lemma \ref{lem:cesaro_convergence}, it follows that for every $\eta \in (1/2, 1)$, 
\begin{equation}
     \limsup_{t \rightarrow \infty}\frac{1}{t^{2-2\eta}}\left(\frac{1}{t}\sum_{i=1}^t (\hat\mu_t(b) - \mu(b))^2\right) \rightarrow 0. 
\end{equation}
Plugging in our convergence rates to Equation \eqref{eq:term_i}, we obtain
\begin{align}
    (i) &\leq K\sum_{b \in [K]}\underbrace{\left(\frac{1}{t}\sum_{i=1}^t \left( \hat\mu_i(b) - \mu(b) \right)^2\right)}_{=o(1/t^{2-2\eta})} + 2\underbrace{\left(\frac{K}{t}\sum_{i=1}^t \sum_{b \in [K]} (\hat\mu_i(b) - \mu(b))^2 \right)^{1/2}}_{=o(1/t^{1-\eta})}\left(\frac{1}{t}\sum_{i=1}^t (\mu_i(a) - Z_i(a))^2\right)^{1/2}\\
    &\leq o(1/t^{2-2\eta}) + o(1/t^{1-\eta})\left(\frac{1}{t}\sum_{i=1}^t (\mu_i(a) - Z_i(a))^2\right)^{1/2}.
\end{align}
By the fact that $|\mu_i(a)-Z_i(a)| = |\sum_{b \in [K]} w_i^a(b)\phi_i(b)| \leq 4B(1+\kappa)$, we obtain $ \left(\frac{1}{t}\sum_{i=1}^t (\mu_i(a) - Z_i(a))^2\right)^{1/2} \leq 4B(1+\kappa)$, ensuring that term $(i)$ is of order $o(1/t^{1-\eta})$ for any $\eta \in (1/2, 1)$. 

To control term $(ii)$ in Equation \eqref{eq:terms_i_ii}, we apply Lemma \ref{lem:hall_martingale} and repeat our application of Lemma \ref{lem:cesaro_convergence}, using the fact that term $(ii)$ (multiplied by $t$) is simply the sum of a martingale difference sequence. Our convergence result holds under any sequence of weights $(\bm{w}_t^a)_{t=1}^\infty$, where $\bm{w}_t^a \in \Delta(a)$.

First, we verify the conditions of Lemma \ref{lem:hall_martingale}, using $\gamma_i(a)=(Z_i(a) - \mu_i(a))^2 - \sigma_i^2(a)$ as our martingale difference terms. By definition of $\sigma^2_t(a)$, we obtain $ \EE\left[ \gamma_i(a) | H_{i-1}\right] = \EE\left[(Z_i(a) - \mu_i(a))^2 - \sigma_i^2(a) | H_{i-1}\right] = 0$. To apply Lemma \ref{lem:hall_martingale}, we also require that there exists some $\beta > 1/2$ such that $\lim_{t \rightarrow 0 }\sum_{i=1}^t \frac{1}{i^{2\beta}} \EE[\gamma_i(a)^2 | \Hcal_{i-1}] < \infty$. To prove this, we first bound the conditional squared expectation of $\gamma_i^2(a)$ as follows:
\begin{equation}
    \EE\left[\gamma_i(a)^2 | H_{i-1} \right] = \EE\left[\left( \left({Z_i(a)} - \mu_i(a)\right)^2 - \sigma_i^2(a) \right)\right] \leq \EE\left[\left( \left(|Z_i(a) - \mu_i(a)|\right)^2 + |\sigma_i^2(a)| \right)\right].
\end{equation}
Note that $|Z_i(a) - \mu_i(a)| \leq 4B(1 + \kappa )$ and $\sigma_i^2(a) = \EE\left[(Z_i(a) - \mu_i(a))^2 | H_{i-1} \right] \leq (4B(1 + \kappa ))^2$, resulting in the following deterministic upper bound for the squared conditional expectation $ \EE\left[\gamma_i(a)^2 | H_{i-1} \right]$:
\begin{equation}
    \EE\left[\gamma_i(a)^2 | H_{i-1} \right] \leq 2(4B(1 + \kappa ))^2.
\end{equation}
Setting $\beta = 3/4$ and denoting $\zeta(3/2)$ as the Riemann-Zeta function, we obtain
\begin{equation}
    \lim_{t \rightarrow 0 }\sum_{i=1}^t \frac{1}{i^{2\beta}} \EE[\gamma_i(a)^2 | \Fcal_{i-1}] = 2(4B(1+\kappa))^2 \sum_{i=1}^\infty \frac{1}{i^{3/2}} = 2(4B(1+\kappa))^2\zeta(3/2) \approx 5.2(4B(1+\kappa))^2 < \infty,
\end{equation}
almost surely, and therefore Lemma \ref{lem:hall_martingale} directly applies to our martingale $\sum_{i=1}^t \gamma_i(a) = \sum_{i=1}^t (Z_i(a) - \mu_i(a))^2 - \sigma_i^2(a)$. By direct application of Lemma \ref{lem:hall_martingale} with $\beta = 3/4$, we obtain the following result in an almost-sure sense:
\begin{equation}
    \lim_{t \rightarrow \infty}\frac{1}{t^{3/4}}\left( \sum_{i=1}^t \gamma_i(a)\right) = \lim_{t \rightarrow \infty}\frac{\frac{1}{t}\left(\sum_{i=1}^t(Z_i(a) - \mu_i(a))^2 - \sigma_i^2(a)\right)}{t^{-1/4}} =  0.
\end{equation}
This immediately implies term $(ii) = \frac{1}{t}\left(\sum_{i=1}^t(Z_i(a) - \mu_i(a))^2 - \sigma_i^2(a)\right)$ is of order $o(t^{-1/4})$ almost surely. Combining our convergence rates for terms $(i)$ and $(ii)$, we obtain that our estimated variance satisfies
\begin{equation}
    \hat\sigma_t^2(a) - \frac{1}{t}\sum_{i=1}^t \sigma_i^2(a) = (i) + (ii) = o(1/t^{1-\eta}) + o(1/t^{1/4})
\end{equation}
for any $\eta \in (1/2, 1)$. Setting $\eta = 1/4$, we satisfy Condition $(B3)$. By satisfying all conditions of Lemma \ref{thm:waudby_smith}, the results of Theorem \ref{thm:error_control} follow.

\subsection{Proof of Lemma \ref{lem:asymp_valid_bai}}

The proof of Lemma \ref{lem:asymp_valid_bai} follows from (i) the results of Theorem \ref{thm:error_control} and (ii) convergence of our SNR-maximizing weights $\bm{w}_t^a$ and running mean estimates $\hat\mu_t(a)$. We begin by proving the convergence of our SNR-maximizing weights $\bm{w}_t^a$ for all $a\in [K]$ under the conditions of Theorem \ref{thm:error_control} in Lemma \ref{lem:convergence_snr_maximizing_weights}.

\begin{lemma}[Convergence of SNR-Maximizing Weights]\label{lem:convergence_snr_maximizing_weights}
Under the conditions of Theorem \ref{thm:error_control}, $\bm{w}_t^a(b) \rightarrow \bm{w}_\infty^a(b)$ for all $b\in [K]$ and $a \in [K]$ almost surely, where $\bm{w}_\infty^a$ is as defined in Theorem \ref{thm:sample_complexity_general}.
\end{lemma}

\begin{proof}[Proof of Lemma \ref{lem:convergence_snr_maximizing_weights}]
For the best arm $a^*$, we show that $\bm{w}_t^{a^*} \rightarrow \bm{w}_0^{a^*}$. In the proof of Condition $(B3)$ for Theorem \ref{thm:error_control}, we proved that $\hat\mu_t(a) \rightarrow \mu(a) $ almost surely for all $a\in [K]$, at a rate of $O(\sqrt{\log\log t/ t})$. Let $\omega \in \Omega$ denote a sample path, where $P(\Omega) = 1$, and let $X(\omega)$ denote the realization of a random variable $X$ on sample path $\omega$. Let $\delta(\bm{\mu}) = \mu(a^*) - \max_{b\neq a^*} \mu(b)$. 

By definition of almost sure convergence, for every $\omega \in \Omega$, there exists a $t_{a^*}(\omega) < \infty$ such that $\hat\mu_t(a^*)(\omega) > \mu(a^*)(\omega) - \delta(\bm{\mu})/2$ for all $t \geq t_{a^*}(\omega)$. Likewise, for all $a \neq a^*$, there exists a $t_{a}(\omega) < \infty$ such $\hat\mu_t(a)(\omega) < \mu(a)(\omega) + \delta(\bm{\mu})/2$ for all $t \geq t_a(\omega)$. Then, for every $\omega \in \Omega$, there exists  $t(\omega) = \sup_{a \in [K]} t_a(\omega)$ such that $\hat\mu_t(a^*) > \max_{b \neq a}\hat\mu_t(b)$ for all $t \geq t(\omega)$, and $P(\lim_{t \rightarrow \infty} \mathbf{1}[\hat\mu_t(a^*) > \max_{b\neq a^*} \hat\mu_t(b)]) = 1$. We can express our limiting weight $\bm{w}_t^{a^*}$ as 
\begin{equation}
    \bm{w}_t^{a^*} = \mathbf{1}\left[\hat\mu_t(a^*) > \max_{b\neq a^*} \hat\mu_t(b)\right] \bm{w}_0^{a^*} + \mathbf{1}\left[\hat\mu_t(a^*) \leq \max_{b \neq a^*} \hat\mu_t(b)\right] \tilde{\bm{w}}_t^{a^*},
\end{equation}
where $\tilde{\bm{w}}_t^{a^*} = \argmax_{w \in \Delta(a^*)} \frac{\sum_{b \in [K]}w(b) \hat\mu_{t-1}(b)}{\hat\sigma_{t-1}(\bm{w})}$. Because $\mathbf{1}\left[\hat\mu_t(a^*) > \max_{b\neq a^*} \hat\mu_t(b)\right] \rightarrow 1$ almost surely, it immediately follows that $\bm{w}_t^{a^*} \rightarrow \bm{w}_0^{a^*}$ almost surely in an element-wise sense. 

To prove that our SNR-maximizing weights $\bm{w}_t^a$ converge to unique limit $\bm{w}_\infty^a$ for $a \neq a^*$, we leverage the results of Lemma \ref{lem:unique_optima} to ensure $\bm{w}_\infty^a$ is unique for all $a \neq a^*$. We then use Lemma \ref{lem:conv_argmax} to show that our empirical SNR-maximizing weights $\bm{w}_t^a \rightarrow \bm{w}_\infty^a$ almost surely. In the proof of Condition $(B3)$ for Theorem \ref{thm:error_control}, we show $\hat\mu_t(a) \rightarrow \mu(a)$ a.s. for all $a\in [K]$, ensuring $\lim_{t \rightarrow \infty}\sum_{b \in [K]} w(b) \hat\mu_{t-1}(b) =\sum_{b \in [K]} w(b) \mu(b)$ a.s. 

For the variance terms, recall that we establish almost-sure convergence of $\hat\sigma^2_{t}(a) - \frac{1}{t}\sum_{i=1}^t \sigma^2(a)$ for any sequence of weights $(\bm{w}_t^a)_{t=1}^\infty$ in order for  Condition $(B3)$ to hold. Note that $\sigma^2_{t-1}(\bm{w})$ is equivalent to $\hat\sigma^2_{t-1}(a)$ with $\bm{w}_i^a = \bm{w}$ for all $i \in [t-1]$, and so we obtain $|\hat\sigma_{t-1}^2(\bm{w}) - \frac{1}{t-1}\sum_{i=1}^{t-1} \sigma_{i}^2(\bm{w})| \rightarrow 0$ almost surely, where
\begin{align}
 \sigma_{i}^2(\bm{w}) =& \sum_{b \in [K]}w^2(b) \EE_{P_X}\left[\frac{v(x,b)}{\pi_i(x,b)} | H_{i-1}\right]\label{l:prob_1}  \\
 &+\EE_{P_X}\left[\left(\sum_{b \in [K]}w(b) \left(g(x,b) - \mu(b)\right) \right)^2 \right] \\
 &+ \sum_{b \in [K]} w^2(b) \EE_{P_X}\left[ \frac{(g_i(x,b) - g(x,b))^2}{\pi_i(x,b)} | H_{i-1} \right]\label{l:prob_2} \\
 &- \EE\left[\left( \sum_{b \in [K]}w(b) (g_i(x,b) -g(x,b)) \right)^2 | H_{i-1}\right]\label{l:prob_3}
\end{align}
follows from  our conditional variance expansion in lines \eqref{l:expected_vars}-\eqref{l:rem_2}. We now show that $\sigma_i^2(\bm{w})$ converges to $\sigma_\infty^2(\bm{w})$ almost surely, and use Lemma \ref{lem:cesaro_convergence} to show $\frac{1}{t-1}\sum_{i=1}^{t-1} \sigma_{i}^2(\bm{w})$ converges to $\sigma_\infty^2(\bm{w})$ as well. 

First, note that only lines \eqref{l:prob_1}, \eqref{l:prob_2}, and \eqref{l:prob_3} contain $i$-dependent terms. We take the limit of each of these terms to show that $\sigma_i^2(\bm{w}) \rightarrow \sigma^2_\infty(\bm{w})$ as defined in Theorem \ref{thm:sample_complexity_general}. Let $\pi_\infty$ denote the $L_2$ limit of $\pi_t$, as defined in Equation \ref{thm:error_control}. By the boundedness of $v(x,b)$ due to $|Y_t| \leq B$ and $|g_t(x,b)| \leq B$ for all $t \in \NN$, $x \in \Xcal$, $b \in [K]$ and $\frac{1}{\pi_t(x,b)} \leq \kappa < \infty$ for all $t \in \NN$, $x \in \Xcal$, $b \in [K]$, the difference between the term on line \eqref{l:prob_1} and its corresponding quantity with $\pi_\infty$ satisfies
\begin{align}
    &\lim_{t \rightarrow \infty}\left|\sum_{b \in [K]}w^2(b) \EE_{P_X}\left[\frac{v(x,b)}{\pi_i(x,b)} | H_{i-1}\right] - \sum_{b \in [K]}w^2(b) \EE_{P_X}\left[\frac{v(x,b)}{\pi_\infty(x,b)}\right] \right| = \\
    & \lim_{t \rightarrow \infty} \left|\sum_{b \in [K]}w^2(b) \EE_{P_X}\left[v(x,b)\left(\frac{1}{\pi_i(x,b)}- \frac{1}{\pi_\infty(x,b)}\right) | H_{i-1}\right] \right| \leq  \\
    & \lim_{t \rightarrow \infty}\sum_{b \in [K]}w^2(b) \EE_{P_X}\left[ \left|v(x,b)\left(\frac{\pi_\infty(x,b) - \pi_i(x,b)}{\pi_\infty(x,b)\pi_i(x,b)}\right)\right| | H_{i-1}\right] \leq \\
    & \lim_{t \rightarrow \infty}\sum_{b \in [K]}w^2(b) \EE_{P_X}\left[ \left|\left(\frac{v(x,b)}{\pi_\infty(x,b)\pi_i(x,b)}\right)^2\right| | H_{i-1}\right]^{1/2}\EE\left[\left|\pi_\infty(x,b) - \pi_i(x,b)\right|^2|H_{i-1}\right]^{1/2}
\end{align}
where the last inequality follows from Holder's inequality with $p=q=2$. 
By Condition $(A1)$ of Theorem \ref{thm:error_control}, $\EE\left[\left|\pi_\infty(x,b) - \pi_i(x,b)\right|^2|H_{i-1}\right]^{1/2} \rightarrow 0$ almost surely. By our boundedness assumptions on $Y$ and $\pi_i(x,b)$, we obtain $v(x,b) = \EE[ (Y - g(x,b))^2 | A=b, X=x ] \leq 4B^2$ and $\pi_\infty(x,b)\pi_i(x,b)\leq \kappa^2$, and therefore $\EE_{P_X}\left[ \left|\left(\frac{v(x,b)}{\pi_\infty(x,b)\pi_i(x,b)}\right)^2\right| | H_{i-1}\right]^{1/2} \leq 4B^2\kappa^2$. As a result, we obtain that the limit of the terms in line \eqref{l:prob_1} is 
\begin{equation}
    \lim_{t \rightarrow \infty}\sum_{b \in [K]}w^2(b) \EE_{P_X}\left[\frac{v(x,b)}{\pi_i(x,b)} | H_{i-1}\right] = \sum_{b \in [K]}w^2(b) \EE_{P_X}\left[\frac{v(x,b)}{\pi_\infty(x,b)}\right]
\end{equation}

To obtain the limit of line \eqref{l:prob_2}, we show that the difference between $\EE_{P_X}\left[ \frac{(g_i(x,b) - g(x,b))^2}{\pi_i(x,b)} | H_{i-1} \right]$ and $\EE_{P_X}\left[ \frac{(g_\infty(x,b) - g(x,b))^2}{\pi_\infty(x,b)}\right]$ converges to zero almost surely. We bound the magnitude of the difference as  
\begin{align}
    &\left|\EE_{P_X}\left[ \frac{(g_i(x,b) - g(x,b))^2}{\pi_i(x,b)} | H_{i-1} \right] - \EE_{P_X}\left[ \frac{(g_\infty(x,b) - g(x,b))^2}{\pi_\infty(x,b)}\right]\right| \leq \\
    & \underbrace{\left|\EE_{P_X}\left[ \frac{(g_i(x,b) - g(x,b))^2}{\pi_i(x,b)} | H_{i-1} \right] - \EE_{P_X}\left[ \frac{(g_\infty(x,b) - g(x,b))^2}{\pi_i(x,b)} | H_{i-1}\right]\right|}_{(a)}\label{l:subprob_1} + \\
    & \underbrace{\left|\EE_{P_X}\left[ \frac{(g_\infty(x,b) - g(x,b))^2}{\pi_i(x,b)} | H_{i-1} \right] - \EE_{P_X}\left[ \frac{(g_\infty(x,b) - g(x,b))^2}{\pi_\infty(x,b)} \right]\right|}_{(b)}\label{l:subprob_2}
\end{align}

For term $(b)$, we repeat our steps for showing that the term on line \eqref{l:prob_1} converges almost surely to the desired limit. We can upper bound term $(b)$ as follows:
\begin{align}
    (b) &= \left|\EE_{P_X}\left[ \frac{(g_\infty(x,b) - g(x,b))^2}{\pi_i(x,b)} | H_{i-1} \right] - \EE_{P_X}\left[ \frac{(g_\infty(x,b) - g(x,b))^2}{\pi_\infty(x,b)} \right]\right| \\
    &\leq \EE_{P_X}\left[\left| \frac{(g_\infty(x,b) - g(x,b))^2}{\pi_i(x,b)\pi_\infty(x,b)} \left(\pi_\infty(x,b) - \pi_i(x,b)\right) \right| | H_{i-1} \right] \\
    &\leq \underbrace{\EE_{P_X}\left[\left| \frac{(g_\infty(x,b) - g(x,b))^2}{\pi_i(x,b)\pi_\infty(x,b)}  \right|^2 | H_{i-1} \right]^{1/2}}_{\leq 4B^2\kappa^2} \underbrace{\EE[\left|\pi_\infty(x,b) - \pi_i(x,b)\right|^2 | H_{i-1}]^{1/2}}_{=o(1)}.
\end{align}
Because $\EE_{P_X}\left[\left| \frac{(g_\infty(x,b) - g(x,b))^2}{\pi_i(x,b)\pi_\infty(x,b)}  \right|^2 | H_{i-1} \right]^{1/2} \leq 4B^2\kappa^2$ and $\EE[\left|\pi_\infty(x,b) - \pi_i(x,b)\right|^2 | H_{i-1}]^{1/2} = \| \pi_\infty - \pi(x,b) \|_{L_2(P_{H_{i-1}})}$ is of order $o(1)$, term $(b)$ vanishes to zero almost surely. We now show that term $(a)$ also vanishes almost surely.   

For term $(a)$, we expand our expression to obtain
\begin{align}
    &\left|\EE_{P_X}\left[ \frac{(g_i(x,b) - g(x,b))^2 - (g_\infty(x,b) - g(x,b))^2}{\pi_i(x,b)} | H_{i-1} \right]\right| = \\
    &\left|\EE_{P_X}\left[ \frac{(g_i(x,b) - g_\infty(x,b))(g_\infty(x,b) + g_i(x,b) - 2g(x,b))}{\pi_i(x,b)} | H_{i-1} \right]\right| \leq \\
    &4B\kappa \EE_{P_X}\left[ \ |g_i(x,b) - g_\infty(x,b)| \ | H_{i-1}\right] ,
\end{align}
where the last inequality follows from the fact that $|\frac{g_\infty(x,b) + g_i(x,b) - 2g(x,b)}{\pi_i(x,b)} | \leq 4B\kappa$. By Holder's inequality,
\begin{equation}
    \EE_{P_X}\left[   |g_i(x,b) - g_\infty(x,b) |  | H_{i-1}\right] \leq \EE_{P_X}\left[1 \right]^{1/2} \EE_{P_X}\left[   |g_i(x,b) - g_\infty(x,b) |^2  \right]^{1/2} = \|  g_i(x,b) - g_\infty(x,b)\|_{L_2(P_{H_{i-1}})},
\end{equation}
which is $o(1)$ by the $L_2$-convergence of $g_i$ in Condition $(A2)$ of Theorem \ref{thm:error_control}. Thus, we obtain 
\begin{equation}
    \lim_{i \rightarrow \infty} \EE_{P_X}\left[ \frac{(g_i(x,b) - g(x,b))^2}{\pi_i(x,b)} | H_{i-1} \right] = \EE_{P_X}\left[ \frac{(g_\infty(x,b) - g(x,b))^2}{\pi_\infty(x,b)}\right],
\end{equation}
and the term in line \eqref{l:prob_2} converges to $\sum_{b \in [K]} w^2(b) \EE_{P_X}\left[ \frac{(g_\infty(x,b) - g(x,b))^2}{\pi_\infty(x,b)} \right]$ almost surely. Lastly, for the term in \eqref{l:prob_3}, we repeat the steps for showing term $(a)$ in Equation \eqref{l:subprob_1} vanishes almost surely to obtain 
\begin{equation}
    \lim_{t \rightarrow \infty}\EE\left[\left( \sum_{b \in [K]}w(b) (g_i(x,b) -g(x,b)) \right)^2 | H_{i-1}\right] = \EE\left[\left( \sum_{b \in [K]}w(b) (g_\infty(x,b) -g(x,b)) \right)^2 \right]
\end{equation}
almost surely. Putting our results together, we obtain that 
\begin{align}
 \lim_{i \rightarrow \infty}\sigma_{i}^2(\bm{w}) =& \sum_{b \in [K]}w^2(b) \EE_{P_X}\left[\frac{v(x,b)}{\pi_\infty(x,b)} \right]  +\EE_{P_X}\left[\left(\sum_{b \in [K]}w(b) \left(g(x,b) - \mu(b)\right) \right)^2 \right]\label{eq:variance_convergence_1} \\
 &+ \sum_{b \in [K]} w^2(b) \EE_{P_X}\left[ \frac{(g_\infty(x,b) - g(x,b))^2}{\pi_\infty(x,b)} \right] - \EE\left[\left( \sum_{b \in [K]}w(b) (g_\infty(x,b) -g(x,b)) \right)^2 \right] \\
 & = \EE_{P_\infty}\left[ \left(\sum_{b \in [K]} w(b) \left( \phi_\infty(b) - \mu(b)\right) \right)^2\right] = \sigma_\infty^2(\bm{w}),\label{eq:variance_convergence}
\end{align}
where $\phi_\infty(b)$ is defined as in Theorem \ref{thm:sample_complexity_general}. Note that because $\lim_{t \rightarrow \infty} \sigma_t^2(\bm{w}) = \sigma_\infty^2(\bm{w})$ almost surely, it follows that $\frac{1}{t}\sum_{i=1}^t \sigma_i^2(\bm{w}) \rightarrow \sigma^2_\infty(\bm{w})$ almost surely as well from Lemma \ref{lem:cesaro_convergence}. By the proof of Condition $(B3)$ for Theorem \ref{thm:error_control}, we obtain $\hat\sigma^2_{t-1}(\bm{w}) - \frac{1}{t-1}\sum_{i=1}^{t-1}\sigma^2_i(\bm{w}) \rightarrow 0$ almost surely, and therefore $\hat\sigma^2_{t-1}(\bm{w}) \rightarrow \sigma_\infty^2(\bm{w})$ almost surely. Note that by the continuous mapping theorem, $\hat\sigma_{t-1}(\bm{w}) \rightarrow \sigma_\infty(\bm{w})$ as well. 

The numerator $f(\bm{w}) = \lim_{t \rightarrow \infty}\sum_{b \in [K]} w(b) \hat\mu_{t-1}(b)$ and denominator $g(\bm{w}) = \lim_{t \rightarrow \infty}\hat\sigma_{t-1}(\bm{w})$ of our limiting SNR-maximization problem satisfy $f(\bm{w}) = \sum_{b \in [K]}w(b) \mu(b)$ and $g(\bm{w}) = \sigma_\infty(\bm{w})$. We now show that the conditions of Lemma \ref{lem:unique_optima} are satisfied, ensuring $\argmax_{\bm{w}\in \Delta(a)} f(\bm{w})/ g(\bm{w})$ is a single vector $\bm{w}_\infty^a$. 

Note that $f(\bm{w})$ is affine, and $\Delta(a)$ is a nonempty compact convex set. To satisfy the conditions of Lemma \ref{lem:unique_optima}, it only remains to show that (i) $g(\bm{w})$ is strictly convex and positive and (ii) $\max_{\bm{w}\in \Delta(a)}f(\bm{w})/g(\bm{w}) >0$ for Lemma \ref{lem:unique_optima} to hold. We begin with strict convexity. Let $\bm{\phi}_\infty \in \RR^K$ be the vector with entries ${\phi}_\infty(a) = g_\infty(X,a) + \frac{\mathbf{1}[A=a](Y - g_\infty(X,a))}{\pi_\infty(X,a)} - \mu(a)$. Then, the limiting denominator $g(\bm{w})$ can be re-expressed as 
\begin{equation}
    g(\bm{w}) =   \|\bm{\phi}_\infty^\top \bm{w}   \|_{L_2(P_\infty)}.
\end{equation}
We now show that $g(\bm{w})$ must be strictly convex under the assumption that $\Sigma_{\infty}$ (as defined in Theorem \ref{thm:error_control}) is invertible. Because $\|\cdot \|_{L_2(P_X)}$ is a norm, for any $\lambda \in [0,1]$ and $\bm{w}_1, \bm{w}_2 \in \Delta(a)$, we obtain the following result through the triangle inequality:
\begin{equation}\label{eq:triangle_inequality}
    \|\bm{\phi}_{\infty}^\top\left(\lambda \bm{w}_1 + (1-\lambda)\bm{w}_2  \right) \|_2 \leq \lambda \|\bm{\phi}_{\infty}^\top\bm{w}_1 \|_2  + (1-\lambda) \|\bm{\phi}_{\infty}^\top\bm{w}_2 \|_2.
\end{equation}
Thus, $g(\bm{w})$ is convex for all $t \geq t'$. To show our convexity is strict, we proceed by contradiction. For equality to occur in Equation \eqref{eq:triangle_inequality}, we require $\bm{\phi}_{\infty}^\top \bm{w}_1$ and $\bm{\phi}_{\infty}^\top \bm{w}_2$ to be collinear. Assuming that $\bm{\phi}_{\infty}^\top \bm{w}_1$ and $\bm{\phi}_{\infty}^\top \bm{w}_2$ are collinear, there exists $c\neq 1$ such that $c\bm{\phi}_{\infty}^\top \bm{w}_1 = \bm{\phi}_{\infty}^\top \bm{w}_2$. Under the assumption that $\Sigma_\infty^{-1} = \left(\EE_{P_\infty}[\bm{\phi}_\infty \bm\phi_\infty^\top] \right)^{-1}$ exists (Condition $(A3)$ of Theorem \ref{thm:error_control}) and multiplying both sides by $\left( \left[\bm{\phi}_{\infty}\bm{\phi}_{\infty}^\top\right]^{-1} \bm{\phi}_{\infty} \right)$, we obtain
\begin{equation}
    \left( \left[\bm{\phi}_{\infty}\bm{\phi}_{\infty}^\top\right]^{-1} \bm{\phi}_{\infty} \right)\bm{\phi}_{\infty}^\top \bm{w}_2 = c\left( \left[\bm{\phi}_{\infty}\bm{\phi}_{\infty}^\top\right]^{-1} \bm{\phi}_{\infty} \right)\bm{\phi}_{\infty}^\top \bm{w}_1 \implies \bm{w_2} = c\bm{w}_1.
\end{equation}
However, note that $w_2(a) = w_1(a) = -1$ and for any $c\neq 1$, $cw_1(a) \neq -1$. This leads to our contradiction, ensuring the limiting denominator $g(\bm{w})$ is strictly convex. To show $g(\bm{w})$ is strictly positive, note that $\|\bm{\phi}_{\infty}^\top \bm{w} \|_{L_2(P_\infty)}$ is the limiting variance for a weighted combination of arm mean estimates. Under Assumption \ref{assump:nonzero_variances} and the fact that there exists one entry $w(a) = -1$, it follows that this term must be strictly positive. 

Finally, to show that our limiting SNR-maximization objective $ \max_{\bm{w}\in\Delta(a)} f(\bm{w})/g(\bm{w})$ has positive value, note that the choice of $\bm{w}^a_{base}$, where $w^a_{base}(a^*) = 1$, $w^a_{base}(a) = -1$, and $w^a_{base}(b) = 0$ for all $b \not\in \{a, a^*\}$ yields a positive objective value. Because $ \max_{\bm{w}\in\Delta(a)} f(\bm{w})/g(\bm{w}) \geq f(\bm{w}_{base}^a)/ g(\bm{w}_{base}^a)$, it must also be positive. Thus, by direct application of Lemma \ref{lem:unique_optima}, we obtain that 
\begin{equation}
 \bm{w}_\infty^a = \argmax_{\bm{w} \in \Delta(a)} \frac{\lim_{t \rightarrow \infty} \sum_{b \in [K]} w(b)\hat\mu_{t-1}(b)}{\lim_{t \rightarrow \infty}\hat\sigma_{t-1}(\bm{w})}    = \argmax_{\bm{w} \in \Delta(a)} \frac{\sum_{b \in [K]} w(b) \mu(b)}{\sigma_\infty(\bm{w})}
\end{equation}
is the unique maximizer of the limiting signal-to-noise ratio. 

We now apply Lemma \ref{lem:conv_argmax} to show that our empirical SNR-maximizing weights $\bm{w}_t^a$ converge to $\bm{w}_\infty^a$. First, note that the empirical SNR objective is uniformly Lipschitz with respect to $\bm{w} \in \Delta(a)$ almost surely as $t \rightarrow \infty$. Thus, by Chapter 1 of \cite{vandervaart1996weak}, it suffices to show pointwise almost sure convergence on a dense subset of $\Delta(a)$. 

We now proceed to show pointwise convergence. As shown above, for any $\bm{w} \in \Delta(a)$, 
\begin{align}
    \lim_{t \rightarrow \infty} \sum_{b \in [K]} w(b)\hat\mu_{t-1}(b) = \sum_{b \in [K]} w(b) \mu(b), \quad \lim_{t \rightarrow \infty}\hat\sigma_{t-1}(\bm{w}) = \sigma_\infty(\bm{w}) > 0
\end{align}
almost surely. By the quotient rule for limits and the fact that $\sigma_\infty(\bm{w}) > 0$ for all $\bm{w} \in \Delta(a)$, we obtain $\lim_{t \rightarrow \infty}\frac{\sum_{b \in [K]} w(b)\hat\mu_{t-1}(b)}{\hat\sigma_{t-1}(\bm{w})} = \frac{\sum_{b \in [K]} w(b) \mu(b)}{\sigma_\infty(\bm{w})}$ almost surely for all $\bm{w} \in \Delta(a)$. By construction,
\begin{equation}
    \bm{w}_t^a = \mathbf{1}\left[\hat{\mu}_t(a) < \max_{b \in [K]}\hat\mu_t(b)\right] \tilde{\bm{w}}_t^a + \mathbf{1}\left[\hat{\mu}_t(a) = \max_{b \in [K]}\hat\mu_t(b)\right] {\bm{w}}_0^a 
\end{equation}
where $\tilde{\bm{w}}_t^a \in \argmax_{\bm{w} \in \Delta(a)} \frac{\sum_{b \in [K]} w(b)\hat\mu_{t-1}(b)}{\hat\sigma_{t-1}(\bm{w})}$ for each $t \in \NN$. Because $\mathbf{1}\left[\hat{\mu}_t(a) < \max_{b \in [K]}\hat\mu_t(b)\right] \rightarrow 1$ almost surely as $t \rightarrow \infty$, we obtain $|\bm{w}_t^a(b) - \tilde{\bm{w}}_t^a(b)| \rightarrow 0$ almost surely for all $a\in[K], b\in [K]$, and by direct application of Lemma \ref{lem:conv_argmax} to $\tilde{\bm{w}}_t^a$, we obtain $\tilde{\bm{w}}_t^a(b) \rightarrow \bm{w}_\infty^a(b)$ almost surely for all $b \in [K], a\in [K]$. Therefore, $\bm{w}_t^a(b) \rightarrow \bm{w}_\infty^a(b)$ for all $a \in [K]$, $b\in [K]$ almost surely.
\end{proof}

We now proceed to the proof of Lemma \ref{lem:asymp_valid_bai}. To satisfy asymptotic $\alpha$-level correctness as in Definition \ref{defn:asymp_alpha_correctness}, we require (i) finite stopping times, i.e. $\tau = \inf\{t \in \NN: |C_t(H_t, \alpha)| \leq 1 \} < \infty$, and (ii) the limiting error rate is below $\alpha$, i.e. $\limsup_{\alpha \rightarrow 0}\frac{P( \hat{a} \neq a^*)}{\alpha} \leq 1$. We start with the proof of finite stopping times. 

\paragraph{Finite Stopping Times}
To prove that stopping times are finite, we first consider the stopping time $\tau$ \textit{without} a burn-in period (i.e. $t_0 = 0$). Consider an auxiliary random variable $\tilde\tau = \inf\{t \in \NN: \sup_{i \leq t} L_i^a(H_i, \alpha, \rho) > 0 \ \forall a \neq a^*\}$, the minimum number of samples to reject all suboptimal arms $a \neq a^*$. By definition, note that $\tilde{\tau} \geq \tau$ deterministically. We will show that $\tilde{\tau}$ is finite almost surely for any fixed $\alpha \in (0,1)$, $\rho > 0$, and $\bm{w}_0^a \in \Delta(a)$ for all $a \in [K]$. To show $\tilde\tau$ is finite almost surely, we show that $L_t^a(H_t, \alpha, \rho) > 0$ for all $a \neq a^*$ almost surely. We first derive the almost-sure limit of our score process below, using our existing results:
\begin{align}
    \left|\hat{\psi}_t(a) - \sum_{b \in [K]} w_\infty^a(b) \mu(b)\right| &= \left|\frac{1}{t}\sum_{i=1}^t\sum_{b \in [K]} \left(w_i^a(b) \phi_i(b) - w_\infty^a(b)\mu(b)\right)\right| \\
    & \leq \left|\frac{1}{t}\sum_{i=1}^t\sum_{b \in [K]} w_i^a(b) (\phi_i(b) -  \mu(b)) \right| +  \left|\frac{1}{t}\sum_{i=1}^t\sum_{b \in [K]}\mu(b)(w_i^a(b) - w_\infty^a(b))\right|\label{l:prob}
\end{align}
The first term on line \eqref{l:prob} converges almost surely to zero by the fact that $\frac{1}{t}\sum_{i=1}^t \phi_i(b) \rightarrow \mu(b)$ almost surely for all $b \in [K]$. The second term on line \eqref{l:prob} vanishes due to Lemmas \ref{lem:cesaro_convergence} and \ref{lem:convergence_snr_maximizing_weights}. Thus, we obtain $\hat\psi_t(a) \rightarrow \sum_{b \in [K]} w_\infty^a(b)\mu(b)$ almost surely for all $a \in [K]$. Likewise, we obtain $\hat\sigma_t(b) \rightarrow \sigma_\infty(\bm{w}_\infty^a) > 0$ almost surely by applying the same argument to the result in lines \eqref{eq:variance_convergence_1}-\eqref{eq:variance_convergence} and Lemma \ref{lem:convergence_snr_maximizing_weights}. Thus, we have that
\begin{equation}
    \lim_{t \rightarrow \infty}\frac{\hat\psi_t(a)}{\hat\sigma_t(a)} = \frac{\sum_{b \in [K]}w_\infty^a(b)\mu(b)}{\sigma_\infty(\bm{w}_\infty^a)} \label{eq:key_identity}
\end{equation}
almost surely. By definition of $\bm{w}_\infty^a$, we also have that 
\begin{equation}
    \frac{\sum_{b \in [K]}w_\infty^a(b)\mu(b)}{\sigma_\infty(\bm{w}_\infty^a)} = \max_{\bm{w} \in \Delta(a)} \frac{\sum_{b \in [K]}w(b) \mu(b)}{\sigma_\infty(\bm{w})} \geq \frac{\mu(a^*) - \mu(a)}{(4B(1+\kappa))},
\end{equation}
where our lower bound is a direct consequence of variance bounds derived from $|Y_t| \leq B$, $|g_\infty(x,b)| \leq B$, and $\bm{w} \in \Delta(a)$. Thus, $\lim_{t \rightarrow 0}\frac{\hat\psi_t(a)}{\hat\sigma_t(a)}$ converges to a constant. Note that $\ell_{t,\alpha,\rho}(\hat\sigma_t(a))$ is upper bounded as follows: 
\begin{equation}
    \ell_{t,\alpha,\rho}(\hat\sigma_t(a)) \leq \ell_{t,\alpha,\rho}(x) = t^{-1/2}\sqrt{\frac{2(\rho^2+1/t(4B(1+\kappa))^2)}{\rho^2}\log\left(1+\frac{\sqrt{t (4B(1+\kappa))^2\rho^2 + 1}}{2\alpha}\right)}
\end{equation}
by the same variance bounds, and vanishes towards zero almost surely as $t \rightarrow \infty$. As a result, 
\begin{equation}
    \liminf_{t \rightarrow \infty}\mathbf{1}[L_t^a(H_t, \alpha,\rho ) > 0] = \liminf_{t \rightarrow \infty}\mathbf{1}\left[\frac{\hat\psi_t(a)}{\hat\sigma_t(a)} > \ell_{t,\alpha, \rho}(\hat\sigma_t(a))\right]  = 1
\end{equation}
almost surely for any fixed $\alpha \in [0,1]$, $\rho > 0$, and $\bm{w}_0^a \in \Delta(a)$ for all $a \neq a^*$. Thus, for all $\omega$ in $\Omega$ such that $P(\Omega) = 1$, there exists a $t_a(\omega) < \infty$ such that for all $t \geq t_a(\omega)$, $L_t^a(H_t, \alpha, \rho)(\omega) > 0$. Setting $t(\omega) = \max_{a \neq a^*} t_a(\omega)$, we obtain $\tilde{\tau}(\omega) \leq t(\omega) < \infty$. Thus, $\tilde{\tau}$ is finite almost surely, and because $\tau \leq \tilde{\tau}$ deterministically, $\tau$ is finite almost surely as well. Lastly, note that for any \textit{fixed} burn-in time $t_0$, the stopping time $\tau_{t_0}$ satisfies $t_0 \leq \tau_{t_0}(\omega) \leq \max(t_0, t(\omega)) < \infty$, where $t(\omega)$ is defined as above. Consequently, for any fixed burn-in time $t_0$, we obtain that $\tau_{t_0}$ is finite almost surely.

\paragraph{Error Control} 
To show that we control error rates as desired, recall that Algorithm \ref{alg:bai_short} returns the wrong arm $\hat{a} \neq a^*$ if either (i) $a^* \not\in C_t(H_t, \alpha)$ and $|C_t(H_t, \alpha)| = 1$ or (ii) $|C_t(H_t, \alpha)| =0$ and $\hat{a} \not\in \argmin_{a \in [K]}\hat\psi_t(a) - \hat\sigma_t(a) \ell_{t,\alpha,\rho}(\hat\sigma_t(a))$.  In either case, it requires $a^* \not\in C_t(H_t, \alpha)$, and therefore
\begin{equation}
    P(\hat{a} \neq a^*) \leq P(\exists t \in \NN: a^* \not\in C_t(H_t,\alpha)). 
\end{equation}
By the results of Theorem \ref{thm:error_control} and $\tau < \infty$ for all fixed $\alpha \in (0,1)$, $\rho > 0$, and $\bm{w}_0^a \in \Delta(a)$ for all $a \in [K]$,
\begin{equation}
    \limsup_{\alpha \rightarrow 0}\frac{P(\hat{a} \neq a^*)}{\alpha} = \limsup_{\alpha \rightarrow 0}\frac{P(\exists t \in \NN: a^* \not\in C_t(H_t,\alpha) )}{\alpha} \leq 1,
\end{equation}
and therefore we satisfy the error control requirement of Defintion \ref{defn:asymp_alpha_correctness}. 

\subsection{Proof of Theorem \ref{thm:sample_complexity_general} }

Theorem \ref{thm:sample_complexity_general} guarantees upper bounds both in expectation and almost surely. We begin by considering the stopping time $\tau$ in the setting where the burn-in time $t_0$ is equal to zero. We then prove our bounds hold in an almost-sure sense, and leverage Egorov's Theorem to convert our almost-sure bounds to bounds in expectation. By showing that the stopping time (without a burn-in period) must be of order $\log(1/\alpha)$, we show that our choice of burn-in time does not affect the asymptotic sample complexity.

\paragraph{Almost-Sure Limit For Stopping Times}
We proceed in a similar manner to the proof of finite stopping times for Lemma \ref{lem:asymp_valid_bai}. Let $\tilde\tau = \inf\{t \in \NN: \sup_{i \leq t} L_i^a(H_i, \alpha, \rho) > 0 \ \forall a \neq a^*\}$, the minimum number of samples to reject all suboptimal arms $a \neq a^*$. By definition, note that $\tilde{\tau} \geq \tau$ deterministically. Note that $\tilde\tau$ must satisfy the following inequality almost surely for some random $b \neq a^*$ (which may depend on $\alpha$):
\begin{align}
        \tilde{\tau}\frac{\sum_{i=1}^{\tilde\tau}Z_i(b)}{\tilde{\tau}} -  \hat\sigma_{\tilde{\tau}}(b )\sqrt{ \tilde{\tau}\frac{2(\rho^2 + 1/(\tilde{\tau}\hat\sigma_{\tilde{\tau}}^2(b) )}{\rho^2} \log\left( 1+ \frac{\sqrt{\tilde{\tau} \hat\sigma_{\tilde{\tau}}^2(b )\rho^2 + 1 }}{2\alpha}\right)} \in [0, c],
\end{align}
where $Z_i(a) = \sum_{b \in [K]} w_i^a(b) \phi_i(b)$ and the bound $c$ is a deterministic constant that (i) upper bounds the overshoot beyond zero and (ii) does not depend on $\alpha$. This follows from the definition of the stopping criterion for $\tilde\tau$ and the fact that $1/\pi_t(x,a) \leq \kappa$, $|Y_t| \leq B$, and $|g_t(x,b)|$ for all $x \in \Xcal$, $b \in [K]$ and $t \in \NN$. We can rewrite the condition above as the following:
\begin{align}
        &\frac{\left(\frac{\sum_{i=1}^{\tilde{\tau}}{Z}_i(b)}{\tilde\tau}\right)^2}{\hat\sigma^2_{\tilde{\tau}}(b)} \frac{\tilde\tau}{{ \frac{2(\rho^2 + 1/(\tilde{\tau}\hat\sigma^2_{\tilde{\tau}}(b ) ))}{\rho^2} \log\left( 1+ \frac{\sqrt{\tilde\tau \hat\sigma^2_{\tilde{\tau}}(b)\rho^2 + 1 }}{2\alpha}\right)}} \\
        &\in \left[1, \left(1+\frac{c}{\hat\sigma^2_{\tilde{\tau}}(b)\tilde\tau\frac{2(\rho^2 + 1/(\tilde\tau\hat\sigma^2_{\tilde{\tau}}(b )))}{\rho^2} \log\left( 1+ \frac{\sqrt{\tilde\tau \hat\sigma^2_{\tilde{\tau}}( b )\rho^2 + 1 }}{2\alpha}\right)}\right)^2\right] \label{eq:deterministic_bounds}.
\end{align}
Note $\tilde{\tau} \geq t_0(\alpha)$ deterministically, where $t_0(\alpha)$ is as defined in the proof of Theorem \ref{thm:error_control}, and $t_0(\alpha) \rightarrow \infty$ as $\alpha \rightarrow 0$. We proceed by taking limits on both sides. First, note that the first term on the LHS and the upper bound on the RHS are bounded by or converge to the following limits almost surely: 
\begin{align}
    &\lim_{\alpha \rightarrow 0}\frac{\left(\frac{\sum_{i=1}^{\tilde{\tau}}{Z}_i(b)}{\tilde\tau}\right)^2}{\hat\sigma^2_{\tilde{\tau}}(b)} \leq \min_{b \neq a^*}   \frac{\left(\sum_{a \in [K]}w_\infty^b(a)\mu(a)\right)^2}{\sigma_\infty^2(\bm{w}_\infty^b)},\label{eq:key_identity_repeat} \\
    &\lim_{\alpha\rightarrow 0}\frac{c}{\hat\sigma^2_{\tilde{\tau}}(b)\tilde\tau\frac{2(\rho^2 + 1/(\tilde\tau\hat\sigma^2_{\tilde{\tau}}(b )))}{\rho^2} \log\left( 1+ \frac{\sqrt{\tilde\tau \hat\sigma^2_{\tilde{\tau}}( b )\rho^2 + 1 }}{2\alpha}\right)} = 0. \label{eq:107}
\end{align}
Line \eqref{eq:key_identity_repeat} follows directly from the result in Equation \eqref{eq:key_identity} that holds for all $b \neq a^*$, our random index $b$ satisfying $b \neq a^*$, and the fact that $\tilde\tau \geq t_0(\alpha)$ and $t_0(\alpha) \rightarrow \infty$ as $\alpha \rightarrow 0$. Line \eqref{eq:107} follows from the fact that $\hat\sigma^2_{\tilde{\tau}}(b) \geq m > 0$ by Assumption $(A3)$ of Theorem \ref{thm:error_control} and the fact that $\tilde\tau \geq t_0(\alpha)$, $t_0(\alpha) \rightarrow \infty$ as $\alpha \rightarrow 0$. By the limits above and the fact that $\frac{\left(\sum_{b \in [K]}w_\infty^a(b)\mu(b)\right)^2}{\sigma_\infty^2(\bm{w}_\infty^a)} > 0$ for all $b \neq a^*$, 
\begin{equation}
    \lim_{\alpha \rightarrow 0}  \frac{\tilde\tau}{{ \frac{2(\rho^2 + 1/(\tilde{\tau}\hat\sigma^2_{\tilde{\tau}}( b ) ))}{\rho^2} \log\left( 1+ \frac{\sqrt{\tilde\tau \hat\sigma^2_{\tilde{\tau}}( b )\rho^2 + 1 }}{2\alpha}\right)}} \leq \max_{b \neq a^*} \frac{\sigma_\infty^2(\bm{w}_\infty^b)}{\left(\sum_{a \in [K]}w_\infty^b(a)\mu(a)\right)^2}
\end{equation}
almost surely. To obtain our desired bound, we re-express the term in the limit above as
\begin{align}
    &\frac{\tilde\tau}{{ \frac{2(\rho^2 + 1/(\tilde{\tau}\hat\sigma^2_{\tilde{\tau}}(b ) ))}{\rho^2} \log\left( 1+ \frac{\sqrt{\tilde\tau \hat\sigma^2_{\tilde{\tau}}(b )\rho^2 + 1 }}{2\alpha}\right)}}=\\ & \frac{\tilde\tau}{{ \frac{2(\rho^2 + 1/(\tilde\tau\hat\sigma^2_{\tilde{\tau}}(b ) ))}{\rho^2} \left(\log\left( \frac{2\alpha + \sqrt{\tilde\tau \hat\sigma^2_{\tilde{\tau}}(b )\rho^2 + 1 }}{2}  \right) + \log(1/\alpha)  \right)}}=\\
    & \frac{\rho^2}{2(\rho^2 + 1/(\tilde\tau \hat\sigma^2_{\tilde{\tau}}(b )))}\left(\frac{\log(1/\alpha)}{\tilde\tau} + \frac{\log\left( \frac{2\alpha + \sqrt{\tilde{\tau} \hat\sigma^2_{\tilde{\tau}}(b )\rho^2 + 1 }}{2}  \right)}{\tilde\tau}\right)^{-1}.
\end{align}
Note that $\tilde{\tau} \rightarrow \infty$ due to $\tilde{\tau} \geq t_0(\alpha)$ and $t_0(\alpha) \rightarrow \infty$ as $\alpha \rightarrow 0$. Furthermore, note that $\hat\sigma_{\tilde\tau}^2(b) \geq m > 0$ almost surely by Condition $(A3)$ of Theorem \ref{thm:error_control}. As a result we obtain the desired almost-sure limiting expression
\begin{equation}
    \lim_{\alpha \rightarrow 0}\frac{\tilde\tau}{{ \frac{2(\rho^2 + 1/(\tilde{\tau}\hat\sigma^2_{\tilde{\tau}}(b ) ))}{\rho^2} \log\left( 1+ \frac{\sqrt{\tilde\tau \hat\sigma^2_{\tilde{\tau}}(b )\rho^2 + 1 }}{2\alpha}\right)}} = \lim_{\alpha \rightarrow 0}\frac{1}{2}\frac{\tilde\tau}{\log(1/\alpha)} \leq  \max_{b \neq a^*}\frac{\sigma_\infty^2(\bm{w}_\infty^b)}{\left(\sum_{a \in [K]}w_\infty^b(a)\mu(a)\right)^2}.
\end{equation}
Because $\tau \leq \tilde{\tau}$ deterministically, we obtain the desired result that $\lim_{\alpha \rightarrow 0} \frac{\tau}{\log(1/\alpha)} \leq \lim_{\alpha \rightarrow 0} \frac{\tilde\tau}{\log(1/\alpha)} \leq \max_{b\neq a^*} \frac{2\sigma_\infty^2(\bm{w}_\infty^b)}{\left(\sum_{a \in [K]}w_\infty^b(a)\mu(a)\right)^2} = \Gamma_1$ almost surely. 

\paragraph{Bounds on Expected Stopping Times}
Given our almost-sure upper bound on $\lim_{\alpha \rightarrow 0} \frac{\tau}{\log(1/\alpha)}$, we now show that the expected stopping time satisfies the same bound. First, we rearrange the deterministic bounds in Equation \eqref{eq:deterministic_bounds} for $\tilde{\tau}$ to obtain the following for some (random) index $b \neq a^*$:
\begin{equation}\label{eq:one_above}
    \left[\frac{2 \hat\sigma_{\tilde\tau}^2(b)}{\left(\frac{\sum_{i=1}^{\tilde\tau}{Z}_i(b)}{\tilde\tau}\right)^2}\right]\leq \frac{\tilde{\tau}}{\log(1/\alpha)} \leq \left[\frac{2 \hat\sigma_{\tilde\tau}^2(b)}{\left(\frac{\sum_{i=1}^{\tilde\tau}{Z}_i(b)}{\tilde\tau}\right)^2}\right] + o_\alpha(1). 
\end{equation}
where the asymptotically negligible term $o_{\alpha}(1)$ term vanishes as a result of (i) $\tilde{\tau} \geq t_0(\alpha)$ and (ii) $t_0(\alpha) \rightarrow \infty$ as $\alpha \rightarrow 0$. To show that $\EE[\tau/\log(1/\alpha)]$ has the same limiting upper bound as $\tau/\log(1/\alpha)$, we rearrange Equation \eqref{eq:one_above} as follows:
\begin{equation}\label{eq:one_more_above}
    1 \leq  \frac{\tilde{\tau}}{\log(1/\alpha)}\left[\frac{\left(\frac{\sum_{i=1}^{\tilde\tau}{Z}_i(b)}{\tilde\tau}\right)^2}{2 \hat\sigma_{\tilde\tau}^2(b)}\right] \leq 1+ c_1(\alpha).
\end{equation}
Note that $c_1(\alpha)$ is a vanishing, $o(1)$ constant with respect to $\alpha \rightarrow 0$ due to $\hat\sigma_{\tilde{\tau}}^2(b) \geq m$ for all $b \in [K]$. To proceed, we leverage Egorov's Theorem, which enables us to bound our expectation. For completeness, we provide a simplified version below. 

\begin{lemma}[Egorov's Theorem]\label{lem:egorov_thoerem}
    Let $X_\alpha(\omega) \in \RR$ be a sequence of real-valued measurable functions, and assume $\omega \in \Omega$, where $P(\Omega)= 1$. Assume that $X_\alpha(\omega) \rightarrow X$ $P$-almost surely as $\alpha \rightarrow 0$. Then, $\forall \epsilon > 0$, there exists a measurable subset $\Omega_{G,\epsilon} \subseteq \Omega$ such that $X_{\alpha}(\omega) \rightarrow X$ uniformly, and $\Omega_{B,\epsilon} \coloneq \Omega \setminus \Omega_{G, \epsilon}$ has $P(\Omega_{B, \epsilon}) < \epsilon$. 
\end{lemma}

Using Lemma \ref{lem:egorov_thoerem}, we can rewrite the middle term of the inequalities in Equation \eqref{eq:one_more_above} as follows, denoting $\Omega_{G, \epsilon}$ as the set of sample paths where $\left[\frac{\left(\frac{\sum_{i=1}^{\tilde\tau}{Z}_i(b)}{\tilde\tau}\right)^2}{2 \hat\sigma_{\tilde\tau}^2(b)}\right]$ uniformly converges to $\frac{(\sum_{a \in [K]}w_\infty^b(a) \mu(a))^2}{\sigma_\infty^2(\bm{w}_\infty^b)}$ for all $b \neq a^*$:
\begin{align}
    \left[\frac{\tilde\tau}{\log(1/\alpha)}\left(\frac{\left(\frac{\sum_{i=1}^{\tilde\tau}{Z}_i(b)}{\tilde\tau}\right)^2}{2 \hat\sigma_{\tilde\tau}^2(b)}\right) \right] &= \underbrace{ 
    \frac{\tilde\tau}{\log(1/\alpha)}\left(\frac{\left(\frac{\sum_{i=1}^{\tilde\tau}{Z}_i(b)}{\tilde\tau}\right)^2}{2 \hat\sigma_{\tilde\tau}^2(b)}\right) \mathbf{1}\left[ \Omega_{G, \epsilon} \right]}_{\text{Term $(a)$: on }\Omega_{G, \epsilon}}\label{l:term_1}\\
    &+ \underbrace{\frac{\tilde\tau}{\log(1/\alpha)}\left(\frac{\left(\frac{\sum_{i=1}^{\tilde\tau}{Z}_i(b)}{\tilde\tau}\right)^2}{2 \hat\sigma_{\tilde\tau}^2(b)}\right)\mathbf{1}[\Omega_{B, \epsilon}] }_{\text{Term $(b)$: on }\Omega_{B , \epsilon}}, \label{l:term_2}
\end{align}
For term $(b)$, note that the deterministic inequality in Equation \eqref{eq:one_more_above} ensures that 
\begin{equation}
    \frac{\tilde\tau}{\log(1/\alpha)}\left(\frac{\left(\frac{\sum_{i=1}^{\tilde\tau}{Z}_i(b)}{\tilde\tau}\right)^2}{2 \hat\sigma_{\tilde\tau}^2(b)}\right)\mathbf{1}[\Omega_{B, \epsilon}] \in [1, 1+c_1(\alpha)],
\end{equation}
where $c_1(\alpha) = o_\alpha(1)$. For term $(a)$, we leverage uniform convergence to bound its value. By definition of uniform convergence, note that for all $\omega \in \Omega_{G, \epsilon}$ and $\delta > 0$, there exists an $\alpha(\delta) \in (0,1)$ independent of $\omega$ such that for all $b \neq a^*$, 
\begin{equation}
    \quad \forall  \ 0<\alpha \leq \alpha(\delta), \ \omega \in \Omega_{G, \epsilon} \quad \left|\left(\frac{\left(\frac{\sum_{i=1}^{\tilde\tau}{Z}_i(b)}{\tilde\tau}\right)^2}{2 \hat\sigma_{\tilde\tau}^2(b)}\right)(\omega) - \frac{\left(\sum_{a \in [K]}w_\infty^b(a) \mu(a)\right)^2}{2\sigma_\infty^2(\bm{w}_\infty^b)}\right| \leq \delta,
\end{equation}
which implies that for all $\omega \in \Omega$, $\left(\frac{\left(\frac{\sum_{i=1}^{\tilde\tau}{Z}_i(b)}{\tilde\tau}\right)^2}{2 \hat\sigma_{\tilde\tau}^2(b)}\right)(\omega) = \frac{\left(\sum_{a \in [K]}w_\infty^b(a) \mu(a)\right)^2}{2\sigma_\infty^2(\bm{w}_\infty^b)} + c_2(\alpha)$, where $c_2(\alpha) = o_\alpha(1)$ is a vanishing term that does not depend on $\omega$ and only depends on $\alpha$. Using these results from term $(a)$, term $(b)$, and the deterministic bounds $\frac{\tilde\tau}{\log(1/\alpha)}\left(\frac{\left(\frac{\sum_{i=1}^{\tilde\tau}{Z}_i(b)}{\tilde\tau}\right)^2}{2 \hat\sigma_{\tilde\tau}^2(b)}\right) 
\in [1, 1+c_1(\alpha)]$, Equations \eqref{l:term_1} and \eqref{l:term_2} imply the following inequality: 
\begin{equation}
    \frac{\tilde\tau}{\log(1/\alpha)}\left( \frac{\left(\sum_{a \in [K] }w_\infty^b(a)\mu(a) \right)^2}{2\sigma^2_{\infty}(\bm{w}_\infty^b)} + c_2(\alpha)\right)\mathbf{1}[\Omega_{G, \epsilon}] + \mathbf{1}[\Omega_{B, \epsilon}]\leq 1+c_1(\alpha)
\end{equation}
We now take the minimum over $b \neq a^*$ for the SNR ratio, resulting in the inequality 
\begin{equation}
   \frac{\tilde\tau}{\log(1/\alpha)}\left( \min_{b \neq a^*}\frac{\left(\sum_{a \in [K] }w_\infty^b(a)\mu(a) \right)^2}{2\sigma^2_{\infty}(\bm{w}_\infty^b)} + c_2(\alpha)\right)\mathbf{1}[\Omega_{G, \epsilon}] + \mathbf{1}[\Omega_{B, \epsilon}]\leq 1+c_1(\alpha)
\end{equation}
Let $\alpha \leq \alpha\left(\min_{b \neq a^*}\frac{\left(\sum_{a \in [K] }w_\infty^b(a)\mu(a) \right)^2}{4\sigma^2_{\infty}(\bm{w}_\infty^b)}\right)$ small enough such that the vanishing term $c_2(\alpha)$ satisfies $|c_2(\alpha)| < \min_{b \neq a^*}\frac{\left(\sum_{a \in [K] }w_\infty^b(a)\mu(a) \right)^2}{4\sigma^2_{\infty}(\bm{w}_\infty^b)} $. This ensures $\left( \min_{b \neq a^*}\frac{\left(\sum_{a \in [K] }w_\infty^b(a)\mu(a) \right)^2}{2\sigma^2_{\infty}(\bm{w}_\infty^b)} + c_2(\alpha)\right)$ is bounded below by a positive constant. Taking expectations and rearranging, we obtain 
\begin{equation}\label{eq:almost_done}
    \EE\left[ \frac{\tilde\tau}{\log(1/\alpha)} \mathbf{1}[\Omega_{G, \epsilon}] \right] \leq \frac{1+c_1(\alpha) - \epsilon}{\left( \min_{b \neq a^*}\frac{\left(\sum_{a \in [K] }w_\infty^b(a)\mu(a) \right)^2}{2\sigma^2_{\infty}(\bm{w}_\infty^b)} + c_2(\alpha)\right)}.
\end{equation}
Note that the expectation must exist, as both sides of the inequality are dominated a constant function for sufficiently small $\alpha$. Taking the limit with respect to $\epsilon \rightarrow 0$ on both sides, we obtain
\begin{align}
    \lim_{\epsilon \rightarrow 0}\frac{1+c_1(\alpha) - \epsilon}{\left( \min_{b \neq a^*}\frac{\left(\sum_{a \in [K] }w_\infty^b(a)\mu(a) \right)^2}{2\sigma^2_{\infty}(\bm{w}_\infty^b)} + c_2(\alpha)\right)} &= \frac{1+c_1(\alpha)}{\left( \min_{b \neq a^*}\frac{\left(\sum_{a \in [K] }w_\infty^b(a)\mu(a) \right)^2}{2\sigma^2_{\infty}(\bm{w}_\infty^b)} + c_2(\alpha)\right)}, \\
    \lim_{\epsilon \rightarrow 0} \EE\left[ \frac{\tilde\tau}{\log(1/\alpha)} \mathbf{1}[\Omega_{G, \epsilon}] \right] &= \EE_{\epsilon}\left[ \frac{\tilde{\tau}}{\log(1/\alpha)} \right],
\end{align}
where the latter equality is valid due to the monotone convergence theorem and the non-negativity of $\tilde{\tau}/\log(1/\alpha)$. Our limits, combined with Equation \eqref{eq:almost_done}, yield
\begin{equation}
\EE_{\epsilon}\left[ \frac{\tilde{\tau}}{\log(1/\alpha)} \right] \leq \frac{1+c_1(\alpha)}{\left( \min_{b \neq a^*}\frac{\left(\sum_{a \in [K] }w_\infty^b(a)\mu(a) \right)^2}{2\sigma^2_{\infty}(\bm{w}_\infty^b)} + c_2(\alpha)\right)}.
\end{equation}
Taking limits with respect to $\alpha$ on both sides of our inequality, we obtain 
\begin{equation}
\lim_{\alpha \rightarrow 0}\EE_{\epsilon}\left[ \frac{\tilde{\tau}}{\log(1/\alpha)} \right] \leq \max_{b \neq a^*}\frac{2\sigma^2_{\infty}(\bm{w}_\infty^b)}{\left(\sum_{a \in [K] }w_\infty^b(a)\mu(a) \right)^2}.
\end{equation}
Because $\tilde{\tau} \geq \tau$ deterministically by definition, we obtain our desired result:
\begin{equation}
    \lim_{\alpha \rightarrow 0}\EE_{\epsilon}\left[ \frac{{\tau}}{\log(1/\alpha)} \right]\leq \lim_{\alpha \rightarrow 0}\EE_{\epsilon}\left[ \frac{\tilde{\tau}}{\log(1/\alpha)} \right] \leq \max_{b \neq a^*}\frac{2\sigma^2_{\infty}(\bm{w}_\infty^b)}{\left(\sum_{a \in [K] }w_\infty^b(a)\mu(a) \right)^2}.
\end{equation}

\paragraph{Negligibility of the Burn-in Period} We conclude our proofs by noting that our burn-in times $t_0(\alpha)$ are \textit{negligible} relative to the the order of the stopping time $\tau$, resulting in the same almost sure and expected stopping time bounds. We can account for our burn-in times $t_0(\alpha)$ with an upper bound on $\tau_0(\alpha)$ as 
\begin{equation}
    \tau_{t_0(\alpha)} = \inf\{t \geq t_0(\alpha): |C_t(H_t, \alpha)| \leq 1\}.
\end{equation}
Because of the condition $\lim_{\alpha \rightarrow 0}t_0(\alpha)/\log(1/\alpha) = 0$ and $\tau$ is of order $1/\log(1/\alpha)$ as $\alpha \rightarrow 0$ with probability one, it follows that $\tau_{t_0(\alpha)} = \tau$ almost surely as $\alpha \rightarrow 0$. By repeating the same exact argument above, we obtain $\lim_{\alpha \rightarrow 0}\frac{\tau_{t_0(\alpha)}}{\log(1/\alpha)} = \lim_{\alpha \rightarrow 0}\frac{\tau}{\log(1/\alpha)}$ and $\lim_{\alpha \rightarrow 0}\EE\left[\frac{\tau_{t_0(\alpha)}}{\log(1/\alpha)}\right] = \lim_{\alpha \rightarrow 0}\EE\left[\frac{\tau}{\log(1/\alpha)}\right]$, and our bounds hold for $\tau_{t_0(\alpha)}$ both almost surely and in expectation.

\newpage
\subsection{Proof of Lemma \ref{lem:convexity_policy}}

\paragraph{Proof Sketch} To prove this result, we leverage Danskin's Theorem to show that each inner maximization problem $F_a(\pi) = \min_{\bm{w} \in \Delta(a)}\frac{ \EE_{P_X}\left[\frac{v(x,b) + r_\infty(x,b)^2}{\pi(x,b)}\right]+ \EE_{P_X}\left[ \left(\sum_{b \in [K]} w(b) (g(x,b)- \mu(b)) \right)^2 - \left(\sum_{b \in [K]} w(b) r_\infty(x,b) \right)^2 \right] }{\left(\sum_{b \in [K]} w(b)  \mu(b) \right)^2}$ is strictly convex with respect to the function $\pi$. It then follows that $F(\pi)$ is strictly convex due to the maximum of $K-1$ strictly convex function begin strictly convex. By strict convexity of our objective function $F(\pi)$ and the fact that $\Pi$ is a convex set, we obtain that $\pi_* = \argmax_{\pi \in \Pi}F(\pi)$ is unique. To begin, we first start by stating Danskin's Theorem, which characterizes  the Frechet derivative of inner minimization problems $F_a(\pi)$. 

\begin{lemma}[Danskin's Theorem (Theorem 4.13 of \cite{bonnans2000perturbation})]\label{lem:danskin}
    Consider the function $v(u) \coloneqq \min_{x \in \Xcal, x \in \Theta} f(x, u)$, where $\Ucal$ is a Banach space, $\Xcal$ is a Hausdorff topolical space, $\Theta \subset \Xcal$ is nonempty and closed, and $f: \Xcal \times \Ucal \rightarrow \RR$ is continuous. Suppose that for all $x \in \Xcal$ the function $f(x, \cdot)$ is (Gâteaux) differentiable, such that $f(x,u)$ and $D_u f(x,u)$ are continuous on $\Xcal \times \Ucal$. Furthermore, assume that there exists an $\alpha \in \RR$ and compact set $C \subset \Xcal$ such that for every $u$ near $u_0 \in \Ucal$, $\text{lev}_{\alpha}f(\cdot, u) \coloneqq \{x \in \Theta: f(x, u) \leq \alpha\}$ is nonempty and contained in $C$. Then, $v(\cdot)$ is Fréchet directionally differentiable at $u_0$ and $\nabla_{d}v'(u_0) = \inf_{x \in \Scal(u_0)} D_u f(x, u_0) d$, where $S(u_0) \coloneqq \argmin_{x \in \Theta} f(x, u_0)$.
\end{lemma}

Recall that our original optimization problem can be rewritten as 
\begin{align}
    &\pi \in \argmin_{\pi} \max_{a \neq a^*} F_a(\pi), \quad F_a(\pi) = \min_{\bm{w} \in \Delta(a), \bm{w}^\top \bm{\mu} \geq 0} f(\pi, \bm{w}), \\
    &f(\pi, \bm{w}) = \frac{\sum_{b \in [K]} w(b)^2\EE_{P_X}\left[\frac{v(x,b) + r_\infty(x,b)^2}{\pi(x,b)}\right]+ \EE_{P_X}\left[ \left(\sum_{b \in [K]} w(b) (g(x,b)- \mu(b)) \right)^2 - \left(\sum_{b \in [K]} w(b) r_\infty(x,b) \right)^2 \right] }{\left(\sum_{b \in [K]} w(b)  \mu(b) \right)^2}.
\end{align}
We will apply Danskin's Theorem to the function $f(\pi, \bm{w})$ first, where $v(u) = F_a(\pi)$. Here, $\Xcal \in \RR^K$, $\Theta = \Delta(a)$, and $\Ucal $ is the space $L_2(P_X: \RR^K)$, which strictly contains our set of valid policies $\Pi$. Note that by Lemma \ref{lem:convergence_snr_maximizing_weights}, $x \in \Scal(u_0)$ only has a single element, which we denote as $\bm{w}^a_\pi$ for each $F_a(\pi)$. Then, it follows that the directional derivative of $F_a(\pi)$ with respect to direction $d$ in the space of policies $\pi$ is equal to 
\begin{equation}
    \nabla_{d} F_a(\pi) = \left\langle D_{\pi}f_a(\pi, \bm{w}_\pi^a,) d \right\rangle_{L_2(P_X:\RR^K)} = \left\langle\left[\frac{-w_\pi^a(b)^2 \left[\frac{v(x,b) + r_\infty(x,b)^2}{\pi(x,b)^2} \right] }{\left(\sum_{b \in [K]} w_\pi^a(b)\mu(b) \right)^2} \right]_{x \in \Xcal, b \in [K]}, d \right\rangle_{L_2(P_X:\RR^K)}.
\end{equation}
By taking the second-order Fréchet derivative of $F_a(\pi)$, we obtain
\begin{equation}
    \nabla^2 F_a(\pi)(d, v) = \int_x \sum_{b \in [K]}\frac{2w_\pi^a(b)^2\left[\frac{v(x,b) + r_\infty(x,b)^2}{\pi(x,b)^3}\right] }{\left(\sum_{b \in [K]} w_\pi^a(b) \mu(b) \right)^2} d(x,b) v(x,b) dP_X. 
\end{equation}
Note that for any $h \in L_2(P_X: \RR^K)$, $\nabla^2 F_a(\pi)(h,h) \geq 0$, so $F_a(\pi)$ is convex with respect to $\pi$. Note that because $v(x,b) > 0$ for all $x$, it follows that $F_a(\pi)$ is strictly convex with respect to $\pi$. 

We now have that each functional $F_a(\pi)$ is strictly convex with respect to the function $\pi$. To see that our pointwise maximum over $a \neq a^*$ retains strict convexity, we use the standard definition of strict convexity. Let $\lambda \in (0,1)$, $\pi_1, \pi_2 \in \Pi$, and $\pi(\lambda) = \lambda\pi_1 + (1-\lambda)\pi_2$. Then, 
    \begin{equation}
        \max_{a \neq a^*} F_a(\pi(\lambda))  < \max_{a \neq a^*} \lambda F_a(\pi_1) + (1-\lambda)F_a(\pi_2) \leq \lambda \max_{a \neq a^*} F_a(\pi) + (1-\lambda) \max_{a \neq a^*} F_a(\pi_2),
    \end{equation}
    where the first inequality follows from the strict convexity of $F_a(\pi)$. Thus, the function $F(\pi)$ is strictly convex with respect to $\pi$. Lastly, by noting that $\Pi$ forms a convex set in $L_2(P_X: \RR^K)$, we obtain that our initial problem is minimizing a strictly convex objective over a convex set, resulting in a unique optimal $\pi$.

\subsection{Proof of Lemma \ref{lem:optimal_policy_structure}}

To prove the results of this Lemma, we first prove that the optimal policy $\pi_*$ takes a simple form characterized by Lemma \ref{lem:optimal_pi_reduction}. After establishing the result of Lemma \ref{lem:optimal_pi_reduction}, we provide the desired result by re-parameterizing the results of Lemma \ref{lem:optimal_pi_reduction} in Lemma \ref{lem:log_q_reformulation}. We begin with our proof of Lemma \ref{lem:optimal_pi_reduction} below. 

\begin{lemma}[Structure of Optimal $\pi$]\label{lem:optimal_pi_reduction}
    Let $G(\pi)$ be the expression presented in Equation \eqref{eq:G_pi}, and let all conditions of Lemma \ref{lem:optimal_policy_structure} hold. Then, for all $\pi \in \argmax_{\pi \in \Pi} G(\pi)$, there exists a corresponding vector $\bm{q}_\pi \in \RR^K_{++}$ such that 
    \begin{equation}
        \pi(x,b) = \frac{\sqrt{q(b) (v(x,b) + r_\infty(x,b)^2)}}{\sum_{b \in [K]}\sqrt{q(b) (v(x,b)+ r_\infty(x,b)^2)}},
    \end{equation}
    where $r_\infty(x,b) = g_\infty(x,b) - g(x,b)$ denotes the limiting error for the $(x,b)$ pair. 
\end{lemma}

\begin{proof}[Proof of Lemma \ref{lem:optimal_pi_reduction}]
    For each $a\neq a^*$, let $\bm{w}_{\pi_*}^a$ denote the unique weights\footnote{We prove the uniqueness of such weights in Lemma \ref{lem:convergence_snr_maximizing_weights}.} that satisfy the following equation:
    \begin{equation}
        \bm{w}_{\pi_*}^a = \argmin_{\bm{w} \in \Delta(a), \bm{w}^\top \bm{\mu}\geq 0} f(\pi_*, \bm{w}).
    \end{equation}
    Then, note that the optimal $\pi_*$ is also the solution to the problem using fixed weight vectors $\bm{w}_{\pi_*}^a$:
    \begin{equation}
        \min_{\pi \in \Pi}\max_{a \neq a^*} f(\pi, \bm{w}_{\pi_*}^a).
    \end{equation}
    To simplify notation, we will use $f_a(\pi) \coloneqq f(\pi, \bm{w}_{\pi_*}^a)$ throughout the remainder of this section. We now show that the problem presented above is a convex optimization problem. First, rewriting our optimization problem in epigraph form to remove the inner maximum over $a \neq a^*$, we obtain
    \begin{align}
        \pi_* \in \argmin_{\pi \in \Pi, \  c \in \RR} c  \quad \text{ s.t. } \quad f(\pi)-c \leq 0  \quad \forall a \neq a^*.
    \end{align}
    By the convexity of $F_a(\pi)$ (proof in Lemma \ref{lem:convexity_policy}), the function $f(\pi)$ is strictly convex with respect to $\pi$, and therefore our problem is simply an affine objective with a convex feasible set, which is a convex problem. Note that a trivial interior solution exists by setting $\pi(x,b) = 1/K$ for all $x\in \Xcal$ and $b \in [K]$, and therefore Slater's condition holds. As a result of Slater's condition, we then have that the KKT conditions characterize the optimal solution set of $\pi$. Writing out our optimization problem explicitly, we obtain:
    \begin{align}
        &\min_{\pi, c \in \RR} c \\
        & \text{s.t. } \ \pi(x,\cdot) \in \Delta^K \ P_X\text{ almost surely}, \\
        &\quad f_a(\pi) - c \leq 0  \text{ for all } a \neq a^*,\label{l:norm_constraint}
    \end{align}
    which corresponds to the following Lagrangian formulation
    \begin{equation}
        \Lcal(\pi, c, \bm{\lambda}, \bm{\gamma}, \bm{\epsilon}) = c + \left(\sum_{b \neq a^*} \lambda(b)(f_a(\pi)- c)\right)+ \EE_{P_X}\left[\gamma(x) \left( \sum_{b \in [K]}\pi(x,b) - 1\right) \right] - \EE_{P_X}\left[ \sum_{b \in [K]}\epsilon(x,b) \pi(x,b)  \right],
    \end{equation}
    where $\epsilon$ is a nonnegative function, $\bm{\gamma}$ is a function, and $\bm{\lambda} \in \RR_+^K$ is a nonnegative vector, following from Section 3.2 of \cite{shapiro2014lectures}. Grouping terms, $\Lcal$ can be reduced to
    \begin{align}
        \Lcal(\pi, c, \bm{\lambda}, \bm{\gamma}, \bm{\epsilon}) =& c\left( 1 - \sum_{b \neq a^*} \lambda(b) \right) \\
        &+\sum_{b \neq a^*} \lambda(b)\frac{\EE_{P_X}\left[\left(\sum_{a \in [K]} w_{\pi_*}^b(a)(g(x,a) - \mu(a)) \right)^2 + \left(\sum_{b \in [K]} w_{\pi_*}^a(b) r_\infty(x,b)  \right)^2\right]}{\left(\sum_{a \in [K]} w_{\pi_*}^b(a)\mu(a)\right)^2} \\
        &+ \sum_{a \in [K]}\underbrace{\left(\sum_{b \neq a^*}\lambda(b) \frac{w_{\pi_*}^b(a)^2}{\left(\sum_{a' \in [K]} w_{\pi_*}^b(a')\mu(a') \right)^2}\right)}_{q(a)} \EE\left[ \frac{v(x,a) + r_\infty(x,b)^2}{\pi(x,a)}\right] \label{l:dual_plug_in} \\
        &+\EE_{P_X}\left[\gamma(x) \left( \sum_{b \in [K]}\pi(x,b) - 1\right) \right] - \EE_{P_X}\left[ \sum_{b \in [K]}\epsilon(x,b) \pi(x,b)  \right],
    \end{align}

    where $q(a) = \left(\sum_{b \neq a^*}\lambda(b) \frac{w_\infty^b(a)^2}{\left(\sum_{a' \in [K]} w_\infty^b(a')\mu(a') \right)^2}\right)$ defines the mixture weights specified above. Note that by the stationary KKT conditions for parameter $c$, we must have
    \begin{equation}
        \frac{\partial}{\partial c} \Lcal = 1-\sum_{b\neq a^*} \lambda(b)= 0,
    \end{equation}
    and therefore $\bm{\lambda} \in \Delta^{K-1}$. Using the KKT stationary conditions with respect to $\pi(x,b)$, we obtain
    \begin{equation}
        \forall x\in \Xcal, b\in [K], \ \frac{\partial}{\partial \pi(x,b)} \Lcal = -\frac{q(b) (v(x,b) + r_\infty(x,b)^2)}{\pi^2(x,b)} + \gamma(x) - \epsilon(x,b) = 0.
    \end{equation}
    Because $\pi(x,b) = 0$ results in an infinite objective value and our goal is to minimize the objective $c$, $\pi_*(x,b) > 0$. By complimentary slackness, $\epsilon_*(x,b)\pi_*(x,b) = 0$. Thus, $\epsilon(x,b) = 0$ for all $b \in \Gcal(x)$, and therefore 
    \begin{equation}
        \forall b \in \Gcal(x), \ \frac{q(b) (v(x,b) + r_\infty(x,b)^2)}{\pi^2(x,b)} =  \gamma(x) \implies \forall b \in \Gcal(x), \ \pi(x,b) = \sqrt{\frac{q(b)(v(x,b) + r_\infty(x,b)^2) }{\gamma(x)}}.
    \end{equation}
    By the primal feasibility condition, note that $\sum_{b \in \Gcal(x)} \pi(x,b) = 1$, and therefore 
    \begin{equation}
        \sum_{b \in \Gcal(x)}\sqrt\frac{q(b) v(x,b)}{\gamma(x)} = 1 \implies  \gamma(x) = \left(\sum_{b \in \Gcal(x)} \sqrt{q(b)\left(v(x,b) + r_\infty(x,b)^2\right) }\right)^2. 
    \end{equation}
    Plugging the value of $\gamma(x)$ back into our solution, we obtain 
    \begin{equation}\label{eq:optimal_policy_soln}
        \pi_*(x,b) = \frac{\sqrt{q(b)(v(x,b) + r_\infty(x,b)^2)}}{\sum_{b \in [K]} \sqrt{q(b)(v(x,b)+ r_\infty(x,b)^2)}}. 
    \end{equation}
    Note that if $q(b) = 0$ for any $b \in [K]$, our objective takes an infinite value. Thus, $q(b) >0$ for all $b \in [K]$. To show that this structure holds for our original optimization problem (without fixed weights), note that Equation \eqref{eq:optimal_policy_soln} holds for all $\tilde\pi$ that satisfy  
    \begin{equation}
        \tilde\pi = \argmin_{\pi \in \Pi} \max_{a\neq a^*} f(\pi, \bm{w}_{\pi_*}^a),
    \end{equation}
    which has the same exact objective value at the minimizing solution as 
    \begin{equation}
        \pi_* = \argmin_{\pi \in \Pi} \max_{a \neq a^*} \min_{\bm{w}\in \Delta(a), \bm{w}^\top \bm{\mu}\geq 0 } f(\pi, \bm{w}).
    \end{equation}
    The solution to the latter equation is unique (as shown in Lemma \ref{lem:convexity_policy}), and by definition, the first problem and the second problem have the same objective value. Thus, it must be that $\pi_* \in \{\tilde\pi \in \Pi: \tilde\pi = \argmin_{\pi \in \Pi} \max_{a\neq a^*} f_a(\pi, \bm{w}_{\pi_*}^a) \}$, and thus $\pi_*$ satisfies this structure as well. 
\end{proof}

Given that our optimal policy $\pi$ satisfies the simple parametric model with $K$ parameters, we now turn to solving the optimization problem for our reduced set of parameters $\bm{q} \in \RR^K_{++}$. However, naively plugging in the structure of $\pi$ with respect to $\bm{q}$ in our objective problem results in nonconvexity of our initial problem $\min_{\bm{q} \in \RR^K_{++}} F(\pi)$. Instead, we first provide a simple reparameterized model that builds upon the results of Lemma \ref{lem:optimal_pi_reduction} and maintains the strict convexity results of $F(\pi)$ with respect to $\pi$.

\begin{lemma}[Reformulation of Optimal $\pi$]\label{lem:log_q_reformulation}
    Let $G(\pi)$ be the expression presented in Equation \eqref{eq:G_pi}, and let all conditions of Lemma \ref{lem:optimal_policy_structure} hold. Then, for all $\pi \in \argmax_{\pi \in \Pi}G(\pi)$, there exists a corresponding vector $\bm\theta \in \RR^K$ with $\theta(K)=0$ such that 
    \begin{equation}
        \pi_*(x,b)^{-1} = \sum_{a \in [K]}\frac{ \sqrt{\left(v(x,a) + r_\infty(x,a)^2 \right)} }{\sqrt{\left(v(x,b) + r_\infty(x,b)^2 \right)} }\exp\left(\theta(a) - \theta(b)\right)
    \end{equation}
\end{lemma}

\begin{proof}[Proof of Lemma \ref{lem:log_q_reformulation}]
    To prove this result, first note that by Lemma \ref{lem:optimal_pi_reduction}, it holds that there exists a $\bm{q} \in \RR_{++}^K$ such that 
    \begin{equation}
        \pi_*(x,b) = \frac{\sqrt{q(b) (v(x,b) + r_\infty(x,b)^2)}}{\sum_{a \in [K]} \sqrt{q(a) (v(x,a) + r_\infty(x,a)^2)}}
    \end{equation}
    Note that $\sqrt{q(a)} > 0$ by Lemma \ref{lem:optimal_pi_reduction}, and therefore we set $\theta(a) = \log(\sqrt{q(a)})$, where $\bm{\theta} \in \RR^K$ is the same set as $\bm{q} \in \RR^K_{++}$. Thus, we can re-express $\pi_*$ as 
    \begin{equation}
        \pi_*(x,b) = \frac{\exp(\theta(b))\sqrt{(v(x,b) + r_\infty(x,b)^2)}}{\sum_{a \in [K]}\exp(\theta(a))\sqrt{(v(x,a) + r_\infty(x,a)^2)}}.
    \end{equation}
    To ensure that our reformulation of $\pi_*$ preserves the strict convexity, we will show that fixing $\theta(K) = 0$ is equivalent to our reformulation above. First, note that by dividing both numerator and denominator by $\exp(\theta(K))$, 
    \begin{equation}
         \pi_{*}(x,b)=\frac{\exp(\theta(b) - \theta(K))\sqrt{(v(x,b) + r_\infty(x,b)^2)}}{\sum_{a \in [K]}\exp(\theta(a) - \theta(K))\sqrt{(v(x,a) + r_\infty(x,a)^2)}}
    \end{equation}
    is identical to our first formulation. Let $\bm\theta' \in \RR^K$ with $\theta(K) = 0$. Then, $\theta'(b) = \theta(b) - \theta(K)$ provides the equivalent policy. Therefore, for our optimal policy $\pi_*$, there exists a $\bm{\theta}' \in \RR^K$ with $\theta(K) = 0$ such that 
    \begin{equation}
        \pi_*(x,b) = \frac{\exp(\theta'(b))\sqrt{(v(x,b) + r_\infty(x,b)^2)}}{\sum_{a \in [K]}\exp(\theta'(a))\sqrt{(v(x,a) + r_\infty(x,a)^2)}}.
    \end{equation}
\end{proof}

The result of Lemma \ref{lem:optimal_pi_reduction}, paired with the reformulation in Lemma \ref{lem:log_q_reformulation}, obtains the results of Lemma \ref{lem:optimal_policy_structure}.

\subsection{Proof of Lemma \ref{lem:subgradient_lemma}}

The proof of Lemma \ref{lem:subgradient_lemma} follows from an application of Danskin's Theorem (Lemma \ref{lem:danskin}) and a standard result for optimization over maxima. We provide the latter result in Lemma \ref{lem:clarke_subgradient}.

\begin{lemma}[Subgradient Set for Pointwise Maxima of Convex Functions (Theorem 10.31 of \citealp{rockafellar2009variational})]\label{lem:clarke_subgradient}

    Let $f_i: \Xcal \rightarrow \RR$ be convex, differential functions with respect to $x$ for all $i \in [K]$, and let $g(x) = \max_{i \in [K]} f_i(x)$. Then, the subgradient set of $g$ evaluated at point $x$, denoted as $\partial g(x)$, is given by
    \begin{equation}
        \partial g(x) = \text{conv}\left( \{\nabla_x f_i(x): f_i(x) = g(x) \} \right),
    \end{equation}
    where $\text{conv}\left(\{v_i\}_{i \in \Scal} \right)$ denotes the convex hull of functions $\{v_i\}_{i \in \Scal}$. 
\end{lemma}

Our results follow from applying Danskin's Theorem to estimated functions $F_{a,t}$ (defined in Equations \ref{eqs:empirical_counterpart} - \ref{eqs:empirical_counterpart_2}), and then directly applying Lemma \ref{lem:clarke_subgradient}. To proceed, we first derive the gradient of estimated functions $F_{a,t}$. Let $\EE_{\PP_t(X)}[f_t(x)] = \frac{1}{t}\sum_{i=1}^t f_i(x_i)$ denote the empirical measure with respect to $X$ and a sequence of $\Fcal_{i-1}$-measurable functions $(f_i)_{i \in \NN}$. By direct application of Danskin's Theorem on the function $F_{a,t}(\bm{\theta})$,

\begin{align}
    &\frac{\partial}{\partial \theta(c)} F_{a,t}(\bm\theta) =  \bigg(\sum_{b \in [K]} w_{\bm\theta}^a(b)^2 \EE_{\PP_t(X)}\left[\sqrt{V_t(x,b)V_t(x,c)} \right]\exp (\theta(c) - \theta(b)) \\
    & - w_{\bm\theta}^a(c)^2\sum_{a \in [K] }\EE_{\PP_t(X)}\left[ \sqrt{V_t(x,a)V_t(x,c)}  \right]\exp(\theta(a) - \theta(c))\bigg) \bigg/ \left(\sum_{b \in [K]} w_{\bm{\theta}}^a(b) \hat\mu_t(b)\right)^2 \\
    & = \sum_{b \in [K]}\frac{\EE_{\PP_t(X)}\left[\sqrt{V_t(x,b)V_t(x,c)} \right]}{\left(\sum_{b \in [K]} w_{\bm{\theta}}^a(b) \hat\mu_t(b)\right)^2} \left(  w_\theta^a(b)^2\exp(\theta(c) - \theta(b)) - w_\theta^a(c)^2 \exp(\theta(b)-\theta(c))\right), \label{eq:first_order_partial}
\end{align}
where $\bm{w}_{\bm\theta}^a$ denotes the optimal, unique $\bm{w}^a \in \Delta(a)$, $\bm{w}^\top \hat{\bm{\mu}}_{t-1} \geq 0$ vectors that maximize $f_t(\bm{\theta}, \bm{w})$ for a given $\bm\theta$. Before applying Lemma \ref{lem:clarke_subgradient}, we first establish (i) the uniqueness of $\bm{w}_{\bm\theta}^a$ for each $a \not\in \argmax_{b \in [K]} \hat\mu_{t-1}(b)$ (such that the gradients of $F_{a,t}$ are as shown above) and (ii) strict convexity of $F_{a,t}$. 

\paragraph{Proof of Unique Weights} To show that the vectors $\bm{w}_{\bm\theta}^a$ for all $a \not\in \Acal_t(\bm{\theta})$ are unique, we apply Lemma \ref{lem:unique_optima} to the empirical SNR ratio $f_t^{-1/2}(\bm{\theta}, \bm{w})$ with respect to $\bm{w}$. Note that $\bm{w}_{\bm{\theta}}^a$ is defined as
\begin{equation}
    \bm{w}_{\bm{\theta}}^a = \argmin_{\bm{w} \in \Delta(a) , \bm{w}^\top \hat{\bm{\mu}}_{t-1} \ge 0} f_t(\bm{\theta}, \bm{w}) = \argmax_{\bm{w}\in \Delta(a)} f^{-1/2}_t(\bm{\theta}, \bm{w}),
\end{equation}
where $f_t^{-1/2}(\bm{\theta}, \bm{w})$ is defined as
\begin{align}
    f_t^{-1/2}(\bm{\theta}, \bm{w}) &= \frac{\sum_{b \in [K]} w(b) \hat\mu_{t-1}(b)}{\sqrt{ \sum_{b \in [K]} \frac{w(b)^2}{t} \sum_{i=1}^t\left[V_i(X_i,b)\sum_{a \in [K]}\frac{ \sqrt{V_i(X_i,a) } }{\sqrt{V_i(X_i,b) } }\exp\left(\theta(a) - \theta(b)\right)\right] + l_t(\bm{w})}}\\
    l_t(\bm{w}) &= \frac{1}{t}\sum_{i=1}^t\left[ \left(\sum_{b \in [K]} w(b) (g_i(X_i,b)- \hat\mu_{t-1}(b)) \right)^2 \right].
\end{align}
Note that the numerator of $f_t^{-1/2}$ is strictly positive for any $\bm{w}_{\bm{\theta}}^a$ due to $a \not\in \argmax_{b \in [K]}\hat\mu_{t-1}(b)$. Thus, to apply Lemma \ref{lem:unique_optima}, it only remains to show that the denominator is strictly convex and positive. Because $V_i(X_i, a) \geq \epsilon$ for all $X_i \in \Xcal$, $a \in [K]$, $t \in \NN$, it follows that the denominator of $f_t^{-1/2}$ is positive. To show that the denominator of $f_t^{-1/2}$ is strictly convex, we first rewrite the squared denominator in matrix notation. The term $l_t(\bm{w})$ can be expressed as
\begin{equation}
    l_t(\bm{w}) = \bm{w}^\top \bm{D}\bm{w}, \quad \bm{D} = \left(\frac{1}{t}\sum_{i=1}^t \bm{u}_i\bm{u}_i^\top \right)\in \RR^{K\times K},\quad \bm{u}_i = \begin{bmatrix} g_i(X_i, 1) - \hat\mu_{t-1}(1) \\ 
    \cdots \\
    g_i(X_i, K) - \hat\mu_{t-1}(K)\end{bmatrix} \in \RR^K,
\end{equation}
where the matrix $\bm{D}$ is positive semi-definite by construction. For the remaining term in the squared denominator of $f_t^{-1/2}$, we rewrite the terms in matrix notation as 
\begin{align}
     &\sum_{b \in [K]} \frac{w(b)^2}{t} \sum_{i=1}^t\left[V_i(X_i,b)\sum_{a \in [K]}\frac{ \sqrt{V_i(X_i,a) } }{\sqrt{V_i(X_i,b) } }\exp\left(\theta(a) - \theta(b)\right)\right] = \bm{w}^\top \bm{E} \bm{w}, \\
     &\bm{E} = \text{diag}(\bm{c}) \in \RR^{K\times K}, \quad \bm{c} = \begin{bmatrix} \frac{1}{t}\sum_{i=1}^t\sum_{a \in [K]} \sqrt{V_i(X_i, 1)V_i(X_i, a)} \exp\left(\theta(a) - \theta(1)\right) \\
     \cdots \\
     \frac{1}{t}\sum_{i=1}^t\sum_{a \in [K]} \sqrt{V_i(X_i, K)V_i(X_i, a)} \exp\left(\theta(a) - \theta(K)\right)
     \end{bmatrix} \in \RR^K.
\end{align}
The vector $c \in \RR^K$ is strictly positive, and therefore the matrix $\bm{E}$ is a positive definite matrix. Combining both reformulations, we obtain that the denominator of $f_t^{-1/2}$ with respect to $\bm{w}$ is equal to 
\begin{equation}
    \sqrt{\sum_{b \in [K]} \frac{w(b)^2}{t} \sum_{i=1}^t\left[V_i(X_i,b)\sum_{a \in [K]}\frac{ \sqrt{V_i(X_i,a) } }{\sqrt{V_i(X_i,b) } }\exp\left(\theta(a) - \theta(b)\right)\right] + l_t(\bm{w})} = \sqrt{\bm{w}^\top \left( \bm{E} + \bm{D} \right)\bm{w}} = \|\bm{w} \|_{\bm{E}+\bm{D}},
\end{equation}
where $\| \cdot\|_{\bm{M}}$ denotes the norm with respect to the inner product $\langle \bm{x},\bm{y}\rangle_{\bm{M}} =  \bm{x}^\top\bm{M}\bm{y}$. Because $\bm{E}$ is positive definite and $\bm{D}$ is positive semi-definite, $\bm{E}+ \bm{D}$ is positive definite, and therefore the norm $\| \bm{w} \|_{\bm{E} + \bm{D}}$ is strictly convex with respect to $\bm{w}$. Thus, the denominator of $f_t^{-1/2}$ is positive and strictly convex, the numerator is affine. By direct application of Lemma \ref{lem:unique_optima}, it then follows that $f_t^{-1/2}$ has a unique maximizing $\bm{w}_{\bm\theta}^a$, and therefore $f_t$ has a unique minimizing $\bm{w}_{\bm{\theta}}^a$ for all $a \not\in \argmax_{b \in [K]} \hat\mu_{t-1}(b)$ under the constraints that $\bm{w}^\top \hat{\bm{\mu}}_{t-1} \geq 0$ and $\bm{w} \in \Delta(a)$. Thus, the gradient of $F_{a,t}(\bm{\theta})$ in Equation \eqref{eq:first_order_partial} is correct by direct application of Danskin's Theorem (Lemma \ref{lem:danskin}).  

\paragraph{Strict Convexity of $F_{a,t}(\bm{\theta})$} To show (strict) convexity of functions $F_{a,t}(\bm{\theta})$, we take second partial derivatives with respect to $\bm{\theta}$ below:
\begin{align}
    &\frac{\partial^2}{\partial \theta(c)\partial\theta(b)} F_{a,t}(\bm{\theta}) = -\frac{\EE_{\PP_t(X)}\left[\sqrt{v_t(x,b)v_t(x,c)} \right]}{\left(\sum_{b \in [K]} w_{\bm{\theta}}^a(b) \hat\mu_t(b)\right)^2} \left(  w_\theta^a(b)^2\exp(\theta(c) - \theta(b)) + w_\theta^a(c)^2 \exp(\theta(b)-\theta(c))\right),\\
    & \frac{\partial^2}{\partial^2 \theta(c)} F_{a,t}(\bm{\theta}) = \sum_{b \in [K]}\frac{\EE_{\PP_t(X)}\left[\sqrt{v_t(x,b)v_t(x,c)} \right]}{\left(\sum_{b \in [K]} w_{\bm{\theta}}^a(b) \mu(b)\right)^2} \left(  w_\theta^a(b)^2\exp(\theta(c) - \theta(b)) + w_\theta^a(c)^2 \exp(\theta(b)-\theta(c))\right).
\end{align}
To see that our Hessian $H$ of $F(\bm\theta)$ is positive definite, consider any vector $\bm{z} \in \RR^K$ with $z(K) = 0$. Then, 
\begin{equation}
    \bm{z}^\top H \bm{z} = \frac{1}{\left(\sum_{b \in [K]} w_\theta^a(b)\hat\mu_t(b) \right)^2}\sum_{a \in [K]}\sum_{b \in [K]} w_\theta^a(b)^2\EE_{\PP_t(X)}\left[\sqrt{v_t(x,b)v_t(x,c)} \right]\exp(\theta(a) - \theta(b))(z(a) - z(b))^2,
\end{equation}
which is strictly nonnegative for any $\bm{z}$. To show that our expression is strictly positive, note that our expression can only be zero if $(z(a) - z(b))^2 = 0$ for any $a, b \in [K]$. Note that $z(K) = 0$, so for our expression to be zero, we require $\bm{z} = 0$. Thus, our Hessian is positive definite, and each $F_{a,t}(\bm{\theta})$ is strictly convex. Because $G_t(\bm{\theta})$ is a maximum of $|\Acal_t(\bm{\theta})|$ strictly convex functions with respect to $\bm{\theta}$, note that $G_t(\bm{\theta})$ is also strictly convex, as shown in the proof of Lemma \ref{lem:convexity_policy}.

\paragraph{Obtaining the Subgradient Set} To obtain the subgradient set shown in Lemma \ref{lem:subgradient_lemma}, we now apply Lemma \ref{lem:clarke_subgradient} directly to $G_t(\bm{\theta})$. For all $a \in [K]$, the function $F_{a,t}(\bm{\theta})$ is a convex, differential function with gradients defined Equation \eqref{eq:first_order_partial}. By direct application of Lemma \ref{lem:clarke_subgradient}, we conclude that the subgradient set of $G_t(\bm{\theta})$ is as defined in Lemma \ref{lem:subgradient_lemma}.

\subsection{Proof of Theorem \ref{thm:projected_subgradient_descent_convergence}}

Theorem \ref{thm:projected_subgradient_descent_convergence} makes two claims: (i) $\pi_t(x,a) \geq 1/\kappa$ for all $t \in \NN$, $x \in \Xcal$, $a \in [K]$, and (ii) $\lim_{t \rightarrow \infty} \|\pi_t(\cdot, a) - \pi_\infty\|_{L_2(P_{X|H_{t-1}})} = 0$ almost surely. To begin, we start with our strict positivity result. 

\subsubsection{Proof of Strict Positivity}
Strict positivity is a direct consequence of the bounds $[\epsilon, B^2]$ and $[-S,S]$ enforced on $V_t$ and $\bm{\theta}_t$ respectively. Recall that our sampling scheme takes the form
\begin{equation}
    \pi_t^{-1}(x,b) = \sum_{a \in [K]} \sqrt{\frac{V_t(x,a)}{V_t(x,b)}} \exp(\theta_t(a) - \theta_t(b)),
\end{equation}
and by the bounds $|\theta(a)| < S$ for all $a \neq K$ and $V_t(x,a) \geq \epsilon$ for all $x \in \Xcal, a \in [K], t \in \NN$, 
\begin{equation}\label{eq:l0}
    0 <\pi_t^{-1}(x,b) = \sum_{a \in [K]} \sqrt{\frac{V_t(x,a)}{V_t(x,b)}} \exp(\theta_t(a) - \theta_t(b)) \leq K \sqrt{\frac{B^2}{\epsilon}}\exp(2S) < \infty.
\end{equation}
Because $\pi_t^{-1}(x,b) \leq K \sqrt{\frac{B^2}{\epsilon}}\exp(2S)$, it follows that $\pi_t(x,b) \geq 1/\kappa$ for $\kappa = K \sqrt{\frac{B^2}{\epsilon}}\exp(2S)$.  

\subsubsection{Proof of Convergence} 
To prove that $\|\pi_t(\cdot, b) - \pi_\infty(\cdot, b) \|_{L_2(P_{X|H_{t-1}})}$ almost surely for all $b \in [K]$, we first show that $$\| \bm\theta_{t} - \bm{\theta}_{*}\|_{2} \rightarrow 0$$ almost surely is sufficient. The $L_2$ norm $\|\pi_t(\cdot, b) - \pi_\infty(\cdot, b) \|_{L_2(P_{X|H_{t-1}})}$ is upper bounded by 
\begin{align}
    &\quad \|\pi_t(\cdot, b) - \pi_\infty(\cdot, b) \|_{L_2(P_{X|H_{t-1}})} \\
    &=\left\|\frac{1}{\sum_{a \in [K]} \sqrt{\frac{V_t(x,a)}{V_t(x,b)}} \exp(\theta_t(a) - \theta_t(b))} - \frac{1}{\sum_{a \in [K]} \sqrt{\frac{V_\infty(x,a)}{V_\infty(x,b)}} \exp(\theta_\infty(a) - \theta_\infty(b))} \right\|_{L_2(P_{X|H_{t-1}})} \\ 
    &\leq  \left(K \sqrt{\frac{B^2}{\epsilon}}\exp(2S)\right)^2\left\|  \sum_{a \in [K]} \left(\sqrt{\frac{V_\infty(x,a)}{V_\infty(x,b)}} \exp(\theta_\infty(a) - \theta_\infty(b)) - \sqrt{\frac{V_t(x,a)}{V_t(x,b)}} \exp(\theta_t(a) - \theta_t(b)) \right) \right\|_{L_2(P_{X|H_{t-1}})} \label{eq:l1}\\ 
    & \leq \left(K \sqrt{\frac{B^2}{\epsilon}}\exp(2S)\right)^2\left\|  \sum_{a \in [K]} \sqrt{\frac{V_\infty(x,a)}{V_\infty(x,b)}} \left(\exp(\theta_\infty(a) - \theta_\infty(b)) - \exp(\theta_t(a) - \theta_t(b)) \right) \right\|_{L_2(P_{X|H_{t-1}})}\label{eq:l2} \\ 
    &\quad + \left(K \sqrt{\frac{B^2}{\epsilon}}\exp(2S)\right)^2\left\|  \sum_{a \in [K]} \left(\sqrt{\frac{V_\infty(x,a)}{V_\infty(x,b)}}  - \sqrt{\frac{V_t(x,a)}{V_t(x,b)}}\right) \exp(\theta_t(a) - \theta_t(b))  \right\|_{L_2(P_{X|H_{t-1}})} \\ 
    &\leq \left(K \sqrt{\frac{B^2}{\epsilon}}\exp(2S)\right)^2 \frac{B}{\sqrt\epsilon} \sum_{a \in [K]} \left\| \frac{\exp(\theta_\infty(a))}{\exp(\theta_\infty(b))} - \frac{\exp(\theta_t(a))}{\exp(\theta_t(b))}  \right\|_{L_2(P_{X|H_{t-1}})} \label{eq:l3} \\ 
    &\quad + \left(K \sqrt{\frac{B^2}{\epsilon}}\exp(2S)\right)^2 \exp(2S) \sum_{a \in [K]} \left\| \left(\sqrt{\frac{V_\infty(x,a)}{V_\infty(x,b)}}  - \sqrt{\frac{V_t(x,a)}{V_t(x,b)}}\right) \right\|_{L_2(P_{X|H_{t-1}})},
\end{align}
where line \eqref{eq:l1} follows from the bounds on $\pi(x, b)$ (shown in Equation \eqref{eq:l0}), line \eqref{eq:l2} follows from adding and subtracting terms $\sum_{a \in [K]} \sqrt\frac{V_\infty(x,a)}{V_\infty(x,b)}\exp(\theta_t(a) - \theta_t(b))$ and subadditivity of norms, and line \eqref{eq:l3} follows from bounds on $V_t(x,a)$, $V_\infty(x,a)$, bounds on $\Theta$, and the subadditivity of norms. Thus, our policy $\pi_t$ converges to $\pi_\infty$ in $L_2(P_{X|H_{t-1}})$ as long as for all $a,b \in [K]$, we satisfy 
\begin{align}
     &\text{(Term }A)\quad \left\| \left(\sqrt{\frac{V_\infty(x,a)}{V_\infty(x,b)}}  - \sqrt{\frac{V_t(x,a)}{V_t(x,b)}}\right) \right\|_{L_2(P_{X|H_{t-1}})} \rightarrow 0, \\
      &\text{(Term }B) \quad \left\|\frac{\exp(\theta_\infty(a))}{\exp(\theta_\infty(b))} - \frac{\exp(\theta_t(a))}{\exp(\theta_t(b))}  \right\|_{L_2(P_{X|H_{t-1}})} \rightarrow 0
\end{align}
To show that Term $A$ converges, note that 
\begin{align}
    &\left\| \left({\frac{V_\infty(x,a)}{V_\infty(x,b)}}  - {\frac{V_t(x,a)}{V_t(x,b)}}\right) \right\|_{L_2(P_{X|H_{t-1}})}\label{eq:to_prove}\\
    &= \left\| \left(\sqrt{\frac{V_\infty(x,a)}{V_\infty(x,b)}}  + \sqrt{\frac{V_t(x,a)}{V_t(x,b)}}\right)\left(\sqrt{\frac{V_\infty(x,a)}{V_\infty(x,b)}}  - \sqrt{\frac{V_t(x,a)}{V_t(x,b)}}\right) \right\|_{L_2(P_{X|H_{t-1}})}\\
    &\geq \frac{2\sqrt{\epsilon}}{B}\left\| \left(\sqrt{\frac{V_\infty(x,a)}{V_\infty(x,b)}}  - \sqrt{\frac{V_t(x,a)}{V_t(x,b)}}\right)\right\|_{L_2(P_{X|H_{t-1}})},
\end{align}
where the inequality above holds due to the bounds $V_t(x,a) \in [\epsilon, B^2]$ for all $t \in \NN$, $x \in \Xcal$, $a \in [K]$. Thus, to show that Term $A$ converges, we show that the expression in line \eqref{eq:to_prove} converges almost surely to zero. To prove this, note that  
\begin{align}
    &\left\|\left({\frac{V_\infty(x,a)}{V_\infty(x,b)}}  - {\frac{V_t(x,a)}{V_t(x,b)}}\right) \right\|_{L_2(P_{X|H_{t-1}})} = \\
    &\left\|{\frac{V_\infty(x,a) - V_t(x,a)}{V_\infty(x,b)}}  + V_t(x,a)\left({\frac{1}{V_\infty(x,b)}-\frac{1}{V_t(x,b)}}\right) \right\|_{L_2(P_{X|H_{t-1}})}=\\ 
    &  \left\|{\frac{V_\infty(x,a) - V_t(x,a)}{V_\infty(x,b)}}  + V_t(x,a)\left({\frac{V_t(x,b) - V_\infty(x,b)}{V_\infty(x,b)V_t(x,b)}}\right) \right\|_{L_2(P_{X|H_{t-1}})}\leq \\
    & \frac{1}{\epsilon} \|V_\infty(x,a) - V_t(x,a) \|_{L_2(P_{X|H_{t-1}})} + \frac{B^2}{\epsilon^2} \| V_t(x,b) - V_\infty(x,b)\|_{L_2(P_{X|H_{t-1}})},
\end{align}
which converges to zero under the assumption that $\|V_t(x,a) - V_\infty(x,a) \|_{L_2(P_{X|H_{t-1}})} \rightarrow 0$ almost surely. 

We will show that $L_2$ convergence of $\bm\theta$, i.e. $\|\bm\theta_t - \bm\theta_{\infty} \|_2$, is sufficient for control of Term $B$. First, note that Term $B$ can be expressed as $\|\exp(\theta_\infty(a) - \theta_\infty(b))  - \exp(\theta_t(a) - \theta_t(b))\|_{L_2(P_{X|H_{t-1}})}$, and by the mean value theorem and bounds on $\Theta$, there exists a $c \in [-2S, 2S]$ such that 
\begin{align}
    \exp(\theta_\infty(a) - \theta_\infty(b))  - \exp(\theta_t(a) - \theta_t(b)) &= \exp(c) \left(\theta_\infty(a) - \theta_\infty(b) - (\theta_t(a) - \theta_t(b)) \right).
\end{align}
By taking absolute values and replacing $c$ with its upper bound $2S$, we obtain
\begin{equation}
    \left|\exp(\theta_\infty(a) - \theta_\infty(b))  - \exp(\theta_t(a) - \theta_t(b))\right| \leq \exp(2S) \left|\theta_\infty(a) - \theta_\infty(b) - (\theta_t(a) - \theta_t(b)) \right|.
\end{equation}
Now, by squaring both sides, integrating with respect to $P_{X|H_{t-1}}$, and taking square roots, we obtain 
\begin{equation}
    \left\|\frac{\exp(\theta_\infty(a))}{\exp(\theta_\infty(b))} - \frac{\exp(\theta_t(a))}{\exp(\theta_t(b))}  \right\|_{L_2(P_{X|H_{t-1}})} \leq \exp(2S) \left(\left|\theta_\infty(a) - \theta_t(a)| + |\theta_\infty(b)  - \theta_t(b)) \right|\right). 
\end{equation}
We now show that $\|\bm{\theta}_t - \bm{\theta}_\infty \|_2 \rightarrow 0$ almost surely implies the convergence of term $B$. Note that if $\|\bm{\theta}_t - \bm{\theta}_\infty \|_2 \rightarrow 0$ almost surely, then by the Cauchy Schwartz inequality, for all $a \in [K]$,
\begin{equation}
    \|{\theta}_t(a) - \theta_\infty(a)\|_2 = \|\bm{e}_a\left(\bm{\theta}_t - \bm\theta_\infty\right)\|_2 \leq \| \bm{e}_a\|_2\|\left(\bm{\theta}_t - \bm\theta_\infty\right) \|_2 \leq \|\left(\bm{\theta}_t - \bm\theta_\infty\right) \|_2.
\end{equation}
Thus, convergence of $\|\bm{\theta}_t - \bm{\theta}_\infty \|_2$ ensures that Term $B$ vanishes almost surely, and $\|\pi_t - \pi_\infty \|_{P_{X|H_{t-1}}} \rightarrow 0$ almost surely as desired. To prove the convergence of $\bm{\theta}_t$, we control two error terms shown below:
\begin{equation}\label{eq:main_ineq}
    \| \bm{\theta}_t-\bm{\theta}_{\infty}\|_2 \leq \| \bm{\theta}_t - \bm{\theta}_{t,*}\|_2+  \|\bm{\theta}_{t,*} - \bm{\theta}_{\infty} \|_2,
\end{equation}
where $\bm{\theta}_{t,*}$ denotes the minimizing solution of the empirical objective $G_{t}$, as defined in Equation \eqref{eqs:empirical_counterpart}. Our proof proceeds as follows:
\begin{enumerate}
    \item First, we show that at each timestep $t$, projected subgradient descent converges to the minimizing solution of $G_t$ as the number of iterations $N$ diverges towards infinity. This controls the error term $\|\bm{\theta}_t - \bm{\theta}_{t,*}\|_2$ under the assumption that $N$, the number of iterations, diverges to infinity. 
    \item Second, we show that our objective function $G_t$ converges to the objective function $G_\infty$ almost surely. Paired with Lemma \ref{lem:conv_argmax}, we obtain control over the error term $\|\bm{\theta}_{t,*} - \bm{\theta}_{\infty} \|_2$.
    \item Under our additional conditions stated, we show that the limiting policy $\pi_\infty$ is equivalent to the true optimal policy $\pi_* = \argmin_{\pi \in \Pi} G(\pi)$. 
\end{enumerate}

To prove the convergence of $\|\bm\theta_t - \bm{\theta}_{t,*} \|_2$, we leverage the following standard result for the convergence of projected subgradient descent \citep{boyd2014subgradient-ee364b}, provided in Lemma \ref{lem:convergence_subgrad_descent}. 

\begin{lemma}[Convergence of Projected Subgradient Descent \citep{boyd2014subgradient-ee364b}]\label{lem:convergence_subgrad_descent}
    Let $f$ be the convex objective function we wish to minimize, under the constraint that $\bm{x} \in \Theta$. Assume that $\Theta$ is closed and convex, $f$ is convex, and there exists a strictly feasible point $\bm{x} \in \Theta$. Let $\bm{x}^*$ denote a minimizer of the objective function $f$. Let $g^{(k)}$ denote the subgradient and $\bm{x}^{(k)}$ denote the parameter at the $k$-th iteration of projected subgradient descent. Assume that the norm of the subgradients are bounded, i.e. $\exists G < \infty$ such that $\|g^{(k)} \|_2 \leq G$ for all $k$. Furthermore, assume that there exists an $R< \infty$ such that $\|\bm{x}^{(0)} - \bm{x}^* \| \leq G$. Let $f_{\text{best}}^{(k)} \coloneqq \min_{i \in [k]}f(\bm{x}^{(i)})$ denote the value of the best iterate among the first $k$ iterates. Then,
    \begin{equation}
        f_{\text{best}}^{(k)} - f(\bm{x}^*) \leq \frac{R^2 + \sum_{i=1}^k \gamma_k^2}{(2/G)\sum_{i=1}^k \gamma_i},
    \end{equation}
    for projected subgradient descent with step size $\alpha_k = \gamma_k / \| g^{(k)} \|_2$ at iteration $k$.
\end{lemma}

\paragraph{Step 1: Convergence for Estimated Objectives} We use Lemma \ref{lem:convergence_subgrad_descent} to show that for each $t$, $f_{\text{best}}^{(k)}$ converges to $f(\bm{x}^*)$ as $N \rightarrow \infty$. To apply Lemma \ref{lem:convergence_subgrad_descent}, we first show that (i) $\|\bm{\theta}_0 - \bm{\theta}_{t,*} \|_2$ is bounded, where $\bm{\theta}_{t,*} = \argmax_{\bm{\theta} \in \Theta} G_t(\bm{\theta})$ and $\Theta$ is as defined in Theorem \ref{thm:projected_subgradient_descent_convergence}, and (ii) the chosen subgradient $\bm{d}_n$ is bounded. The boundedness of $\| \bm\theta_0 - \bm{\theta}_{t, *}\|_2$ follows from the bounds $[-S,S]$, yielding
\begin{equation}
    \| \bm\theta_0 - \bm\theta_{t, *}\|_2 \leq \sqrt{(K-1)4S^2} = 2S\sqrt{K-1}.
\end{equation}
To prove the boundedness of $\bm{d}_n$, we first provide bounds on the squared SNR ratio for each $a \not\in \argmax_{b \in [K]} \hat\mu_{t-1}(b)$. By definition of $\bm{w}_{\bm{\theta}}^a \in \argmin_{\bm{w} \in \Delta(a), \bm{w}^\top \hat{\bm{\mu}}_{t-1}\geq 0} f_t(\bm{\theta}, \bm{w})$, denoting $a_t^* \in \argmax_{b \in [K]} \hat\mu_{t-1}(b)$, we obtain
\begin{equation}\label{eq:one_step_one}
    \frac{\sum_{b \in [K]} \frac{w_{\bm{\theta}}^{a}(b)^2}{t} \sum_{i=1}^t\left[V_i(X_i,b)\sum_{a \in [K]}\frac{ \sqrt{V_i(X_i,a) } }{\sqrt{V_i(X_i,b) } }\exp\left(\theta(a) - \theta(b)\right)\right] + l_t(\bm{w})}{\left(\sum_{b \in [K]} w_{\bm\theta}^a(b)\hat\mu_{t-1}(b) \right)^2} \leq \frac{K^2B^2\exp(2S)+ 4K^2B^2}{\left(\hat\mu_{t-1}(a) - \hat\mu_{t-1}(a_t^*)\right)^2}.
\end{equation}
We now provide a lower bound on the numerator of the left-hand side. By the lower bound $\epsilon$ on $V_i(X_i, a)$ terms and $\bm{\theta} \in \Theta$, 
\begin{equation}\label{eq:two_step_one}
    \sum_{b \in [K]} \frac{w_{\bm{\theta}}^{a}(b)^2}{t} \sum_{i=1}^t\left[V_i(X_i,b)\sum_{a \in [K]}\frac{ \sqrt{V_i(X_i,a) } }{\sqrt{V_i(X_i,b) } }\exp\left(\theta(a) - \theta(b)\right)\right] + l_t(\bm{w}) \geq K\epsilon\exp(-2S).
\end{equation}
Putting the results of Equations \eqref{eq:one_step_one} and \eqref{eq:two_step_one}, we obtain
\begin{equation}\label{eq:bounds_on_mean}
    \left(\sum_{b \in [K]} w_{\bm\theta}^a(b)\hat\mu_{t-1}(b) \right)^2 \geq \frac{K\epsilon\exp(-2S)\left(\hat\mu_{t-1}(a) - \hat\mu_{t-1}(a_t^*)\right)^2}{K^2B^2\exp(2S)+ 4K^2B^2}.
\end{equation}
Given the lower bounds on $\left(\sum_{b \in [K]} w_{\bm\theta}^a(b)\hat\mu_{t-1}(b) \right)^2$, we now turn to bounding the gradient of $F_{a,t}(\bm{\theta})$. By Lemma \ref{lem:subgradient_lemma} and the bounds established in Equation \eqref{eq:bounds_on_mean}, we obtain the following bound $\forall a \not\in \argmax_{b \in [K]} \hat\mu_{t-1}(b)$: 
\begin{align}
    \|\nabla_{\bm\theta} F_{a,t}(\bm{\theta})\|_2 &= \sqrt{\sum_{c \in [K]}\left( \sum_{b \in [K]}\frac{\frac{1}{t}\sum_{i=1}^t\sqrt{V_i(X_i,b)V_i(X_i,c)} }{\left(\sum_{b \in [K]} w_{\bm{\theta}}^a(b) \hat\mu_{t-1}(b)\right)^2} \left(  w_{\bm\theta}^a(b)^2\exp(\theta(c) - \theta(b)) - w_{\bm\theta}^a(c)^2 \exp(\theta(b)-\theta(c))\right)\right)^2  }\\
    &\leq \sqrt{\sum_{c \in [K]} \left(K\frac{B^2\left(K^2B^2\exp(2S)+ 4K^2B^2\right)}{K\epsilon\exp(-2S)\left(\hat\mu_{t-1}(a) - \hat\mu_{t-1}(a_t^*)\right)^2}\left(2\exp(2S) \right) \right)^2} \\
    & \leq \sqrt{K\left(\frac{B^4\left(K^2\exp(2S)+ 4K^2\right)}{\epsilon\exp(-2S)\left(\hat\mu_{t-1}(a) - \hat\mu_{t-1}(a_t^*)\right)^2}\left(2\exp(2S) \right) \right)^2}. \label{eq:}
\end{align}
By the triangle inequality of the $L_2$ norm, we obtain the following bound for the subgradient $\bm{d}_n$:
\begin{align}
    \|\bm{d}_n\|_2 &= \left\|\frac{1}{|\Acal_n(\bm\theta)|}\sum_{a \in \Acal_n(\bm\theta)} \nabla_\theta f_{t}(\bm\theta_n, \bm{w}_n^a)\right\|_2 \\ 
    &\leq \frac{1}{|\Acal_n(\bm\theta)|} \left( |\Acal_n(\bm\theta)| \sqrt{K\left(\frac{B^4\left(K^2\exp(2S)+ 4K^2\right)}{\epsilon\exp(-2S)\left(\hat\mu_{t-1}(a) - \hat\mu_{t-1}(a_t^*)\right)^2}\left(2\exp(2S) \right) \right)^2} \right) \\
    & = \sqrt{K\left(\frac{B^4\left(K^2\exp(2S)+ 4K^2\right)}{\epsilon\exp(-2S)\left(\hat\mu_{t-1}(a) - \hat\mu_{t-1}(a_t^*)\right)^2}\left(2\exp(2S) \right) \right)^2}.
\end{align}
The bounds on the subgradients $\bm{d}_n$, bounds on $\|\bm{\theta}_0 - \bm\theta_{t, *} \|_2$, and the strict convexity of objective function $G_t(\bm{\theta})$ (shown in the proof of Lemma \ref{lem:subgradient_lemma}) ensure that Lemma \ref{lem:convergence_subgrad_descent} holds. Thus, we obtain
\begin{equation}
    G_t(\bm\theta_t) - G_t(\bm\theta_{t, *}) \leq \frac{{K\left(\frac{B^4\left(K^2\exp(2S)+ 4K^2\right)}{\epsilon\exp(-2S)\left(\hat\mu_{t-1}(a) - \hat\mu_{t-1}(a_t^*)\right)^2}\left(2\exp(2S) \right) \right)^2}+\sum_{n=1}^{N} \frac{1}{n^2}}{\left(\frac{2}{2S\sqrt{K-1}}\right)\sum_{n=1}^{N}\frac{1}{n}},
\end{equation}
where $\bm{\theta}_t$ is as defined in Algorithm \ref{alg:proj-subgradient}. Note that as $N$, the number of iterations, approaches infinity, the suboptimality of our solution vanishes, i.e.
\begin{equation}\label{eq:convergence_objective_value}
    \lim_{N \rightarrow \infty} G_t(\bm{\theta}_{t, N}) - G_t(\bm{\theta}_{t,*}) =0
\end{equation}
where $\bm{\theta}_{t,N}$ be the solution returned by Algorithm \ref{alg:proj-subgradient} for objective function $G_t$ after $N$ iterations. This holds due to $\hat\mu_{t-1}(a) < \hat\mu_{t-1}(a_t^*)$, $\lim_{N \rightarrow \infty} \sum_{n=1}^N\frac{1}{n^2} < \infty$, and $\lim_{N \rightarrow \infty}\sum_{n=1}^N \frac{1}{n} = \infty$.

To show that Equation \ref{eq:convergence_objective_value} implies the convergence of our iterates $\bm{\theta}_t \rightarrow \bm{\theta}_{t,*}$, note that (i) the objective function $G_t$ is strictly convex and (ii) the domain $\Theta$ is a compact set. We prove this result via contradiction. Suppose that $\lim_{N \rightarrow \infty}\bm{\theta}_{t,N} \neq \bm{\theta}_{t,*}$, i.e. there exists $\delta >0$ such that for every $N_0$, there exists $N \geq N_0$ with $\|\bm{\theta}_{t,N} - \bm{\theta}_{t, *}\|_2 > \delta$. By the strict convexity of $G_t(\bm{\theta})$, it follows that 
\begin{equation}
    m_\delta = \min_{\|\bm{\theta} - \bm{\theta}_{t,*} \|_2 \geq \delta} G_{t}(\bm\theta) - G_t(\bm{\theta}_{t,*}) > 0.
\end{equation}
Under our contradiction and the compactness of $\Theta$, for some $\delta>0$, there exists an infinite subsequence of iteration indices $\{N_i\}_{i \in \NN}$ where $\|\bm\theta_{t, N_i} - \bm\theta_{t, *} \|_2 \geq \delta$. However, then for all $i \in \NN$, 
\begin{equation}
 G_t(\bm\theta_{t, N_i} ) - G_t(\bm\theta_{t,*}) \geq m_\delta > 0,
\end{equation}
which contradicts the result obtained in Equation \eqref{eq:convergence_objective_value}. Thus, for all $\delta > 0$, there exists an $N_0$ large enough that for all $N \geq N_0$, $\|\bm\theta_{t,N} - \bm{\theta}_{t,*} \|_2 \leq \delta$, i.e.
\begin{equation}
    \lim_{N \rightarrow \infty}\|\bm{\theta}_{t,N} - \bm{\theta}_{t,*}\|_2 = 0.
\end{equation}
Thus, for any fixed realization of $G_t$, the solution $\bm{\theta}_{t, N}$ converges (w.r.t $N$) to the minimizing solution $\bm{\theta}_{t,*}$.

\paragraph{Step 2: Convergence of Limiting Objective}

To establish the convergence of $\bm{\theta}_{t}$ to $\bm{\theta}_\infty$, we first prove the convergence of $G_t$ to $G_\infty$ under our assumptions on $g_t$ and $V_t$. Recall that  $G_t$ is defined as
\begin{align}
    G_t(\bm{\theta}) &=  \max_{a: \hat\mu_{t-1}(a) < \max_{b \in [K]} \hat\mu_{t-1}(b)} F_{a,t}(\bm\theta), \\
    F_{a,t}(\bm\theta) &= \min_{\bm{w} \in \Delta(a), \bm{w}^\top \hat{\bm\mu}_{t-1} \geq 0} f_{t}(\bm{\theta}, \bm{w}), \\
    f_{t}(\bm\theta, \bm{w}) &= \frac{ \sum_{b \in [K]} \frac{w(b)^2}{t} \sum_{i=1}^t\left[V_i(X_i,b)\sum_{a \in [K]}\frac{ \sqrt{V_i(X_i,a) } }{\sqrt{V_i(X_i,b) } }\exp\left(\theta(a) - \theta(b)\right)\right] + l_t(\bm{w})}{\left(\sum_{b \in [K]} w(b)  \hat\mu_{t-1}(b) \right)^2} \\
    l_t(\bm{w}) &= \frac{1}{t}\sum_{i=1}^t\left[ \left(\sum_{b \in [K]} w(b) (g_i(X_i,b)- \hat\mu_{t-1}(b)) \right)^2 \right],
\end{align}
We will show that $\sup_{\bm{\theta}\in \Theta}|G_t(\bm\theta)-G_\infty(\bm{\theta})| \rightarrow 0$ almost surely under our assumptions, allowing for the use of Lemma \ref{lem:conv_argmax}. We can upper bound the difference between $G_t$ and $G$ with the following two terms:
\begin{align}
    \sup_{\bm\theta \in \Theta}|G_t(\bm\theta) - G_\infty(\bm{\theta})|&=  \sup_{\bm\theta \in \Theta}\left| \max_{a: \hat\mu_{t-1}(a) < \max_{b \in [K]} \hat\mu_{t-1}(b)} F_{a,t}(\bm\theta) - \max_{a \neq a^*} F_{a, \infty}(\bm{\theta}) \right|\\
    &\leq 
    \underbrace{\sup_{\bm\theta \in \Theta}\left| \max_{a: \hat\mu_{t-1}(a) < \max_{b \in [K]} \hat\mu_{t-1}(b)} F_{a, t}(\bm{\theta})-\max_{a \neq a^*} F_{a, t}(\bm{\theta}) \right|}_{\text{Term $(i)$}}\\
    &\quad + \underbrace{\sup_{\bm\theta \in \Theta}\left| \max_{a \neq a^*} F_{a,t}(\bm\theta) - \max_{a^*} F_{a, \infty}(\bm{\theta}) \right| }_{\text{Term $(ii)$}}\label{eq:term_2}
\end{align}
We first show that term $(i)$ converges to zero almost surely. Note that term $(i)$ differs only in the set of indices $a$ where the maximum is selected, and therefore for any $\bm{\theta} \in \Theta$, 

\begin{align}
    (i) &= \sup_{\bm\theta \in \Theta}\left| \max_{a: \hat\mu_{t-1}(a) < \max_{b \in [K]} \hat\mu_{t-1}(b)} F_{a, t}(\bm{\theta})-\max_{a \neq a^*} F_{a, t}(\bm{\theta}) \right| \\
    &\leq \sum_{a \in [K]}\sup_{\bm\theta \in \Theta}\biggr| \biggr(\mathbf{1}\left[F_{a,t }(\bm\theta)= \max_{b: \hat\mu_{t-1}(a)<\max_{b \in [K]}\hat\mu_{t-1}(b)} F_{b, t}(\bm{\theta}), a \not\in \argmax_{b \in [K]}\hat\mu_{t-1}(b) \right]-\\
&\quad\quad\quad\quad\quad\quad\quad  \mathbf{1}\left[F_{a, t}(\bm\theta) = \max_{b \neq a^*} F_{a, t}(\bm\theta), a \neq a^*\right]\biggr) F_{a, t}(\bm\theta) \biggr|.
\end{align}
Thus, term $(i)$ converges to zero almost surely as long the indicator functions in the summation above are equal for each $\bm\theta \in \Theta$. Note that for our indicators to align for all $\bm\theta \in \Theta$, we require the set $\{a \in [K]: \hat\mu_{t-1}(a) < \max_{b \in [K]} \hat\mu_{t-1}(b)\} \rightarrow [K]\setminus a^*$ and $\argmax_{b \in [K]}\hat\mu_{t-1}(b) \rightarrow \{a^*\}$ almost surely. By the almost-sure convergence of $\hat\mu_{t-1}(a) \rightarrow \mu(a)$ for all $a \in [K]$ (shown in the proof of Theorem \ref{thm:error_control}) and Assumption \ref{assump:unique_optimal_arm},  
\begin{align}
    \{a \in [K]: \hat\mu_{t-1}(a) < \max_{b \in [K]} \hat\mu_{t-1}(b)\} &\rightarrow [K]\setminus a^* \\
    \argmax_{b \in [K]}\hat\mu_{t-1}(b) &\rightarrow \{a^*\}
\end{align}
almost surely due to the unique optimal arm $a^* = \argmax_{a \in [K]} \mu(a)$, ensuring that term $(i)$ vanishes almost surely. To show that term $(ii)$ vanishes, we first use an upper bound on $F_{a,t}(\bm\theta)$ for all $a \neq a^*$:
\begin{align}
    F_{a,t}(\bm\theta) \leq  |F_{a, t}(\bm{\theta}) - F_{a,\infty}(\bm{\theta})| + F_{a, \infty}(\bm\theta).
\end{align}
Applying the maximum over $a \neq a^*$ on both sides of the inequality above, we obtain
\begin{align}
    \max_{a \neq a^*} F_{a,t}(\bm\theta) &\leq \max_{a \neq a^*}|F_{a, t}(\bm{\theta}) - F_{a,\infty}(\bm{\theta})| + \max_{a \neq a^*}F_{a, \infty}(\bm\theta)\\
    &\leq \sum_{a \neq a^*}|F_{a, t}(\bm{\theta}) - F_{a,\infty}(\bm{\theta})| + \max_{a \neq a^*}F_{a, \infty}(\bm\theta),
\end{align}
which directly implies the following upper bound for term $(ii)$:
\begin{equation}
    (ii) \leq \sup_{\bm\theta \in \Theta}\left(\sum_{a \neq a^*}|F_{a,t}(\bm{\theta}) - F_{a,\infty}(\bm\theta)|\right) \leq \sum_{a \neq a^*}\sup_{\bm\theta \in \Theta}|F_{a,t}(\bm{\theta}) - F_{a,\infty}(\bm\theta)|.
\end{equation}
Thus, to show term $(ii)$ converges appropriately, we show that $\sup_{\bm\theta \in \Theta}|F_{a,t}(\bm{\theta}) - F_{a, \infty}(\bm{\theta})|$ converges to zero almost surely as $t \rightarrow \infty$ for all $a \neq a^*$.  To proceed, let the weights $\bm{w}_{a,t}^{\bm\theta}, \bm{w}_{a, \infty}^{\bm\theta}$ be defined as
$$ \bm{w}_{a,t}^{\bm\theta} = \argmax_{\bm{w}\in \Delta(a)} f^{-1/2}_t(\bm{\theta}, \bm{w}), \quad \bm{w}_{a, \infty}^{\bm\theta} = \argmax_{\bm{w} \in \Delta(a)} f_{\infty}^{-1/2}(\bm{\theta}, \bm{w}),$$
i.e. the choice of weights that maximize the empirical SNR ratio $f_t^{-1/2}$ with estimated variance terms $V_t$. This does not affect our analysis due to the fact that $\argmax_{\bm{w}\in \Delta(a)} f_t^{-1/2}(\bm{\theta}, \bm{w}) = \argmin_{\bm{w}\in \Delta(a), \bm{w}^\top \hat{\bm\mu}_{t-1}\geq 0} f_t(\bm\theta, \bm{w})$ for all $a \not\in \argmax_{b \in [K]}\hat\mu_{t-1}(b)$, and $\bm{w}_{a,t}^{\bm\theta}$ does not appear in the objective function $G_t(\bm{\theta})$ for all $a \in \argmax_{b \in [K]}\hat\mu_{t-1}(b)$. Thus, for $a \neq a^*$, we can rewrite our uniform convergence condition as the following holding almost surely as $t \rightarrow \infty$:
\begin{align}
    \sup_{\bm\theta \in \Theta}|F_{a,t}(\bm{\theta}) - F_{a,\infty}(\bm\theta)| &= \sup_{\bm\theta \in \Theta}\left| f_t(\bm{\theta}, \bm{w}_{a,t}^{\bm\theta}) -  f_\infty(\bm{\theta}, \bm{w}_{a,\infty}^{\bm\theta}) \right| \rightarrow 0. \label{eq:to_prove_2} 
\end{align}
To show uniform convergence over $\Theta$, we first show that for sufficiently large $t$, $\bm{w}_{a,t}^{\bm\theta}$ and $\bm{w}_{a,\infty}^{\bm\theta}$ lie in the set $\Wcal$ that ensures the denominators of $f_t, f_\infty$ are strictly larger than zero. 
\begin{lemma}[Almost Sure Safe Set]\label{lem:safe_set}
    Under the assumptions of Theorem \ref{thm:projected_subgradient_descent_convergence}, for all $\bm\theta \in \Theta$ and $a \neq a^*$, there exists a set $\Wcal = \{\bm{w} \in \Delta(a): \bm{w}^\top \bm{\mu} \geq \sqrt\frac{K\epsilon\exp(-2S)}{K^2B^2(\exp(2S)+4)}\frac{(\mu(a^*)-\mu(a))}{4} \}$ such that $\bm{w}_{a,\infty}^{\bm\theta} \in \Wcal$ and $\bm{w}_{a,t}^{\bm\theta} \in \Wcal$ almost surely as $t \rightarrow \infty$.  
\end{lemma}

\begin{proof}[Proof of Lemma \ref{lem:safe_set}]
    To begin our proof, we first define the denominator of $f_\infty$ as 
    \begin{align}
        Q_\infty(\bm\theta, \bm{w}) &\coloneqq {\EE_{P_X}\left[ \sum_{b \in [K]} \left(w^2(b) V_\infty(X,b) \sum_{a \in [K]}\sqrt{\frac{V_\infty(X,a)}{V_\infty(X,b)}}\exp\left(\theta(a) - \theta(b)\right)\right) \right] + l_\infty(\bm{w})}, \label{eq:q_inf}\\
        Q_t(\bm\theta, \bm{w}) &\coloneqq \sum_{b \in [K]} \frac{w(b)^2}{t} \sum_{i=1}^t\left[V_i(X_i,b)\sum_{a \in [K]}\frac{ \sqrt{V_i(X_i,a) } }{\sqrt{V_i(X_i,b) } }\exp\left(\theta(a) - \theta(b)\right)\right] + l_t(\bm{w})\label{eq:q_t}.
    \end{align}
    By our assumed bounds, we obtain the following bounds on $Q_{t}, Q_{\infty}$ for all $\bm{\theta} \in \Theta$, $\bm{w} \in \Delta(a)$, and $t \in \NN$:
    \begin{align}
        &K\epsilon\exp(-2S)\leq Q_\infty(\bm\theta, \bm{w}) \leq K^2B^2(\exp(2S) + 4),\\ 
        &K\epsilon\exp(-2S)\leq Q_t(\bm\theta, \bm{w}) \leq K^2B^2(\exp(2S) + 4).\label{eq:bounds_empirical}
    \end{align}
    Let $\tilde{\bm{w}}_a = \bm{e}_{a^*}-\bm{e}_{a}$, where $\bm{e}_i \in \RR^K$ denotes the $i$-th unit vector. By definition of $\bm{w}_{a, \infty}^{\bm\theta}$, we obtain 
    \begin{equation}
        f_\infty(\bm{\theta}, \bm{w}_{a, \infty}) \leq f_\infty(\bm{\theta}, \tilde{\bm{w}}_a) \implies \frac{Q_\infty(\bm{\theta}, \bm{w}_{a, \infty}^{\bm\theta})}{\left(\sum_{b \in [K]}w_{a, \infty}(b)\mu(b)\right)^2}  \leq \frac{Q_\infty(\bm{\theta}, \tilde{\bm{w}}_{a})}{\left(\mu(a^*) - \mu(a)\right)^2}, 
    \end{equation}
    and by our uniform bounds on $Q_\infty$ above, we obtain that 
    \begin{equation}
        \sqrt\frac{K\epsilon\exp(-2S)}{K^2B^2(\exp(2S)+4)}(\mu(a^*)-\mu(a)) \leq \left(\sum_{b \in [K]}w_{a, \infty}^{\bm\theta}(b)\mu(b)\right),
    \end{equation}
    demonstrating that $\bm{w}^\top \bm{\mu}$ lies in $\Wcal(\bm{\theta})$. To show $w_{a,t}$ lies in $\Wcal(\bm{\theta})$ almost surely, we show that for all sample paths $\omega \in \Omega$, where $P(\Omega) = 1$, there exists a $t(\omega)$ such that $\bm{w}_{a,t}^{\bm\theta} \in \Wcal$ for all $t \geq t(\omega)$. We denote random variables $X$ corresponding to the sample path $\omega$ as $X(\omega)$. By the almost-sure convergence of $\hat\mu_{t-1}(a)$ to $\mu(a)$ almost surely for all $a \in [K]$ (shown in the proof of Theorem \ref{thm:error_control}), there exists a $t(\omega)$ such that $\forall t \geq t(\omega)$,
    \begin{align}
        &\hat\mu_{t-1}(a^*)(\omega) -\hat\mu_{t-1}(a)(\omega)  \geq \frac{|\mu(a^*) -\mu(a)|}{2}, \\
       \forall a \in [K], \  &\left|\hat\mu_{t-1}(a)(\omega) - \mu(a)\right| \leq \frac{1}{K}\sqrt\frac{K\epsilon\exp(-2S)}{K^2B^2(\exp(2S)+4)}\frac{(\mu(a^*)-\mu(a))}{4}.
    \end{align}
    By repeating the same argument as above using the (random) objective function $f_t(\bm{\theta}, \bm{w})$ and the bounds provided in Equation \eqref{eq:bounds_empirical}, we obtain that for all $t \geq t(\omega)$, for all $\bm{\theta} \in \Theta$, 
    \begin{equation}\label{eq:bounds}
        \sqrt\frac{K\epsilon\exp(-2S)}{K^2B^2(\exp(2S)+4)}\frac{(\mu(a^*)-\mu(a))}{2} \leq \left( \sum_{b \in [K]} w_{a,t}^{\bm\theta}(b) \hat\mu_{t-1}(b)\right) (\omega),
    \end{equation}
    and by the fact that $\left( \sum_{b \in [K]} w_{a,t}^{\bm\theta}(b) \hat\mu_{t-1}(b)\right) (\omega) \leq \left(\sum_{b\in [K]}w_{a,t}^{\bm\theta}(b) \mu(b)\right) + \sum_{b \in [K]}|\hat\mu_{t-1}(b)(\omega)-\mu(b)|$, we obtain the following bound:
    \begin{equation}
        \sqrt\frac{K\epsilon\exp(-2S)}{K^2B^2(\exp(2S)+4)}\frac{(\mu(a^*)-\mu(a))}{4} \leq \left( \sum_{b \in [K]} w_{a,t}^{\bm\theta}(b) \mu(b)\right).
    \end{equation}
    Thus, we obtain $\bm{w}_{a,t} \in \Wcal$ almost surely. 
\end{proof}

To prove that Equation \eqref{eq:to_prove_2} holds, we also leverage the almost-sure convergence of $Q_t(\bm\theta, \bm{w})$ to $Q_\infty(\bm{\theta}, \bm{w})$ uniformly over $\bm{\theta} \in \Theta$ and $\bm{w} \in \Delta(a)$ for all $a \neq a^*$. We provide this result in Lemma \ref{lem:convergence_q} below.

\begin{lemma}[Uniform Convergence of $Q_t$]\label{lem:convergence_q}
    Let $Q_\infty$ and $Q_t$ be defined as in Equations \eqref{eq:q_inf} and \eqref{eq:q_t} respectively. Under the assumptions of Theorem \ref{thm:projected_subgradient_descent_convergence}, $\sup_{(\bm{\theta}, \bm{w}) \in \Theta \times \Delta(a)}| Q_t(\bm{\theta}, \bm{w}) - Q_\infty(\bm{\theta}, \bm{w})| \rightarrow 0$ almost surely. 
\end{lemma}

\begin{proof}[Proof of Lemma \ref{lem:convergence_q}]
To simplify notation, we define the functions $W_t$ and $W_\infty$ as follows:
\begin{align}
    W_t(\bm{\theta}, \bm{w}) &\coloneqq \sum_{b \in [K]} \frac{w(b)^2}{t} \sum_{i=1}^t\left[V_i(X_i,b)\sum_{a \in [K]}\frac{ \sqrt{V_i(X_i,a) } }{\sqrt{V_i(X_i,b) } }\exp\left(\theta(a) - \theta(b)\right)\right], \\
    W_\infty(\bm\theta, \bm{w}) &\coloneqq {\EE_{P_X}\left[ \sum_{b \in [K]} \left(w^2(b) V_\infty(X,b) \sum_{a \in [K]}\sqrt{\frac{V_\infty(X,a)}{V_\infty(X,b)}}\exp\left(\theta(a) - \theta(b)\right)\right) \right]}.
\end{align}
We now upper bound the difference between $Q_t$ and $Q_\infty$ as follows:
\begin{align}
    \sup_{(\bm{\theta}, \bm{w}) \in \Theta \times \Delta(a)}| Q_t(\bm{\theta}, \bm{w}) - Q_\infty(\bm{\theta}, \bm{w})| \leq & \quad \underbrace{\sup_{(\bm{\theta}, \bm{w}) \in \Theta \times \Delta(a)} \left|W_{t}(\bm\theta, \bm{w}) - W_\infty(\bm\theta, \bm{w}) \right|}_{\text{Term $(i)$}}  \\
    &\quad + \underbrace{\sup_{(\bm{\theta}, \bm{w}) \in \Theta \times \Delta(a)} \left|l_t(\bm{w}) - l_\infty(\bm{w}) \right|}_{\text{Term $(ii)$}}.
\end{align}
We begin by term $(ii)$ vanishes, showing that $l_t$ uniformly converges to $l_\infty$ almost surely. First, note that because $l_t(\bm{w})$ is uniformly Lipschitz on $\Delta(a)$, it suffices to show pointwise convergence for a dense subset of $\bm{w} \in \Delta(a)$ (Chapter 1, \cite{vandervaart1996weak}). We now show that we obtain pointwise convergence for $\bm\theta \in \Theta$. Note that $l_t$ can be rewritten as 
\begin{align}
    l_t(\bm{w}) &= \frac{1}{t}\sum_{i=1}^t \left(\sum_{b \in [K]} w(b)\left( g_i(X_i,b) - \mu(b)\right)+\sum_{b \in [K]}w(b)(\mu(b) - \hat\mu_{t-1}(b))\right)^2\\
    &= \frac{1}{t}\sum_{i=1}^t \left(\sum_{b \in [K]} w(b) (g_i(X_i,b) - \mu(b)) \right)^2\\
    &\quad + 2\left(\sum_{b \in [K]}w(b) (\mu(b) - \hat\mu_{t-1}(b)) \right)\left(\frac{1}{t}\sum_{i=1}^t \left(\sum_{b \in [K]} w(b) (g_i(X_i,b) - \mu(b)) \right)\right)  \label{eq:term_1} \\ 
    &\quad +  \left(\sum_{b \in [K]} w(b)(\mu(b) - \hat\mu_{t-1}(b)) \right)^2. \label{eq:term_2}
\end{align}
The terms on Equations \eqref{eq:term_1} and \eqref{eq:term_2} vanish almost surely. By the bounds on $\bm{w} \in \Delta(a)$ and $\hat\mu_{t-1}(a) \rightarrow \mu(a)$ almost surely for all $a \in [K]$, it follows that:
\begin{equation}
    \lim_{t \rightarrow \infty}\left|\sum_{b \in [K]} w(b)(\mu(b) - \hat\mu_{t-1}(b))\right| \leq \lim_{t \rightarrow \infty}\sum_{b \in [K]} |\mu(b) - \hat\mu_{t-1}(b)| = 0.
\end{equation}
By the bounds $\left|\left(\sum_{b \in [K]} w(b)(\mu(b) - \hat\mu_{t-1}(b)) \right)\right| \leq 2KB$ and $\left|\frac{1}{t}\sum_{i=1}^t \left(\sum_{b \in [K]} w(b) (g_i(X_i,b) - \mu(b)) \right)\right| \leq 2KB$ due to $|g_i(X_i, b)| \leq B$ and $|\mu(b)| \leq B $. It then follows that
\begin{align}
    \lim_{t \rightarrow \infty}2\left(\sum_{b \in [K]}w(b) (\mu(b) - \hat\mu_{t-1}(b)) \right)\left(\frac{1}{t}\sum_{i=1}^t \left(\sum_{b \in [K]} w(b) (g_i(X_i,b) - \mu(b)) \right)\right) &\rightarrow 0,\\
    \lim_{t \rightarrow \infty}\left(\sum_{b \in [K]} w(b)(\mu(b) - \hat\mu_{t-1}(b)) \right)^2 &\rightarrow 0
\end{align}
almost surely. Thus, the limit of $l_t(\bm{w})$ is solely dominated by the first term $\frac{1}{t}\sum_{i=1}^t\left(\sum_{b \in [K]} w(b)\left( g_i(X_i,b) - \mu(b)\right)\right)^2$. We now show that this term converges to $l_\infty$ uniformly over $\bm{w} \in \Delta(a)$ almost surely. We first arrange this first term as terms $A_i(\bm{w})$ and $B_i(\bm{w})$ as follows:
\begin{align}
    &\frac{1}{t}\sum_{i=1}^t\left(\sum_{b \in [K]} w(b)\left( g_i(X_i,b) - \mu(b)\right)\right)^2 = \\
    &\frac{1}{t}\sum_{i=1}^t \left( \underbrace{\sum_{b \in [K]} w(b)(g_i(X_i,b) - g_\infty(X_i, b))}_{\coloneqq A_i(\bm{w})} + \underbrace{\sum_{b \in [K]} w(b)(g_\infty(X_i, b) - \mu(b))}_{\coloneqq B_i(\bm{w})} \right)^2=\\
    &\frac{1}{t}\sum_{i=1}^t A_i^2(\bm{w}) + 2A_i(\bm{w})B_i(\bm{w}) + B_i^2(\bm{w})
\end{align}
We deal with the term $B_i(\bm{w})$. Because $g_\infty$ is fixed, $\mu(b)$ is fixed, and $X_i \sim P_X$ i.i.d., by direct application of the strong law of large numbers, 
\begin{equation}
    \lim_{t \rightarrow \infty}\frac{1}{t}\sum_{i=1}^t B_i^2(\bm{w}) = \lim_{t \rightarrow \infty}\frac{1}{t}\sum_{i=1}^t \left( \sum_{b \in [K]} w(b)(g_\infty(X_i, b) - \mu(b))\right)^2  =  \EE_{P_X}\left[ \left(\sum_{b\in [K]} w(b) (g_\infty(X,b) - \mu(b))\right)^2  \right]
\end{equation}
We now show that the terms with $A_i(\bm{w})$ vanish almost surely. 
\begin{align}
    \frac{1}{t}\sum_{i=1}^t A_i(\bm{w}) &= \frac{1}{t}\sum_{i=1}^t \left(\sum_{b \in [K]} w(b) \left( g_i(X_i) - g_\infty(X_i, b) \right) \right) \\
    &= \underbrace{\frac{1}{t}\sum_{i=1}^t \left(\sum_{b \in [K]} w(b)(g_i(X_i,b) - \EE_{P_X}[g_\infty(X,b)|H_{i-1}]) \right)}_{\coloneqq(C)}\\
    &\quad + \underbrace{\frac{1}{t}\sum_{i=1}^t \left(\sum_{b \in [K]} w(b) (\EE_{P_X}[g_\infty(X, b)|H_{i-1}] - g_\infty(X_i,b)) \right)}_{\coloneqq(D)}
\end{align}
Term $(D)$ converges to zero almost surely by the strong law of large numbers by the same logic as the term $\frac{1}{t}\sum_{i=1}^t B_i^2(\bm{w})$. Term $(C)$ vanishes under the assumption that $\|g_t -g_\infty\|_{L_2(P_X|H_{t-1})} \rightarrow 0$ almost surely. To see this, note that 
\begin{align}
    (C) &= \frac{1}{t}\sum_{i=1}^t \left(\sum_{b \in [K]} w(b)(g_i(X_i,b) - \EE_{P_X}[g_i(X,b)|H_{i-1}]) \right) \\
    &\quad + \frac{1}{t}\sum_{i=1}^t \left(\sum_{b \in [K]} w(b)(\EE_{P_X}[g_i(X_i,b)-g_\infty(X,b)|H_{i-1}]) \right),
\end{align}
where the first line converges almost surely to zero by Lemma \ref{lem:mslln}, and the second line converges almost surely to zero by our assumption $\|g_t -g_\infty\|_{L_2(P_X|H_{t-1})} \rightarrow 0$, Holder's inequality, and Lemma \ref{lem:cesaro_convergence}. Thus, by the boundedness of terms $A_i^2(\bm{w}), B_i(\bm{w})$, we obtain
\begin{equation}
    \lim_{t \rightarrow \infty} l_t(\bm{w}) \rightarrow l_\infty(\bm{w}) \quad \forall \bm{w} \in \Delta(a),
\end{equation}
almost surely, which guarantees $|l_\infty(\bm{w}) - l_t(\bm{w})| \rightarrow 0$ almost surely. Since $l_t(\bm{w})$ is uniformly Lipschitz on the compact set $\Delta(a)$ and converges pointwise almost surely on a dense subset of $\Delta(a)$, it converges uniformly almost surely on $\Delta(a)$ (Chapter 1, \cite{vandervaart1996weak}). 

The proof of uniform convergence for $W_t(\bm{\theta}, \bm{w})$ follows from repeating a similar argument to the one above, and leveraging the fact that $Q_t$ is uniformly Lipschitz on $\Theta \times \Delta(a)$ to obtain uniform convergence. 
\end{proof}

We now leverage the results of Lemmas \ref{lem:safe_set} and \ref{lem:convergence_q} in order to prove that Equation \eqref{eq:to_prove_2} holds. Let $t(\omega)$ be as defined in the proof of Lemma \ref{lem:safe_set}. For sample path $\omega \in \Omega$, for $t \geq t(\omega)$, we obtain 
\begin{align}
    \sup_{\bm\theta \in \Theta} |f_t(\bm\theta, \bm{w}_{a,t}^{\bm{\theta}})(\omega) - f_\infty(\bm\theta, \bm{w}_{a,\infty}^{\bm\theta})| &= \sup_{\bm\theta \in \Theta} \left|\min_{\bm{w} \in \Delta(a), \bm{w}^\top \hat{\bm\mu}_{t-1}(a) \geq 0} f_t(\bm\theta, \bm{w})(\omega) - \min_{\bm{w} \in \Delta(a), \bm{w}^\top \bm{\mu} \geq 0 } f_\infty(\bm\theta, \bm{w}) \right| \label{l:1}\\
    &= \sup_{\bm\theta \in \Theta} \left|\min_{\bm{w} \in \Wcal} f_t(\bm\theta, \bm{w})(\omega) - \min_{\bm{w} \in \Wcal} f_\infty(\bm\theta, \bm{w}) \right|\label{l:2} \\
    &\leq \sup_{(\bm\theta,\bm{w}) \in \Theta \times \in \Wcal} |f_t(\bm{\theta}, \bm{w})(\omega) - f_\infty(\bm{\theta, \bm{w}}) | \label{l:3},
\end{align}
where line \eqref{l:1} holds by definition of $\bm{w}_{a,t}^{\bm\theta}$ and $\bm{w}_{a,\infty}^{\bm\theta}$, line \eqref{l:2} holds by definition of $t(\omega)$, and line \eqref{l:3} holds due to the following inequality
\begin{equation}
    \min_{\bm{w} \in \Wcal} f_t(\bm{\theta},\bm{w})(\omega) \leq \min_{\bm{w} \in \Wcal}f_\infty(\bm{\theta}, \bm{w}_{a,\infty}^{\bm\theta}) + |f_t(\bm{\theta}, \bm{w}_{a,\infty}^{\bm\theta})(\omega) - f_\infty(\bm{\theta}, \bm{w}_{a,\infty}^{\bm\theta}) |,
\end{equation}
which implies that
\begin{equation}
    \left|\min_{\bm{w} \in \Wcal} f_t(\bm{\theta},\bm{w})(\omega) - \min_{\bm{w} \in \Wcal}f_\infty(\bm{\theta}, \bm{w}_{a,\infty}^{\bm\theta})\right| \leq \sup_{\bm{w} \in \Wcal} |f_t(\bm{\theta}, \bm{w})(\omega) - f_\infty(\bm{\theta}, \bm{w}) |.
\end{equation}
We now show that $\sup_{(\bm\theta,\bm{w}) \in \Theta \times \in \Wcal} |f_t(\bm{\theta}, \bm{w})(\omega) - f_\infty(\bm{\theta, \bm{w}})$ converges to zero for each $\omega \in \Omega$. First, by rewriting this term, we obtain
\begin{align}
    \sup_{(\bm\theta,\bm{w}) \in \Theta \times \in \Wcal} |f_t(\bm{\theta}, \bm{w})(\omega) - f_\infty(\bm{\theta, \bm{w}})| &\leq \sup_{(\bm\theta,\bm{w}) \in \Theta \times \in \Wcal} \left|\frac{(\bm{w}^\top \bm{\mu})Q_t(\bm{\theta, \bm{w}})(\omega) - (\bm{w}^\top\bm{\mu})Q_\infty(\bm{\theta}, \bm{w})}{(\bm{w}^\top\bm{\mu})(\bm{w}^\top\hat{\bm{\mu}}_{t-1}(\omega)) }\right|\\
    &\quad + \sup_{(\bm\theta,\bm{w}) \in \Theta \times \in \Wcal} \left|\frac{(\bm{w}^\top \bm{\mu})Q_\infty(\bm{\theta, \bm{w}}) - (\bm{w}^\top\hat{\bm{\mu}}_{t-1}(\omega))Q_\infty(\bm{\theta}, \bm{w})}{(\bm{w}^\top\bm{\mu})(\bm{w}^\top\hat{\bm{\mu}}_{t-1}(\omega)) }\right|
\end{align}
By the fact that $\bm{w}^\top \bm{\hat\mu}_{t-1}(\omega) \geq \sqrt\frac{K\epsilon\exp(-2S)}{K^2B^2(\exp(2S)+4)}\frac{(\mu(a^*)-\mu(a))}{4}$ (Equation \eqref{eq:bounds}) for all $\bm{w} \in \Wcal$ for $t \geq t(\omega)$ and the uniform convergence results of Lemma \ref{lem:convergence_q}, we obtain
\begin{align}
    &\lim_{t \rightarrow \infty}\sup_{(\bm\theta,\bm{w}) \in \Theta \times \in \Wcal} \left|\frac{(\bm{w}^\top \bm{\mu})Q_t(\bm{\theta, \bm{w}})(\omega) - (\bm{w}^\top\bm{\mu})Q_\infty(\bm{\theta}, \bm{w})}{(\bm{w}^\top\bm{\mu})(\bm{w}^\top\hat{\bm{\mu}}_{t-1}(\omega)) }\right|\\
    &\leq \frac{1}{\sqrt\frac{K\epsilon\exp(-2S)}{K^2B^2(\exp(2S)+4)}\frac{(\mu(a^*)-\mu(a))}{4}}\lim_{t \rightarrow \infty} |Q_t(\bm{\theta}, \bm{w})(\omega) - Q_\infty(\bm\theta, \bm{w})|\\
    &= 0.
\end{align}
for all $\omega \in \Omega$, resulting in convergence almost surely. By the fact that on $\bm{w}^\top {\bm{\mu}} \geq \sqrt\frac{K\epsilon\exp(-2S)}{K^2B^2(\exp(2S)+4)}\frac{(\mu(a^*)-\mu(a))}{4}$ for $\bm{w} \in \Wcal$, $\hat\mu_{t-1}(a) \rightarrow \mu(a)$ for all $a \in [K]$ almost surely, and $\sup_{(\bm{\theta}, \bm{w}) \in \Theta \times \Delta(a) } Q_\infty(\bm{\theta}, \bm{w}) \leq K^2B^2(\exp(2S)+4)$, we obtain
\begin{align}
    &\lim_{t \rightarrow \infty}\sup_{(\bm\theta, \bm{w}) \in \Theta \times \Wcal} \left|\frac{(\bm{w}^\top \bm{\mu})Q_\infty(\bm{\theta, \bm{w}}) - (\bm{w}^\top\hat{\bm{\mu}}_{t-1}(\omega))Q_\infty(\bm{\theta}, \bm{w})}{(\bm{w}^\top\bm{\mu})(\bm{w}^\top\hat{\bm{\mu}}_{t-1}(\omega)) }\right|\\
    &\leq \left(\frac{K^2B^2(\exp(2S)+4)}{\frac{K\epsilon\exp(-2S)}{K^2B^2(\exp(2S)+4)}\frac{(\mu(a^*)-\mu(a))^2}{16}}\right) \lim_{t \rightarrow \infty}\sup_{\bm{w} \in \Wcal}\left| \sum_{b \in [K]} w(b)(\mu(b)-\hat\mu_{t-1}(b)(\omega)) \right|\\
    & \leq \left(\frac{K^2B^2(\exp(2S)+4)}{\frac{K\epsilon\exp(-2S)}{K^2B^2(\exp(2S)+4)}\frac{(\mu(a^*)-\mu(a))^2}{16}}\right) \lim_{t \rightarrow \infty}\sum_{b \in [K]} \left|\mu(b) - \hat\mu_{t-1}(b)(\omega) \right|\\
    & = 0,
\end{align}
for all $\omega \in \Omega$, resulting convergence almost surely. Thus, we obtain 
\begin{equation}
    \lim_{t \rightarrow \infty}\sup_{\bm\theta \in \Theta}|F_{a,t}(\bm{\theta}) - F_{a,\infty}(\bm\theta)| = \lim_{t \rightarrow \infty}\sup_{\bm\theta \in \Theta}\left| f_t(\bm{\theta}, \bm{w}_{a,t}^{\bm\theta}) -  f_\infty(\bm{\theta}, \bm{w}_{a,\infty}^{\bm\theta}) \right| = 0,
\end{equation}
almost surely, yielding the desired convergence result for Equation \eqref{eq:to_prove_2} and control over term $(ii)$ in line \eqref{eq:term_2}. It then follows that $\lim_{t\rightarrow \infty}\sup_{\bm{\theta} \in \Theta}|G_t(\bm\theta) - G_\infty(\infty)| = 0$ almost surely,
and by the uniqueness of $\bm{\theta}_\infty = \argmin_{\bm{\theta} \in \Theta} G_\infty(\bm\theta)$ (as shown uniformly over $t \in \NN$ in the proof of Lemma \ref{lem:subgradient_lemma}), a direct application of Lemma \ref{lem:conv_argmax} yields our desired result that $\lim_{t \rightarrow \infty} \| \bm{\theta}_{t,*} - \bm{\theta}_{\infty}\|_2 \rightarrow 0$ almost surely. Taking the limits of Equation \eqref{eq:main_ineq} and under the assumption that $N(t) \rightarrow \infty$ as $t \rightarrow \infty$, 
\begin{align}
    \lim_{t \rightarrow \infty}\| \bm{\theta}_t-\bm{\theta}_{\infty}\|_2 &\leq \lim_{t \rightarrow \infty}\| \bm{\theta}_t - \bm{\theta}_{t,*}\|_2+  \lim_{t \rightarrow \infty}\|\bm{\theta}_{t,*} - \bm{\theta}_{\infty} \|_2\\
    & =\lim_{t \rightarrow \infty}\| \bm{\theta}_{t, N(t)} - \bm{\theta}_{t,*}\|_2+  \lim_{t \rightarrow \infty}\|\bm{\theta}_{t,*} - \bm{\theta}_{\infty} \|_2\\
    & = 0
\end{align}
almost surely, and therefore $\|\pi_t(\cdot, b)- \pi_\infty(\cdot, b) \|_{L_2(P_{X|H_{t-1}})}$ converges to zero almost surely.

\paragraph{Step 3: Optimality under Additional Conditions}

To show the final remark of Theorem \ref{thm:projected_subgradient_descent_convergence}, we only need to establish $G_\infty = G$ under our additional assumptions. Note that the function $G(\pi)$ is defined as 
\begin{align}
    G(\pi) &= \max_{a \neq a^*}F_a(\pi),\\
    F_a(\pi) &= \min_{\bm{w} \in \Delta(a), \bm{w}^\top \bm{\mu} \geq 0} f(\pi, \bm{w}), \\
    f(\pi, \bm{w}) &= \frac{\EE_{P_\infty}\left[\left( \sum_{b \in [K]} w(b)\left(g_\infty(X,b) + \frac{\mathbf{1}[A=b](Y - g_\infty(X,b))}{\pi(X,b)} -\mu(b)\right) \right)^2 \right]}{\left(\sum_{b \in [K]} w(b) \mu(b)\right)^2}.
\end{align}
By Lemma \ref{lem:optimal_policy_structure}, we can instead optimize over $\bm{\theta}$ while remaining the same minimizing value for $G(\bm\theta)$:
\begin{align}
    G(\bm\theta) &= \max_{a \neq a^*}F_a(\bm\theta),\\
    F_a(\bm\theta) &= \min_{\bm{w} \in \Delta(a), \bm{w}^\top \bm{\mu} \geq 0} f(\bm\theta, \bm{w}), \\
    f(\bm\theta, \bm{w}) &= \frac{\EE_{P_\infty}\left[\left( \sum_{b \in [K]} w(b)\left(g_\infty(X,b) + \frac{{\mathbf{1}[A=b](Y - g_\infty(X,b))}}{\left(\sum_{a \in [K]} \sqrt{\frac{{V(x,a)}}{{V(x,b)}}} \exp(\theta(a) - \theta(b))\right)^{-1}} -\mu(b)\right) \right)^2 \right]}{\left(\sum_{b \in [K]} w(b) \mu(b)\right)^2}.
\end{align}
The function $G_\infty$ is defined as follows:
\begin{align}
    G_\infty(\bm{\theta})&= \max_{a \neq a^*} F_{a, \infty}(\bm{\theta}\\
    F_{a,\infty}(\bm{\theta}) &= \min_{\bm{w} \in \Delta(a), \bm{w}^\top \bm{\mu} \geq 0} f_{\infty}(\bm\theta, \bm{w}), \\
        f_\infty(\bm{\theta, \bm{w}}) &= \frac{\EE_{P_X}\left[ \sum_{b \in [K]} \left(w^2(b) V_\infty(X,b) \sum_{a \in [K]}\sqrt{\frac{V_\infty(X,a)}{V_\infty(X,b)}}\exp\left(\theta(a) - \theta(b)\right)\right) \right] + l_\infty(\bm{w})}{\left(\sum_{b \in [K]}w(b)\mu(b) \right)^2}, \\
        l_\infty(\bm{w}) &= {\EE_{P_X}\left[ \left(\sum_{b \in [K]} w(b)\left(g_\infty(X, b) - \mu(b) \right) \right)^2 \right]}
\end{align}
Note that $G(\bm\theta)$ is equal to $G_\infty(\bm{\theta})$ as long as the numerators of $f(\bm{\theta}, \bm{w})$ and $f_\infty(\bm\theta, \bm{w})$ are equal. We show this below under the assumption that $V_\infty = V$:
\begin{align}
f(\bm{\theta}, \bm{w}) &= \EE_{P_\infty}\!\left[
\left(
\sum_{b \in [K]} w(b)\left(
g_\infty(X,b)
+ \frac{{\mathbf{1}[A=b](Y-g_\infty(X,b))}}{\left(\sum_{a \in [K]} \sqrt{\frac{V(X,a)}{V(X,b)}} \exp(\theta(a)-\theta(b))\right)^{-1}}
-\mu(b)
\right)
\right)^2
\right] \\
&=
\EE_{P_\infty}\!\left[
\left(
\sum_{b \in [K]} w(b)\bigl(g_\infty(X,b)-\mu(b)\bigr)
+
\sum_{b \in [K]} w(b)\frac{{\mathbf{1}[A=b](Y-g_\infty(X,b))}}{\left(\sum_{a \in [K]} \sqrt{\frac{V(X,a)}{V(X,b)}} \exp(\theta(a)-\theta(b))\right)^{-1}}
\right)^2
\right] \\
&=
\EE_{P_X}\!\left[
\left(\sum_{b \in [K]} w(b)\bigl(g_\infty(X,b)-\mu(b)\bigr)\right)^2
\right] \\
&\quad
+
\EE_{P_X}\!\left[
\sum_{b \in [K]} w^2(b)
\EE\!\left[(Y-g_\infty(X,b))^2 \mid X,A=b\right]
\sum_{a \in [K]} \sqrt{\frac{V(X,a)}{V(X,b)}} \exp(\theta(a)-\theta(b))
\right] \\
&=
l_\infty(\bm w)
+
\EE_{P_X}\!\left[
\sum_{b \in [K]} w^2(b) V(X,b)
\sum_{a \in [K]} \sqrt{\frac{V(X,a)}{V(X,b)}} \exp(\theta(a)-\theta(b))
\right],
\end{align}
which is exactly equal the the numerator of $f_\infty(\bm\theta, \bm{w})$. This concludes our proof.

\subsection{Proof of Lemma \ref{thm:results_general}}
This proof follows from a direct application of Theorem \ref{thm:projected_subgradient_descent_convergence} and Lemma \ref{lem:asymp_valid_bai}. Note that the assumptions of Theorem \ref{thm:projected_subgradient_descent_convergence}, in addition to the results that (i) $\pi_{t}(x,b) \geq 1/\kappa > 0$ for all $x \in \Xcal$, $b \in [K]$, and $t \in \NN$ and (ii) the existence of a limit policy $\pi_\infty$ such that $\|\pi_t(\cdot, b) - \pi_\infty(\cdot, b) \|_{L_2(P_{X|H_{t-1}})} \rightarrow 0$ almost surely, match the assumptions of Lemma \ref{lem:asymp_valid_bai}. As a result, we obtain our sampling policy in Algorithm \ref{alg:proj-subgradient} yields a BAI algorithmic sequence that satisties asymptotic $\alpha$-correctness and terminates in finite time almost surely.

\subsection{Proof of Theorem \ref{thm:sampling_complexity_final}}

To prove Theorem \ref{thm:sampling_complexity_final}, we leverage the results of Theorem \ref{thm:sample_complexity_general} and Lemma \ref{lem:kl_projection}. We start by establishing the stopping time bound for all $P \in \Pcal(\bm{\mu}, \bm{\sigma}^2)$ using Theorem \ref{thm:sample_complexity_general}. First, note that for all $P \in \Pcal(\bm{\mu}, \bm{\sigma}^2)$, under the assumption that $g_\infty=g$, $V_\infty = v$, and $\bm{\theta}_* \in \Theta$, we obtain that for all $P \in \Pcal(\bm{\mu}, \bm{\sigma}^2)$, the stopping time (under the conditions that $t_0(\alpha) = o(\log(1/\alpha))$) is upper bounded by
\begin{equation}
    \Gamma_2'(\bm\mu, \bm\sigma^2) = \inf_{\pi \in \Pi}\sup_{a \neq a^*} \inf_{\bm{w} \in \Delta(a), \bm{w}^\top \bm\mu \geq 0}  2\frac{\EE\left[\sum_{b \in [K]}w^2(b) \frac{v(x,b)}{\pi(x,b)}\right] + \EE_{P_X}\left[\left(\sum_{b \in [K]} w(b) (g(x,b) - \mu(b))\right)^2\right]}{\left( \sum_{b \in [K]} w(b) \mu(b) \right)^2}
\end{equation}
both in expectation and almost surely. Thus, for all $P \in \Pcal(\bm\mu, \bm\sigma^2)$, under our assumptions, 
\begin{equation}
    \lim_{\alpha \rightarrow 0}\frac{\EE_P[\tau_{t_0(\alpha)}]}{\log(1/\alpha)} \leq \Gamma_2'(\bm\mu, \bm\sigma^2), \quad P\left(\lim_{\alpha \rightarrow 0} \frac{\tau_{t_0(\alpha)}}{\log(1/\alpha)} \leq \Gamma_2'(\bm\mu, \bm\sigma^2)\right) = 1.
\end{equation}
We now turn to showing the inequalities presented in Theorem \ref{thm:error_control}. First, by Lemma \ref{lem:kl_projection},
\begin{equation}
    \sup_{\bm{w} \in \Delta(a), \bm{w}^\top \bm{\mu}\geq 0}\frac{1}{2}\left( \frac{\sum_{b \in [K]}w(b) \mu(b)}{\sqrt{\sum_{b \in [K]} w(b)^2\sigma^2(b)/\pi(b)}}\right)^2 = \inf_{\tilde{\bm{\mu}} \in \Hcal_a} \sum_{b \in [K]}\pi(b)  d_{\sigma(b)}\left(\mu(b), \tilde{\mu}(b)\right),
\end{equation}
and by taking the minimum SNR ratio across all suboptimal arms $a \neq a^*$, we obtain
\begin{equation}
    \inf_{a \neq a^*}\sup_{\bm{w} \in \Delta(a), \bm{w}^\top \bm{\mu}\geq 0}\frac{1}{2}\left( \frac{\sum_{b \in [K]}w(b) \mu(b)}{\sqrt{\sum_{b \in [K]} w(b)^2\sigma^2(b)/\pi(b)}}\right)^2 = \inf_{a \neq a^*, \tilde{\bm\mu} \in \Hcal_a} \sum_{b \in [K]}\pi(b)  d_{\sigma(b)}\left(\mu(b), \tilde{\mu}(b)\right).
\end{equation}
By combining the constraints on the minimization on the RHS and taking the supremum over $\pi \in \Pi$,
\begin{equation}
    \sup_{\pi \in \Pi}\inf_{a \neq a^*}\sup_{\bm{w} \in \Delta(a), \bm{w}^\top \bm{\mu}\geq 0}\frac{1}{2}\left( \frac{\sum_{b \in [K]}w(b) \mu(b)}{\sqrt{\sum_{b \in [K]} w(b)^2\sigma^2(b)/\pi(b)}}\right)^2 = \sup_{\pi \in \Pi}\inf_{\tilde{\bm\mu} \not\in \Hcal_{a^*}} \sum_{b \in [K]}\pi(b)  d_{\sigma(b)}\left(\mu(b), \tilde{\mu}(b)\right)
\end{equation}
By taking the inverse of this expression, we obtain
\begin{equation}
    \inf_{\pi \in \Pi} \sup_{a \neq a^*} \inf_{\bm{w}\in \Delta(a), \bm{w}^\top \bm{\mu}} 2\left( \frac{\sqrt{\sum_{b \in [K]} w(b)^2\sigma^2(b)/\pi(b)}}{\sum_{b \in [K]}w(b) \mu(b)}\right)^2 = \Gamma_2(\bm\mu, \bm\sigma^2).
\end{equation}
We now compare the bound we obtained for $\Gamma_2'(\bm\mu, \bm\sigma^2)$ compared to $\Gamma_2(\bm{\mu}, \bm\sigma^2)$. For any fixed choice of $\pi \in \Pi^{\text{MAB}} \coloneqq \{\pi \in \Pi: \pi(x,b) = \pi(b) \ \forall b \in [K], P_X \ a.s.\}$ and $\bm{w} \in \{\bm{w} \in \Delta(a): \bm{w}^\top \bm\mu \geq 0\}$ for all $a \neq a^*$, note that
\begin{align}
    &\frac{\EE\left[\sum_{b \in [K]} w^2(b)\frac{v(x,b)}{\pi(b)}\right] + \EE_{P_X}\left[\left(\sum_{b \in [K]} w(b) (g(x,b) - \mu(b))\right)^2\right]}{\left( \sum_{b \in [K]} w(b) \mu(b) \right)^2} - \left( \frac{\sqrt{\sum_{b \in [K]} w(b)^2\sigma^2(b)/\pi(b)}}{\sum_{b \in [K]}w(b) \mu(b)}\right)^2  \\
    &=\frac{\EE\left[\sum_{b \in [K]} w^2(b) \frac{v(x,b)}{\pi(b)}\right] + \EE_{P_X}\left[\left(\sum_{b \in [K]} w(b) (g(x,b) - \mu(b))\right)^2\right] - \sum_{b \in [K]} w^2(b)\sigma^2(b)/\pi(b)}{\left(\sum_{b \in [K]} w(b) \mu(b) \right)^2}\\
    &= \frac{\EE_{P_X}\left[\left(\sum_{b \in [K]} w(b) (g(x,b) - \mu(b))\right)^2\right] - \sum_{b \in [K]}w^2(b)\EE_{P_X}\left[(g(x,b)- \mu(b))^2\right]/ \pi(b)}{\left(\sum_{b \in [K]} w(b) \mu(b) \right)^2}\label{eq:law_of_total_variance}
\end{align}
where the last line follows from the total law of variance identity given by
\begin{equation}
    \sigma^2(b) = \EE_{P_X}\left[ v(x,b) \right] + \EE_{P_X}\left[(g(x,b)- \mu(b))^2\right].
\end{equation}
To show that this term is nonpositive, note that
\begin{align}
    \EE_{P_X}\left[\left(\sum_{b \in [K]} w(b) (g(x,b) - \mu(b))\right)^2\right] &= \EE_{P_X}\left[ \left(\sum_{b \in [K]} \sqrt{\pi(b)}\frac{w(b)(g(x,b)-\mu(b))}{\sqrt{\pi(b)}} \right)^2\right]\\
    &\leq \EE_{P_X}\left[\left(\sum_{b \in [K]}\pi(b) \right)\left(\sum_{b \in [K]}\frac{w^2(b)(g(x,b)-\mu(b))^2}{\pi(b)} \right)\right] \\ 
    &= \sum_{b \in [K]}w^2(b)\EE_{P_X}\left[(g(x,b)- \mu(b))^2\right]/ \pi(b),
\end{align}
where the inequality is by direct application of Cauchy Schwartz, resulting in the expression being non-positive:
\begin{equation}\label{eq:ref}
    \frac{\EE\left[\sum_{b \in [K]} w^2(b)\frac{v(x,b)}{\pi(b)}\right] + \EE_{P_X}\left[\left(\sum_{b \in [K]} w(b) (g(x,b) - \mu(b))\right)^2\right]}{\left( \sum_{b \in [K]} w(b) \mu(b) \right)^2} - \left( \frac{\sqrt{\sum_{b \in [K]} w(b)^2\sigma^2(b)/\pi(b)}}{\sum_{b \in [K]}w(b) \mu(b)}\right)^2  \leq 0.
\end{equation}
We now prove the first result of Theorem \ref{thm:sampling_complexity_final} by contradiction. Assume that there exists a pair $(\bm{\mu}, \bm{\sigma}^2)$ such that $\Gamma_2'(\bm\mu, \bm\sigma^2) > \Gamma_2(\bm{\mu},\bm\sigma^2)$. By definition of $\Gamma_2'(\bm\mu, \bm\sigma^2)$, 
\begin{equation}
    \Gamma_2'(\bm\mu, \bm\sigma^2) \leq \inf_{\pi \in  \Pi^{\text{MAB}}}\sup_{a \neq a^*} \inf_{\bm{w} \in \Delta(a), \bm{w}^\top \bm\mu \geq 0}  2\frac{\EE_{P_X}\left[\sum_{b \in [K]}w^2(b) \frac{v(x,b)}{\pi(x,b)}\right] + \EE_{P_X}\left[\left(\sum_{b \in [K]} w(b) (g(x,b) - \mu(b))\right)^2\right]}{\left( \sum_{b \in [K]} w(b) \mu(b) \right)^2},
\end{equation}
and by our results above, for any choice of $\pi \in \Pi^{\text{MAB}}$ and $\bm{w} \in \Delta(a)$ for all $a \neq a^*$, Equation \ref{eq:ref} holds, even at the optimal $\pi$ and $\bm{w}$ that achieves $\Gamma_2(\bm\mu, \bm\sigma^2)$. This results in a contradiction, and therefore it must be that for all $P \in \Pcal(\bm\mu, \bm\sigma^2)$, $\Gamma_2'(\bm\mu, \bm\sigma^2) \leq \Gamma_2(\bm{\mu},\bm\sigma^2)$.

To show that the inequality is strict when $P \in \widetilde{\Pcal}(\bm{\mu}, \bm{\sigma}^2)$, we follow the same exact steps as above, with a slight update to our inequality step based on Cauchy Schwartz. In order for the Cauchy Schwartz inequality to be an equality, there must be some function $c(x)$ such that the following holds almost surely w.r.t. $P_X$:
\begin{equation}\label{eq:contradict}
    \sqrt{\pi(b)} = c(x)\frac{w(b)(g(x,b) - \mu(b))} {\sqrt{\pi}(b)} \iff \pi(b) = c(x) w(b)(g(x,b)- \mu(b)).
\end{equation}
We prove this result via contradiction. Assume that Equation \eqref{eq:contradict} is true. Note that if there exists $a,a'$ such that $(g(x,a)-\mu(a))(g(x,b)-\mu(b))$ over some set $\widetilde{\Xcal}$ with positive measure, either $(g(x,a)-\mu(a))$ or $(g(x,b)-\mu(b))$ must be negative. Because $\pi(b) > 0$ for all $b \in [K]$ (otherwise, an infinite stopping time bound), $c(x)$ must be less than zero for one of $a$ or $a'$ for $x \in \widetilde\Xcal$, but must be positive for the other. Thus, Equation \eqref{eq:contradict} cannot be true, and this results in contradiction.

\subsection{Proof of Lemma \ref{lem:two_armed}}

Note that in the two-armed case, $\Delta(a)$ is a singleton, and our stopping time bound (derived as $\Gamma_2'(\bm{\mu}, \bm{\sigma}^2)$ in the proof of Theorem \ref{thm:sample_complexity_general}) becomes
\begin{equation}\label{eq:ref_2}
    \Gamma_2'(\bm{\mu}, \bm{\sigma}^2) = 2\inf_{\pi \in \Pi} \frac{\sum_{b \in [2]}\EE_{P_X}\left[\frac{v(x,b)}{\pi(x,b)}\right] + \EE_{P_X}\left[\left(\sum_{b \in [2]}g(x,b)-\mu(b)\right)^2\right]}{\left(\mu(1) -\mu(2)\right)^2} 
\end{equation}
Note that under the assumptions of Lemma \ref{lem:two_armed}, we achieve the optimal $\pi \in \Pi$ under the stronger conditions of Theorem \ref{thm:projected_subgradient_descent_convergence}. By Section 2.2 of \cite{cook2024semiparametricefficientinferenceadaptive}, the optimal policy $\pi_*$ for minimizing the numerator (only term with $\pi$ dependence) is given by:
\begin{equation}
    \pi_*(x,b) = \frac{\sqrt{v(x,b)}}{\sqrt{v(1,x)}+ \sqrt{v(2,x)}},
\end{equation}
and by plugging in $\pi_*$ in Equation \eqref{eq:ref_2}, we obtain the results of Lemma \ref{lem:two_armed}.

\section{Selection of Hyperparameters}\label{app:hyperparameters}

The parameter $\rho > 0$ governs the time $t^*$ in which our test has maximal power (i.e., where the threshold $\ell_{t,\alpha, \rho}(\tilde\sigma_t(\Wcal_T))$ is relatively smallest). Following the approximate approach of \cite{waudbysmith2024timeuniformcentrallimittheory}, for $\alpha < 0.5$, power is approximately maximized at $t_*$ by setting $\rho$ as the following function of $t_*$ and error level $\alpha$:
\begin{equation}
\rho = \sqrt{\frac{-\log(2\alpha) + \log(1-2\log(2\alpha))}{t^*}}.
\end{equation}
In Theorems \ref{thm:results_general} and \ref{thm:sampling_complexity_final}, we show that stopping times $\tau_{t_0(\alpha)}$ are upper bounded by terms on the order of $1/\log(\alpha)$. Thus, we recommend the choice of $\rho = c \log(1/\alpha)$, where $c$ is a constant chosen based on domain knowledge on the sample complexity of a task and sampling budget. If one expects larger stopping times with cheap samples, we recommend a large choice of $c$; alternatively, for tasks with small expected stopping times and expensive samples, we recommend smaller choices of $c$.

\section{Additional Experiment Details}\label{app:experiments}

\paragraph{Compute Details} All baselines baselines were run locally on a M2 14-inch 2023 MacBook Pro with 16GB of RAM. For our noncontextual baselines, we used the implementation by \cite{bestarmidentification_github}. All default settings (other than arm means and $\alpha$) were kept constant. For CT\&S \citep{kato2024rolecontextualinformationbest}, we implemented their algorithm as described in the main body of the paper (Section \ref{sec:experiments}). For our approach, we implemented our approach in Python as discussed in \ref{sec:experiments} using an Amazon EC2 with instance \texttt{c6in.8xlarge}, parallelized with 24 workers. For all methods, we update both the test statistic and the sampling scheme at each timestep, and set a maximum number of samples as $30,000$. No approach (including ours) failed to terminate. 

\paragraph{Choice of Regression Model} For our approach using conditional regression models $g_t$ and $V_t$, we used a probit model and linear regression model respectively to estimate conditional means and variances. Our choice to leverage simple models for our regressors allowed for us to update the estimates at each timestep without severe computational overhead. In future work, we plan to test more complicated regression functions under a batched updating scheme.

\end{appendices}

\end{document}